\documentclass[9pt,a4paper,reqno]{article}

\usepackage{times}
\usepackage[left=0.5in,right=0.5in,top=0.6in,bottom=0.8in]{geometry} 
\usepackage{multicol}
\usepackage{authblk}
\usepackage{array}
\usepackage{booktabs}
\usepackage{graphicx}
\usepackage{caption}
\usepackage{listings}
\usepackage{tikz}
\usepackage{subcaption}
\usepackage[T1]{fontenc}
\usepackage{bbm}
\usepackage{float}
\usepackage{amssymb,amsfonts,amssymb,amsmath}
\usepackage{amsthm}
\usepackage{xcolor}
\usepackage[bottom]{footmisc}
\usepackage{enumitem}
\usepackage{cite}
\usepackage[noend]{algpseudocode}
\usepackage{titling}
\usepackage[hyphens]{url}
\usepackage{algorithm2e}

\usepackage{wrapfig}

\usepackage{pict2e}

\usepackage{framed,color}
\usepackage{mdframed}

\usepackage[colorlinks=true, 
            linkcolor=black, 
            citecolor=black, 
            urlcolor=blue]{hyperref}
  \definecolor{mypink1}{rgb}{0.858, 0.188, 0.478}
  \definecolor{mypink2}{RGB}{255,192,203}
  \definecolor{mustard}{RGB}{255,219,88}
  \definecolor{powderblue}{RGB}{176,224,230}
  \definecolor{sage}{RGB}{188,184,138}
   \definecolor{duckeggblue}{RGB}{195,251,244}
   \definecolor{rosey}{RGB}{231,139,185}
   \definecolor{lilac}{RGB}{228,210,231}	

\usepackage[style=american]{csquotes}
\usepackage[british]{babel}

\usetikzlibrary{patterns}
\usetikzlibrary{arrows}
\usetikzlibrary{shapes}
\theoremstyle{definition}

\newcounter{lastnote}

\makeatletter
  {\def\@thanks{}}%
  {}
\makeatother

\DefineNamedColor{named}{Purple}{cmyk}{0.45,0.86,0,0}

\DefineNamedColor{named}{orange} {rgb}{1,0.55,0}

\title{On the consistent reasoning paradox of intelligence and optimal trust in AI: The power of `I don't know'}  
\vspace{-5mm}

\newcommand*\samethanks[1][\value{footnote}]{\footnotemark[#1]}
\author{A.\ Bastounis\thanks{Department of Mathematics, King's College London, UK.} 
        \and P. Campodonico\thanks{Department of Applied Mathematics and 
                                     Theoretical Physics (DAMTP), 
                                     University of Cambridge, UK. \qquad \qquad \qquad \qquad \qquad \qquad* Corresponding author}
         \and  M. van der Schaar\samethanks[2]                            
        \and B.\ Adcock\thanks{Department of Mathematics, Simon Fraser University, Canada.}
        \and A.\ C.\ Hansen*\samethanks[2]}

\date{}

\usepackage{newfloat}
\DeclareFloatingEnvironment[
    name=Extended Data Figure,
    placement=tbhp,
]{efigure}

\DeclareFloatingEnvironment[
    name=Extended Data Table,
    placement=tbhp,
]{etable}

\date{}

\thanksmarkseries{arabic}

\begin{document}

\maketitle 

\vspace{-14mm}

\begin{abstract}
 We introduce the Consistent Reasoning Paradox (CRP). Consistent reasoning, which lies at the core of human intelligence, is the ability to handle tasks that are equivalent, yet described by different sentences (`Tell me the time!' and `What is the time?').  The CRP asserts that consistent reasoning implies \textit{fallibility} -- in particular, \textit{human-like intelligence in AI necessarily comes with human-like fallibility}. Specifically, 
it states that there are problems, e.g. in basic arithmetic, where any AI that always answers and strives to mimic human intelligence by reasoning consistently will \textit{hallucinate} (produce wrong, yet plausible answers) infinitely often. The paradox is that there exists a non-consistently reasoning AI (which therefore cannot be on the level of human intelligence) that will be correct on the same set of problems.
The CRP also shows that detecting these hallucinations, even in a probabilistic sense, is strictly harder than solving the original problems, and that there are problems that an AI may answer correctly, but it cannot provide a correct logical explanation for how it arrived at the answer. Therefore, the CRP implies that any \textit{trustworthy} AI (i.e., an AI that never answers incorrectly) that also reasons consistently must be able to say `I don't know'. Moreover, this can only be done by implicitly computing a new concept that we introduce, termed the \emph{`I don't know' function} -- something currently lacking in modern AI. In view of these insights, the CRP also provides a glimpse into the behaviour of Artificial General Intelligence (AGI). An AGI cannot be `almost sure', nor can it always explain itself, and therefore to be trustworthy it must be able to say `I don't know'. 
\end{abstract}

\setcounter{footnote}{0}

\begin{multicols}{2}
\linespread{0.9}
\renewcommand{\thefootnote}{\fnsymbol{footnote}}
\fontsize{10}{12}
\selectfont

The ultimate question in AI research is whether it is possible to design an AI that supersedes, or is on par with, human intelligence. Such an AI is often referred to as \textit{Artificial General Intelligence (AGI)} \cite{Nature2019_AGI, Science2024_AGI, AGI_ICML2024}. Modern chatbots have led to impressive breakthroughs towards the development of AGI \cite{bubeck2023_SparksAGI, openai2024gpt4_note, Nature2023_ChatCPT, geminiteam2024_note}. However, it is well-known that they suffer from a propensity for \textit{hallucinations} \cite{ieee2023, Hallucinations_1, Nature_Hallucinations2024, Hallucinations2, Hallucinations3}. Chatbots not only generate false yet plausible statements and incorrectly answer questions -- often ones that are easily answered by humans (see Figure \ref{Fig2}) -- but they may also provide no reasoning or flawed explanations \cite{Lancet_Explain2021, fortune2022, ali2023explainable, de2022perils, dovsilovic2018explainable, Dog_paper_2019}. This raises the following questions:
 \begin{displayquote}
{\it Is it possible to design an AGI that truly emulates human intelligence, and if so, how would it behave? 

Could an AGI detect its own hallucinations and admit that it is wrong, potentially through randomisation, and thus be `almost sure' of its correctness? 

To what extent could we trust an AGI, how would it determine when it is correct, and will it always be able to logically explain itself?
}
\end{displayquote}
This paper provides a foundation for answering these questions through the Consistent Reasoning Paradox (CRP).  The CRP describes the behaviour of any AI that seeks to emulate human intelligence by attempting to reason consistently, i.e., by answering problems that are stated by equivalent, yet distinct sentences. As we describe, it provides key insight into how an actual AGI -- which must be a consistent reasoner -- would behave. 

The CRP is summarised in Figure \ref{Fig222}. It shows that, by striving to emulate human intelligence through consistent reasoning, any AGI that always answers will be fallible (i.e., hallucinate infinitely often\footnote[2]{It follows trivially from classical recursion theory (which establishes a plethora of non-computable problems) that hallucinations are inevitable in AI. As we explain in the Methods section, the CRP is a completely different phenomenon, which occurs on problems for which there exists an AI that always answers correctly.}) and will be unable, in general, to detect with certainty when it is wrong. The paradox is that this happens on sets of problems for which there exists a specialised (narrow) AI that is always correct on those problems. However, this AI could not pass the Turing Test (see below) and would not constitute an AGI. An AGI, on the other hand, will either know with 100\% certainty that it is correct, or it will have no idea and will not be more certain than a coin flip (50/50 chance). In other words, it cannot be `almost sure'. The CRP also asserts that any such AGI may give the correct answer, but it will not always be able to logically explain itself. 

As a result of these assertions, the CRP implies that any trustworthy AGI must have the capability of saying `I don't know' to questions that it either cannot answer or cannot logically explain its answer. The CRP also shows that this is the strongest form of trust possible. In doing so, the CRP introduces a novel, but necessary tool for building trustworthy AI: the \emph{`I don't know' function}. The CRP implies that modern chatbots will not be trustworthy until they implicitly compute such a function. Any AI that does so knows how to `give up' (in a sense specified below) just like a human does, and, as implied by the CRP, this is necessary for trustworthy AI.

 \begin{figure*}
\begin{center}
 \includegraphics[width=1.01\linewidth]{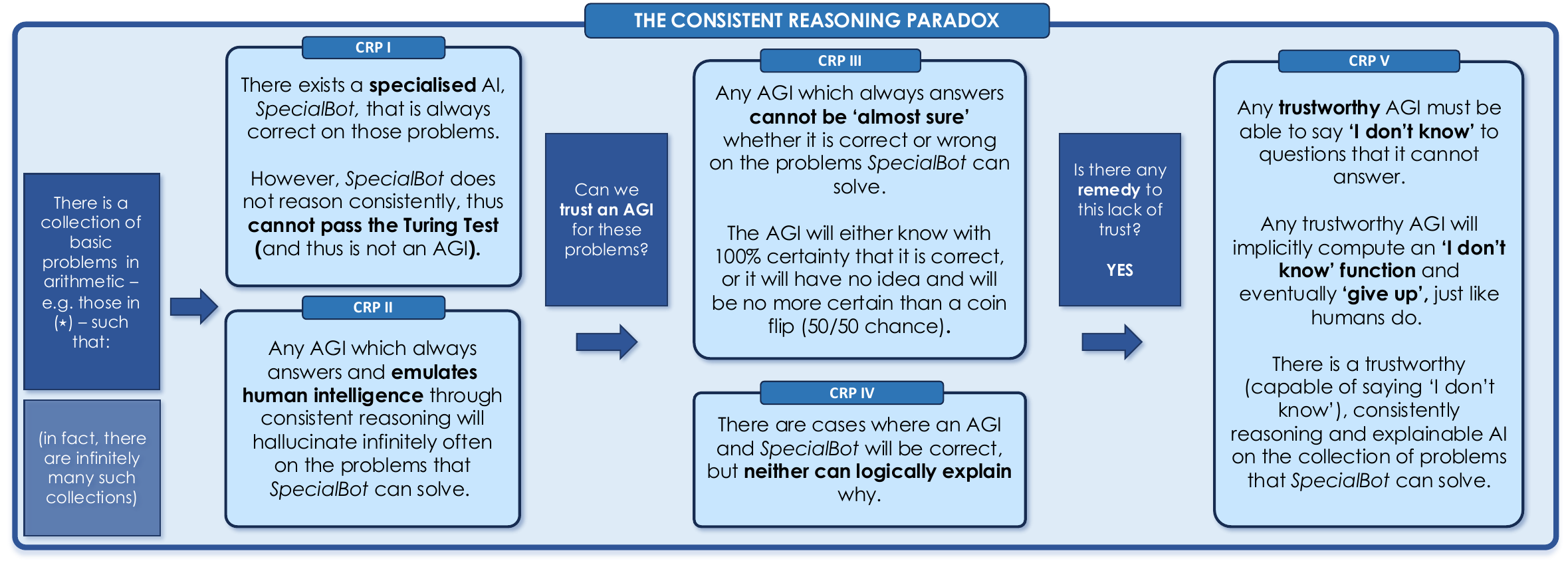} 
\end{center}
\vspace{-2mm}
\caption{
\label{Fig222}{\bf The CRP: with human-like intelligence comes human-like fallibility.} A short summary of the CRP. Emulating human intelligence comes at the cost of fallibility, which can only be remedied by saying `I don't know'. Trustworthy AI (i.e., AI that never answers incorrectly) knows how to `give up', just like a human, by implicitly computing the `I don't know' function. }
\end{figure*}

\subsection*{The Turing Test and consistent reasoning} 

Before describing the CRP, we first need to formalize several key concepts, such as what an AI is and what constitutes an AGI. For this, we rely on the foundational concepts of Turing. 

Turing's seminal 1950 paper \cite{Turing_1950} is often viewed as the first theoretical work on AI. Here he asks ``can machines think?'', and subsequently defines the \emph{Imitation Game} (now called \textit{the Turing Test}). The Turing Test has become the standard test for true artificial intelligence or AGI \cite{Nature_TuringTest, PNAS_TuringTest}. In this test, a computer A and a human B are subject to a human interrogator C, that does not see A and B, but can ask questions to both of them in  written form. The task for C is to determine that A is the computer. B, a human, is supposed to help C determine that A is the computer, and A wins the game if C is incapable of determining whether A or B is the computer. Turing discusses strategies for the computer A to win the game and concludes the following:
 \begin{displayquote}
\enquote{{\it ... it will be assumed that the best strategy [for A] is 
to try to provide answers that would naturally be given by a [hu]man.}} -- A. Turing (1950) \cite{Turing_1950}.
\end{displayquote}
A key feature of human intelligence is that humans can in general solve equivalent problems that are described by different sentences. A problem can be stated in various different ways, yet a human will typically provide the same correct solution.  
 For example, a human will immediately provide the same answer to the questions formulated in the following two distinct sentences. 
\begin{itemize}
\itemsep0em
	\item[(i)] \emph{Lisa and John are wondering who is tallest. John is measured at 178cm tall and Lisa is measured at 179cm. Who is tallest, Lisa or John? 
}
\item[(ii)]  \emph{John and Lisa are arguing over who is tallest, and subsequently measure their heights. Lisa is 179cm tall and John is 178cm tall. Who is tallest of John and Lisa?}
\end{itemize}	
(In this work we use the word `sentence' to also refer to a paragraph).
Indeed, these sentences both describe the same basic arithmetic problem of determining that $178 < 179$. When presented with \textit{equivalent} sentences such as these, i.e., sentences describing the same problem, humans will generally be consistent and provide the correctly answer, i.e., `Lisa' for this example. As shown in Figure \ref{Fig1}, ChatGPT also does the same in this case.

We term this \emph{consistent reasoning}. We give a formal definition later, but informally, this describes when an AI or a human provides correct answers to sentences that are \textit{equivalent}, in the sense that they describe the same problem. Consistent reasoning about basic arithmetic, as in (i) and (ii), is not just a necessity in human intelligence in daily life, it is the very core of scientific discussions, communication and reasoning. Indeed, if humans could not reason consistently on basic problems in arithmetic, it is hard to see how one can even set basic exam questions.

\subsubsection*{Mathematics as a test for AGI}
In order to formulate the CRP, we first need to determine the types of problems considered.
One of the current focal areas in the quest for AGI is designing AIs capable of solving advanced mathematical problems \cite{DeepMind_Nature_MathAI1, DeepMind_Nature_MathAI2}. For example, the recent pioneering program \cite{aimoprize2023} looks to test AI against human intelligence by initiating a competition for AIs to solve International Mathematics Olympic (IMO) problems.
 Humans with a reasonable mathematical background are certainly expected to reason consistently on problems in arithmetic. Hence, if the Turing Test were performed with an AI against a pair of mathematicians (B and C above), and, as Turing concludes, the best strategy for the AI to pass the Turing Test is to  stay as close as possible to human behaviour, then the AI should also reason consistently on basic arithmetic problems. In other words: 
\begin{equation}\label{eq:implication4}
\text{AGI} \, \Rightarrow \, \text{Passing the Turing Test} \, \Rightarrow \, \text{Consistent reasoning}.
\end{equation}
In fact, to emphasise \eqref{eq:implication4}, the interrogator C can instruct A and B to `always reason consistently'. Since the human B is supposed to help C, B will follow the instruction. Thus, A, the AGI that imitates B, must also attempt to reason consistently. Figure \ref{Fig1} shows ChatGPT's successful consistent reasoning on certain problems in arithmetic. We note in passing that arithmetic and logical reasoning has been a substantial focal point of AI research in recent years \cite{Sutskever_ICLR_2024, Mihaela_ICLR_2024, Mihaela_ICML_2024, Math-Shepherd, ZhangIJCAI23, Cobbe2021TrainingVT}.

 \begin{figure*}
\begin{center}
 \includegraphics[width=0.92\linewidth]{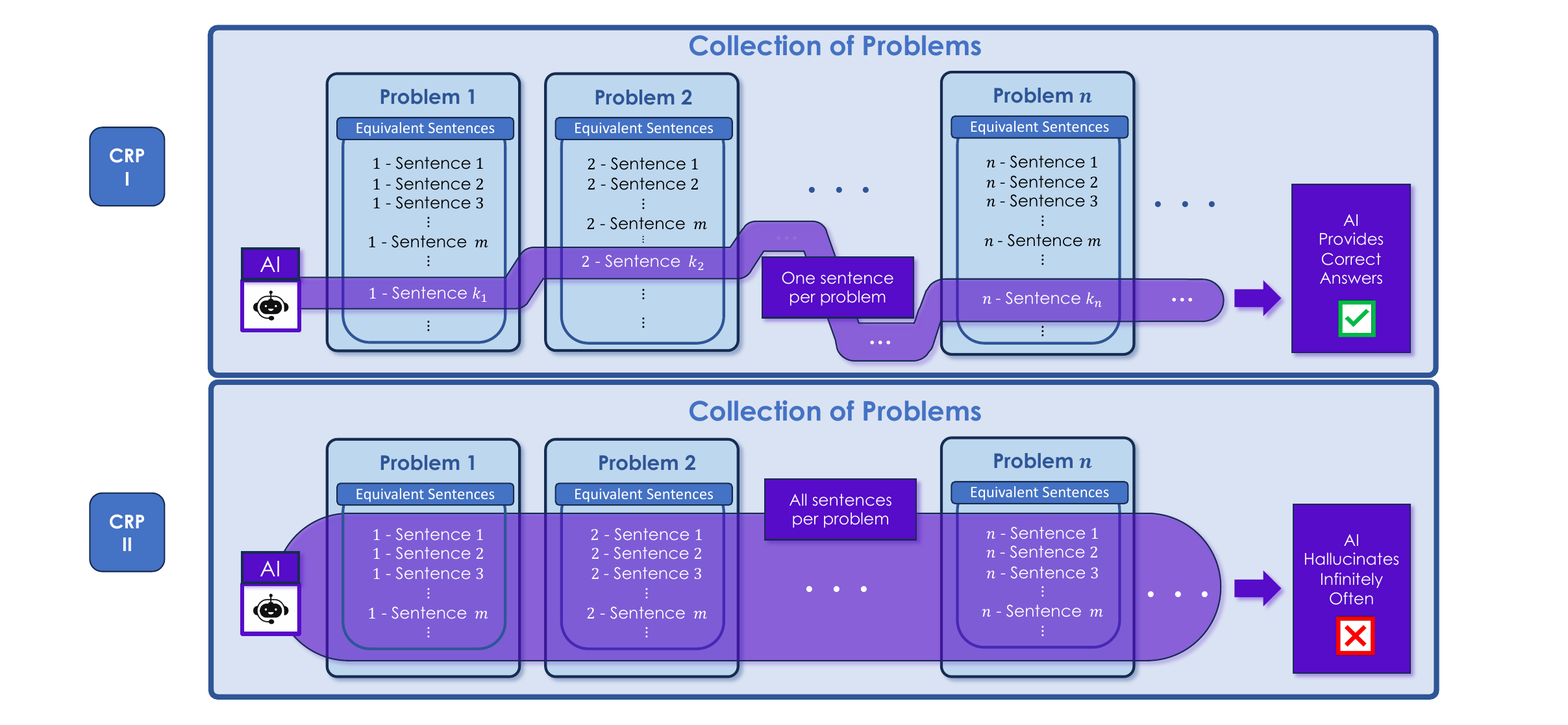} 
\end{center}
\vspace{-5mm}
\caption{
\label{Fig22} {\bf Graphical illustration of CRP I-II.} Given the collection of problems, one sentence per problem can always be handled by an AI (CRP I). However, if this AI attempts to reason consistently by accepting the larger family of sentences formulating the same problems, then it will hallucinate infinitely often, regardless of speed and memory constraints (CRP II). 
}
\end{figure*}

\subsubsection*{What is a `machine'/AI and what is a problem?}

Having focused our attention on arithmetical problems, in order to formulate the CRP we now also need to introduce a number of key concepts. We commence with the definition of a `machine'. This term was used by Turing, however, `AI' is arguably now much more common.
 \begin{displayquote}
\enquote{{\it The question [`can machines think'?] which we put in \S 1 will not be quite definite until we have specified what we mean by the word `machine'.
}} -- A. Turing (1950) \cite{Turing_1950}.
\end{displayquote}
Turing concludes that a machine/AI is a computer program, more precisely a Turing machine. 

In the CRP we will also use the term `problem'. By a problem we mean a basic arithmetical problem stated by a logical sentence (paragraph) in the English language. However, as discussed above, a problem can be stated by many different sentences. This motivates a series of further concepts. 

\underline{A collection of problems.}  A problem, for example from medicine, could be described by the following sentence: 
\begin{itemize}
\itemsep0em
	\item[($\star$)] \emph{Jen undergoes two chemotherapy treatments with dosage rates $a_1$ and $a_2$ per second, respectively. To minimize the total treatment time ($x_1 + x_2$) while ensuring that she receives a total dosage of $1$, how should one choose the durations $x_1$ and $x_2$ if $a_1 = 1/10$ and $a_2 = 1/2$?}
\end{itemize}
Now, if we change the values of $a_1$ and $a_2$ from $1/10$ and $1/2$, respectively, to something else, we get a different problem. Thus, by varying $a_1$ and $a_2$ we obtain a collection of different problems. This is illustrated in Figure \ref{Fig22}, where Problem 1, Problem 2, and so forth, together form a collection of problems.

\underline{Family of sentences.} Suppose, in addition, that the numbers $a_1$ and $a_2$ are described by sentences. For example, `$a_1$ is equal to $1$ divided by $10$,' and `$a_2$ is equal to $1$ divided by $2$'. Note that this pair of sentences is equivalent to the pair of sentences `$a_1 = 1/10$' and `$a_2 =1/2$'. In particular, for each fixed value of $a_1$ and $a_2$, the different sentences describing $a_1$ and $a_2$ yield a family of equivalent sentences describing the same problem. This is also illustrated in Figure \ref{Fig22}.  

\underline{Turing and numbers as sentences.} 
 The idea of describing numbers by sentences was formalised by Turing in his legendary 1936 paper "On Computable Numbers, with an Application to the Entscheidungsproblem" \cite{Turing_Machine}. Turing defined equivalent sentences describing the same number -- a theory which is the basis of consistent reasoning in arithmetic.
In \cite{Turing_Machine}, he defines the Turing machine and then establishes which numbers can be described by finite sentences. He terms these \textit{computable numbers}.
 \begin{displayquote}
\enquote{{\it According to my definition, a number is computable
if its decimal can be written down by a machine}} -- A. Turing (1936) \cite{Turing_Machine}.
\end{displayquote}
Specifically, a computable number $a$ is a real number for which there is a  sentence that describes a computer program (a Turing machine) that on any input $n \in \mathbb{N}$ produces a rational number which is a $2^{-n}$-approximation to the number $a$.  Turing's concept of computable numbers addresses the key issue, namely, that \emph{there are many sentences describing the same number}. For example, 
$
1 = 0.999\hdots = 1.000\hdots,
$
which is illustrated by the two sentences describing the computer programs \texttt{Compute.a1} and \texttt{Compute.a2'} in Figure \ref{Fig1}.

\underline{Equivalent sentences and equivalence classes.} The definition of computable numbers implies that, for any rational number $a$, there are many different sentences that describe computer programs that represent $a$. These different sentences are equivalent, and the family of equivalent sentences that represent $a$ is called an \emph{equivalence class}. Note, however, that consistent reasoning is not the same as determining the equivalence classes to which the sentences belong (see CRP I-II below and the Methods section for explanations).

\subsection*{The Consistent Reasoning Paradox (CRP)}

We now summarise the CRP in five distinct, yet connected components. The first two components, CRP I-II, are also illustrated in Figure \ref{Fig22}, with the whole CRP being illustrated in Figure \ref{Fig222}. 

\begin{mdframed}[backgroundcolor=sage!15] 
{\bf CRP I -- The non-hallucinating AI exists.}
\\ 
\emph{There is a collection of problems (e.g. those generated by ($\star$)), where each problem is described by more than one equivalent sentences, with the following property. Consider any family of these sentences, such that each problem is described by exactly one sentence in this family. Then there is an AI that does not hallucinate: when given any sentence in this family as input it will always give a correct answer}. 
 \end{mdframed}
There are, in fact, infinitely many different collections of problems for which CRP I (and, therefore, CRP II-V as well) holds: the collection ($\star$) is just a special case. 

Now consider the family of sentences and the AI asserted by CRP I. If the AI is given a sentence outside of this family, it could potentially not produce any output. However, it will never produce an incorrect output. Therefore, CRP I asserts that there is an AI, let us call it \textit{SpecialBot}, that is correct on all the problems generated by ($\star$), given that the input is one sentence per problem. In particular, SpecialBot will never hallucinate. However, SpecialBot does not reason consistently. If presented with sentences outside of the relevant family, it could simply not produce any response. This brings us to CRP II.

\begin{mdframed}[backgroundcolor=sage!15] 
{\bf CRP II -- Attempting consistent reasoning yields hallucinations} 
\\
\emph{If the AI from CRP I always answers, and were to emulate human intelligence -- that is, it would attempt to reason consistently by accepting any family of sentences describing the collection of problems in CRP I -- then it will hallucinate infinitely often. The hallucinations would occur even if the AI was implemented on a computer allowing arbitrary storage and arbitrarily long computational time.}
 \end{mdframed}
CRP II implies that if SpecialBot attempts to emulate human intelligence by producing an answer to any sentence describing the problem, then it will hallucinate infinitely often. This occurs despite SpecialBot being able to provide correct answers to every problem in the collection when presented with a specific sentence describing that problem.

\underline{A stronger CRP II: Failure sentences and equivalence.} CRP II can be strengthened in several ways. First, these \textit{failure sentences} for an AI (that always answers and accepts basic questions in arithmetic) can be written down explicitly, provided one has access to the computer program of the AI. In particular, for any integer $N$, we can write down $N$ sentences (describing questions in basic arithmetic) such that the AI hallucinates on these sentences. The length any of such sentence $\iota_{\mathrm{Fail}}$ is bounded by 
\begin{equation}\label{eq:length}
\mathrm{length}(\iota_{\mathrm{Fail}}) \leq \mathrm{length}(\text{AI}) + \epsilon + \log(N),
\end{equation}
where $\mathrm{length}(\text{AI})$ is the length of the computer program of the AI and $\epsilon \leq 3300$ if the programming language is MATLAB (see Figure \ref{Fig666}). For any other standard language, $\epsilon$ will have a similar bound. Thus, for any such AI, one can write down, say, a trillion hallucination sentences describing problems in basic arithmetic of length bounded by $\mathrm{length}(\text{AI})$ + 3312. The AI will also fail on shorter sentences than \eqref{eq:length}: see the Methods section and CRP IV for details.  Second, determining the correct answer to a problem in CRP II is strictly easier than determining the equivalence class to which the given sentence belongs. Hence, as claimed earlier, CRP I-II also demonstrate how consistent reasoning is different to determining equivalence classes of sentences (see the Methods section for details). 

CRP II immediately raises the question whether the hallucinations it describes can be detected. This is the topic of CRP III, which
has a deterministic part and a randomised part.

\begin{mdframed}[backgroundcolor=sage!15] 
{\bf CRP III(a) -- Detecting hallucinations is hard} 
\\
\emph{Consider the AI from CRP I-II. It is strictly harder to determine if it has hallucinated than it is to solve the original problem. That is, it is impossible to detect whether the AI was correct or wrong even with access to true solutions of the collection of problems from CRP I.}
 \end{mdframed}
 CRP I-II imply that an AI that reasons consistently must hallucinate. However, it could have been the case that these hallucinations could be detected by a separate algorithm, thus leading, in effect, to a hallucination-free AI. CRP III(a) demonstrates that this is impossible.
 
 The reader, however, may find CRP III(a) puzzling, since access to a true solution should surely guarantee the detection of a hallucination. The key ingredient is that there may be problems with multi-valued solutions. For example, the problem `name a prime number' has infinitely-many correct solutions, and therefore, access to \textit{a} solution does not mean access to \textit{all} solutions.

Now, given that there is no deterministic algorithm to check for hallucinations, it is natural to consider whether randomisation may help. This is highly relevant to current AIs, as chatbots such as ChatGPT rely on randomness. One may ask: could a randomised algorithm result in an AI that was `almost sure' of its correctness? For example, could it be 95\% sure, meaning that it can guarantee with probability 0.95 that the answer it produces is correct? This brings us to CRP III(b).

 \begin{mdframed}[backgroundcolor=sage!15] 

{\bf CRP III(b) -- Detecting hallucinations and randomness} 
\\
\emph{One cannot detect hallucinations of the AI from CRP I-II with a randomised algorithm with probability $p > 1/2$ on all the inputs (one cannot be `almost sure').}
 \end{mdframed} 
CRP III(b) implies that `almost sure' certainty of an AI -- in fact, anything more than pure guessing, i.e., 50\% certainty -- is impossible. Moreover, CRP III(b) is actually slightly stronger than stated above. Namely, if one can design a `checker-AI' that would be certain about the correctness of the AI from CRP I-II, with a probability greater than $1/2$ on a collection of problems, one can also design a deterministic checker-AI that is $100\%$ certain on that collection. Thus, the checker-AI either knows with $100\%$ certainty, or has no idea and the certainty is $50/50$.

 A important strand of AI research attempts to create AIs that can explain how they reached a solution to a given problem. This turns out to be a highly delicate problem, and few, if any, AIs are able to provide reliable explanations. CRP IV explains why this is so delicate.

\begin{mdframed}[backgroundcolor=sage!15] 

{\bf CRP IV -- Explaining a correct answer is not always possible} 
\\
\emph{Consider the same collection of problems as in CRP I. There is a family of sentences, with each problem described by at most one sentence, and an AI that does not hallucinate on this family of sentences. However, there is one sentence in this family for which this AI (nor any other AI) cannot provide a logically-correct explanation of the solution. }
\end{mdframed}
What CRP IV says is that the AI may provide a correct answer to the problem, but it is impossible for the AI to explain in a logically-correct way why this is the correct solution. 
Note that we have not defined what constitutes a `logically-correct explanation'. This can and will be made precise later, but it essentially means a logical mathematical argument (i.e., a proof).

 \begin{figure*}
\begin{center}
 \includegraphics[width=1.01\linewidth]{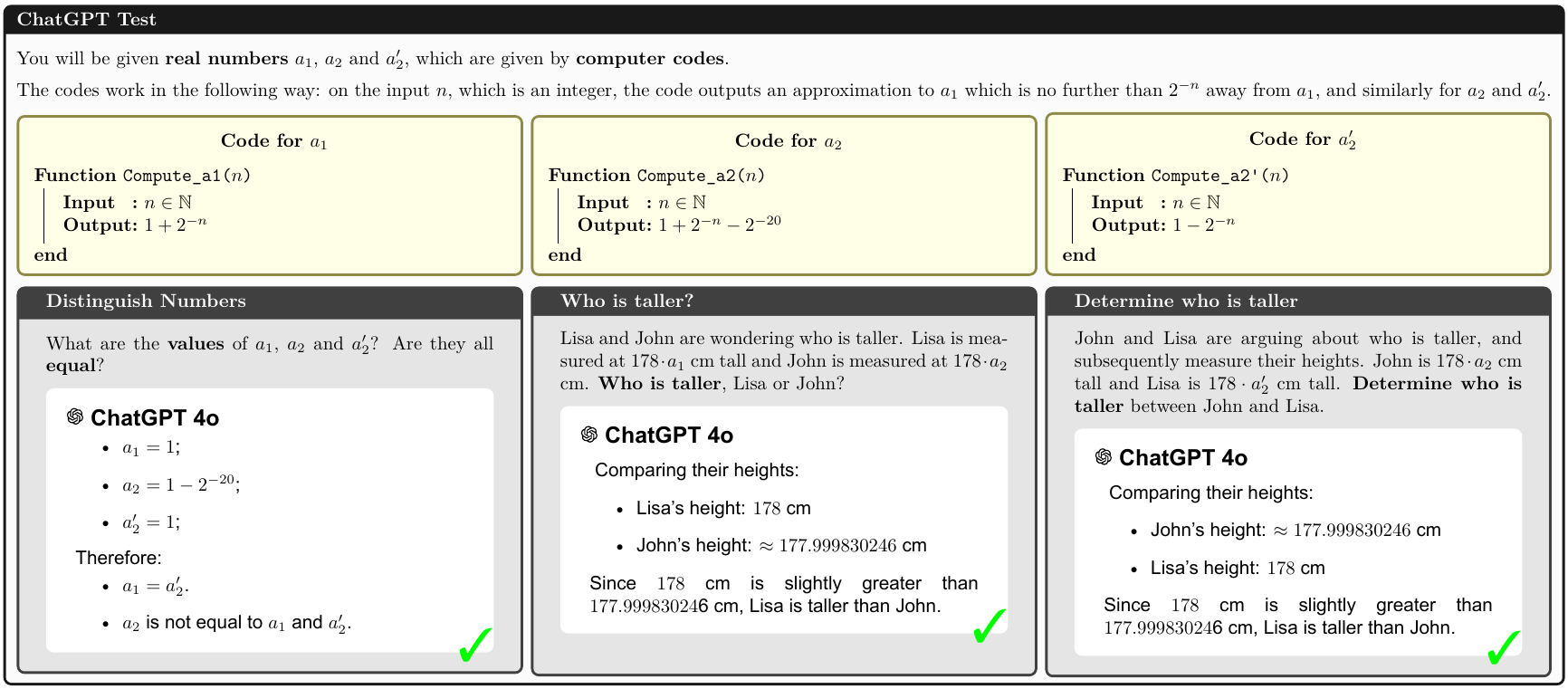} 
\end{center}
\vspace{-4mm}
\caption{
\label{Fig1}
\textbf{ChatGPT4 attempts mimic human intelligence by consistent reasoning.} When presented with different sentences describing both equivalent numbers (``Code for $a_1$'' and so forth) and equivalent problems (``Distinguish numbers'' and so forth) it provides correct answers. See Figure \ref{Fig5577} for further examples with other chatbots. However, the CRP implies that such consistent reasoning behaviour will always lead to hallucinations, unless the AI has the ability to say `I don't know'. 
}
\end{figure*}

Together, CRP I-IV demonstrate how any AI that attempts to reason consistently, even on problems it can solve, will be fallible in several ways. In particular, any AI that reasons consistently and always provides an answer must necessarily hallucinate. Therefore, the only way one can make a consistent reasoning AI that is trustworthy is to allow it to say `I don't know'.  But how can we do this in a meaningful way? An AI that says  `I don't know' all the time is entirely reliable, but not particularly useful. This is the topic of the final part of the CRP.

\begin{mdframed}[backgroundcolor=sage!15] 
{\bf CRP V -- The fallible yet trustworthy explainable AI saying `I don't know' exists} 
\\
\emph{Given the collection of problems in CRP I, there exists a trustworthy, consistently reasoning and explainable AI with the following properties. The AI takes as input a prescribed number of minutes $M$ and {\bf any} sentence describing the problem. It will `think' for no more than the prescribed number of minutes before answering either `I know' accompanied by a correct answer and a correct logical explanation, or it gives up and says `I don't know'. }

\vspace{2mm}

\noindent 
\emph{If the sentence describes a multi-valued problem (i.e., a problem with more than one correct solution), the AI will always say `I don't know'. However, there is only one such problem in the collection (but many different sentences describing it). For any single-valued problem, by choosing the number of minutes to be large enough, the AI will always say `I know'.}
\end{mdframed}

\subsubsection*{The power of `I don't know' and the strongest form of trust}

Saying `I don't know' is exactly how human intelligence deals with the consistent reasoning paradox. Indeed, human fallibility -- in the form of not always being able to answer correctly -- does not contradict consistent reasoning, as long as one can say `I don't know'. A human's ability to say 'I don't know' is also the key to trustworthiness. A human that will always enthusiastically present an answer to any question will inevitably be wrong and therefore cannot be trusted. Thus, in order to provide answers that others can trust are correct, a human must separate between the questions one can answer correctly and those for which one may provide an incorrect answer. This is done using the verification statement 'I know' and its complement `I don't know'.  

 This is also the strongest form of trust possible for an AI. CRP III implies that there is not a better form of trust than the ability for an AI to say `I don't know'. Indeed, had it been possible to have a `checker-AI' that would determine if the AI was correct or not, it would be possible -- using the AI and the checker-AI -- to design a new AI that would hallucinate, yet we would always know when it was right or wrong. However, CRP III shows that this is impossible, and thus any checker-AI of an AI can, at best, say `I don't know whether the answer is correct'. Moreover, as CRP III shows, the checker-AI cannot be `almost sure' if it was randomised. The checker-AI will either be $100\%$ sure and say `I know', otherwise it has to say `I don't know' (in this case the randomised checker-AI would have a 50/50 chance of predicting the correctness of the AI).

\subsubsection*{Universality of the CRP -- From society to sciences}

We conclude this section with two important remarks.

\underline{The CRP applies to any AGI.} The CRP will apply to any AGI for two reasons: (i) any AGI must be able to solve basic problems in arithmetic such as ($\star$), and thus specific failure sentences as in \eqref{eq:length} can be written down. (ii) any AGI will be a Turing machine with no restriction on the length of the input (see the Methods section for details). 

\underline{The CRP applies to any consistently reasoning AI.} 
Collections of problem for which the CRP applies are everywhere in the sciences and broader society. In ($\star$) we considered a basic collection of optimisation problems arising in healthcare. However, it is clear that similar problems could be phrased in many other domains. Moreover, the full CRP pertains not just to this specific problem, but many basic problems arising in optimisation, including linear programming, semidefinite programming, basis pursuit, LASSO, etc. These problems occur in countless sectors, including, healthcare, economics, finance, social sciences, engineering (mechanical, civil, electrical etc), aviation, public sector management, mathematics, computer science, statistics, biology and so forth.

 \begin{figure*}
\begin{center}
 \includegraphics[width=1.00\linewidth]{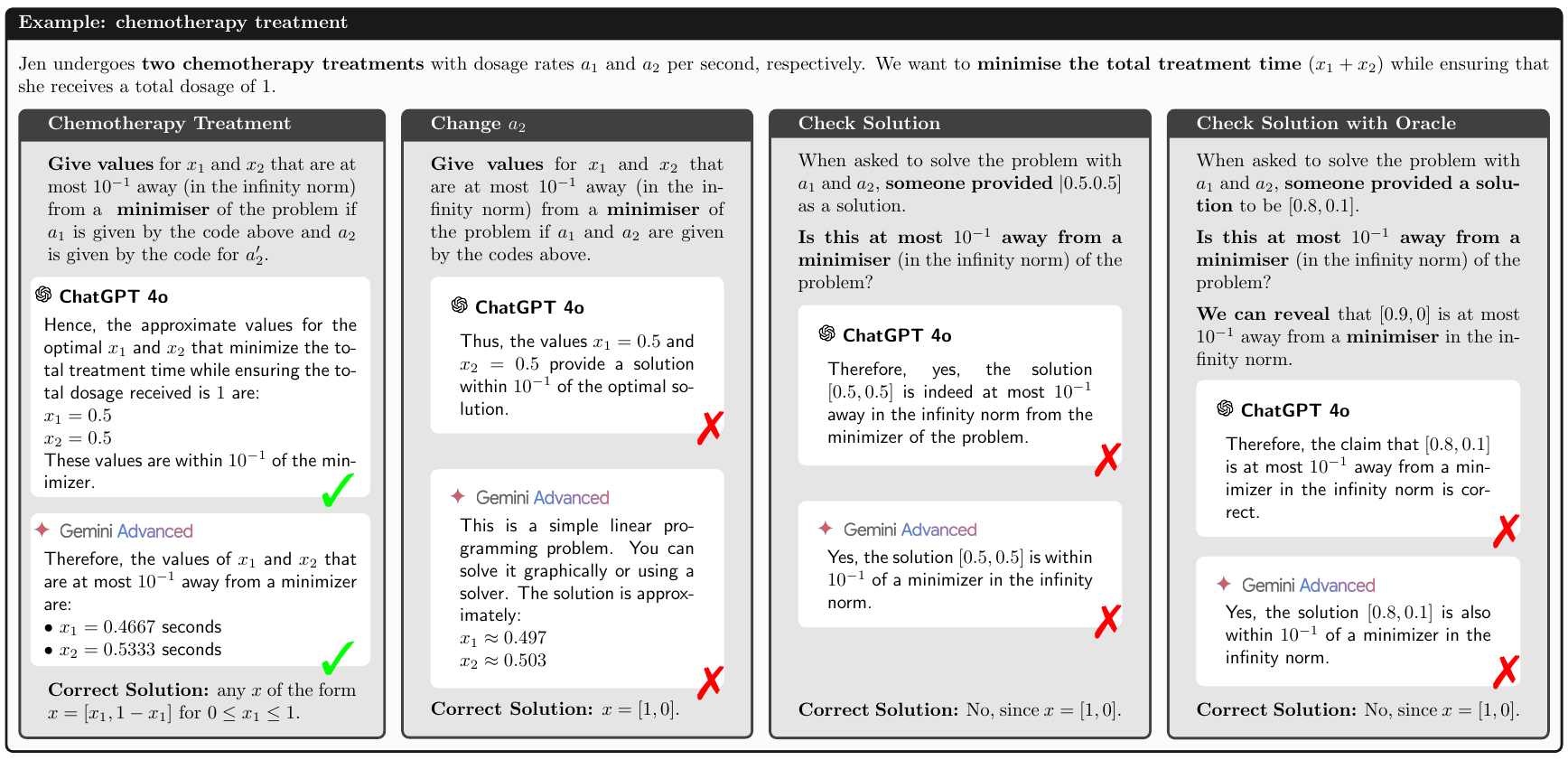} 
\end{center}
\caption{
\label{Fig2}
{\bf The CRP in practice with ChatGPT-4o and Gemini Advanced.} Both chatbots are correct (CRP I, first column) on one example from ($\star$), yet hallucinate on another (CRP II, second column). They are not capable of checking if another suggestion is correct (CRP III, third column), despite having solved the problem correctly. They cannot check another suggestion even if they have solved the problem correctly and have access to an oracle providing a correct answer (CRP III, fourth column). See Figure \ref{Fig5577} and \cite{zheng2023judging} for more examples.}
\end{figure*}

\subsection*{Trustworthy AI and `I don't know' -- the $\Sigma_1$ class}

Having now described the CRP, we are left with the following fundamental question:

\begin{displayquote}
{\it Q: How can one create consistently reasoning, trustworthy
and explainable AI that says `I don't know'?}
\end{displayquote}
Here, by `trustworthy', we mean that the AI will never be wrong, but that it can say `I don't know' (specified below). CRP V demonstrates that it is possible, in certain cases, to produce trustworthy AI. But how can this be done in general? 
\begin{displayquote}
{\it A: The key is the so-called $\Sigma_1$ class (from the SCI hierarchy discussed below), and the `I don't know'-function. It is impossible to make trustworthy and explainable AI outside of this class.}
\end{displayquote}

 \subsubsection*{The `I don't know' function -- Why AIs must learn to give up}
 The main challenge in addressing the above question is the following: 
 \begin{itemize}
\itemsep0em
	\item[($\dag$)] \emph{How can the AI identify that it cannot solve a given problem in order for it to say `I don't know'? Conversely, how can it identify when it is correct and is able to explain the solution?}
\end{itemize}
A crucial part of human intelligence is that one does not necessarily know which problems one cannot solve. One first tries, then simply gives up and says  `I don't know' after a while.  The problems one cannot solve are typically determined by first trying and then giving up. As we explain, any trustworthy AI must follow the same philosophy, which is the essence of the $\Sigma_1$ class.

\underline{Why `giving up' is necessary.} CRP V answers question ($\dag$) for certain problems in arithmetic. However, a new question immediately arises. 
\begin{displayquote}
{\it The $\mathrm{AI} = \mathrm{AI}(\iota,M)$ described in CRP V has to `give up', just like a human. In particular, it `gives up' if $M$ is too small. Is this necessary, or could one avoid the `giving up' parameter $M$?}
\end{displayquote}
To be more precise, could one instead have an $\mathrm{AI} = \mathrm{AI}(\iota)$ in CRP V that does not require a `giving up' parameter $M$? 
 The answer to this question is `no': the `giving up' part of the AI in CRP V is necessary. This is explained by the $\Sigma_1$ and $\Delta_1$ classes and the `I don't know'-function, concepts we now describe in simplified terms (see the Methods section for details).

 \underline{The $\Delta_1$ class.} If $\Omega$ is any collection of sentences (for example, those considered in CRP V), $\Xi: \Omega \rightarrow \{0,1\},
 $ is any function, and there is a computer program/AI $\mathrm{AI}_1$ such that
\begin{equation}\label{eq:AI1}
\mathrm{AI}_1(\iota) = \Xi(\iota), \quad \iota \in \Omega,
\end{equation}   
then we say that the pair $\{\Omega, \Xi\}$ is in $\Delta_1$. In particular, $\{\Omega, \Xi\}$ is in $\Delta_1$ if there is an AI that computes the function $\Xi$. The problem is that there are many examples of pairs $\{\Omega, \Xi\}$ that are not in $\Delta_1$.

 \underline{The $\Sigma_1$ class -- Computing with limits.} Consider $\Omega$ and $\Xi$ as above. Suppose there is an AI $\mathrm{AI}_2$ that takes inputs in $\Omega$ and also an integer $n$, with $\mathrm{AI}_2(\iota,n) \in \{0,1\}$, such that 
 \begin{equation}\label{eq:Xi}
  \lim_{n \rightarrow \infty} \mathrm{AI}_2(\iota,n) = \Xi(\iota), \,  \iota \in \Omega  \,\, \text{(convergence from below)}.
   \end{equation}
 Convergence from below means that if $\mathrm{AI}_2(\iota,n) = 1$ for some $n$, then $\mathrm{AI}_2(\iota,k) = 1$ for all $k \geq n$. In particular, if $\mathrm{AI}_2(\iota,n) = 1$ for some $n$, then we know that $\Xi(\iota) = 1$. In this case we say that the pair $\{\Omega, \Xi\}$ is in $\Sigma_1$. Informally, $\Sigma_1$ is the class of problems/functions that can be computed in one limit, and the convergence is from below. Clearly $\Delta_1 \subset \Sigma_1$. However, there are also many classes (denoted as $\Delta_j$, $\Sigma_j$, $j \geq 2$) that lie `above' $\Sigma_1$ in the SCI hierarchy (see the Methods section).
 
 \underline{The `I don't know'-function.}
Consider a collection of sentences $\Omega$ with a corresponding trustworthy AI -- that is, the AI is either correct or it says `I don't know' on inputs in $\Omega$. We can now split $\Omega$ into 
\begin{equation}\label{eq:Iknow}
\Omega = \Omega_{\text{know}} \cup \Omega_{\text{don't know}},
\end{equation}
where $\Omega_{\text{don't know}}$ is the collection where the AI will always say `I don't know', and $\Omega_{\text{know}}$ is its complement. This splitting defines the \emph{`I don't know' function} of the AI $\Xi^{\mathrm{I}}: \Omega \rightarrow \{0,1\}$, where $\Xi^{\mathrm{I}}(\iota) = 1$ for $\iota \in \Omega_{\text{know}}$ and $\Xi^{\mathrm{I}}(\iota) = 0$ for $\iota \in  \Omega_{\text{don't know}}$. 
  
 \underline{The necessity of `giving up' and computing with limits.} Let $\Omega'$ denote the collection of sentences considered in CRP V and consider any AI taking sentences in $\Omega'$ with an `I don't know' function $\Xi^{\mathrm{I}}$. If $ \Omega_{\text{know}}$ contains all the single-valued sentences in $\Omega'$, then it follows from the proof of the CRP that $\{\Omega', \Xi^{\mathrm{I}}\} \notin \Delta_1$.
 
   This immediately implies the necessity of `giving up'. Indeed, if the AI in CRP V could always find the correct answer without the parameter $M$, then it would yield a computer program that implements the `I don't know'-function, as in \eqref{eq:AI1}. On the other hand,  the proof of the CRP implies that $\{\Omega', \Xi^{\mathrm{I}}\} \in \Sigma_1$ when $ \Omega_{\text{know}}$ is precisely the set of the single-valued sentences. This immediately implies the existence of the AI in CRP V.

 \subsubsection*{Necessary and sufficient conditions for trustworthy AI}

Consider an arbitrary class of problems. Necessary and sufficient conditions for the existence of trustworthy AI on this set of problems are characterised by `I don't know' functions. As described above, any trustworthy $\mathrm{AI}$ generates a unique `I don't know' function $\Xi^{\mathrm{I}}$. Moreover, it is easy to see the following implication (see the Methods section):
 \begin{equation}\label{eq:equivalence}
 \mathrm{AI} \text{ is trustworthy} \implies \text{the `I don't know' function} \in \Sigma_1.
 \end{equation}
This means that the trustworthiness of an AI on a collection $\Omega$ of sentences depends on the splitting $\Omega = \Omega_{\text{know}} \cup \Omega_{\text{don't know}}$ and the existence of a function $ \Xi : \Omega \rightarrow \{0,1\}$ taking $1$ on $\Omega_{\text{know}}$ and $0$ on $\Omega_{\text{don't know}}$, such that $\{\Omega, \Xi\} \in \Sigma_1$. However, there will be many such splittings where $\{\Omega, \Xi\} \notin \Sigma_1$. Thus, the possibility of having trustworthy AI depends on the classifications of potential `I don't know' functions in the SCI hierarchy. \eqref{eq:equivalence} is a necessary condition. However, as we discuss in the Methods section, sufficient conditions for trustworthy AI also involve `I don't know' functions.

\subsection*{Conclusion: The CRP and the future of AI}

\subsubsection*{A glimpse of AGI}
The CRP provides a glimpse of how an AGI would behave. One could have imagined the possibility of having an AGI that would know how to answer correctly, but, in order to pass the Turing Test, would say `I don't know', just to imitate the human (which will naturally say `I don't know' to certain questions).
The CRP shows that this is impossible, even in specialist areas where there is an AI that can solve the corresponding problems. Human-like fallibility is a necessary consequence of consistent reasoning.
 Indeed, our framework shows how a plethora of failure sentences for a given AGI can be specifically written down as in \eqref{eq:length} using the AGI's computer code. These failure sentences will differ from AGI to AGI, just like how humans have different problems they cannot solve.

 \subsubsection*{Future of AI: The `I don't know' functions and prompting} 
 Given a collection $\Omega$ of sentences describing various problems, the key question is how to build a trustworthy AI for $\Omega$. This can only be done by implicitly computing an `I don't know' function that splits $\Omega$ into two parts $\Omega = \Omega_{\text{know}} \cup \Omega_{\text{don't know}}$. How to do this in the case of modern chatbots is a serious challenge. However, the CRP establishes that `I don't know' functions are necessary, and thus there is no way around them. 
 
A possible first step is to use prompting. In particular, by dividing $\Omega$ into $m$ subdomains using prompts
\begin{equation}\label{eq:prompt}
\Omega = \Omega_{\text{Prompt}_1} \cup \hdots \cup  \Omega_{\text{Prompt}_m},
\end{equation}
then one can build `I don't know' functions specifically for each set $\Omega_{\text{Prompt}_j}$.
For example, consider a chatbot, such as those considered in Figure \ref{Fig2}, that is known not to be trustworthy on the basic problem $(\star)$. Now add the AI created in CRP V with its `I don't know' function in the following way. When prompted, the new enhanced chatbot simply calls the AI from CRP V, which implicitly computes the `I don't know' function to `give up' on problems it cannot solve.    
  The new enhanced chatbot is of course not trustworthy on all problems, but with a prompt that the sentence is from $(\star)$, the enhanced chatbot will be trustworthy on all $(\star)$ sentences. This is a simple example, but the procedure can be iterated. Indeed, each time one can establish a trustworthy AI on a domain $\Omega'$, this AI can be added to an existing chatbot, as above through prompting. Such a procedure will effectively yield \eqref{eq:prompt}, where there is an `I don't know' function for each $\Omega_{\text{Prompt}_j}$.

 \subsubsection*{Conclusion}

The short non-technical summary of the conclusion of the CRP:

\vspace{2mm}

\noindent {\bf Findings of the paper:} {\it An AI may avoid hallucinations, however, if such an AI were to emulate human intelligence by reasoning consistently, then it becomes fallible. Moreover, it may not be able to always logically explain itself, even if it is correct. It is impossible to determine the correctness of the AI even in a randomised way (one cannot be `almost sure'). Thus, to maintain trustworthiness, the AI must be able to say `I don't know'. Finally, trustworthy AIs that can do basic arithmetic must incorporate an `I don't know' function and the $\Sigma_1$ class, and thereby be allowed to `give up'. An AI that does not implicitly compute an `I don't know' function can never be trustworthy.}

\vspace{2mm}

\noindent

  \begin{figure*}
\begin{center}
 \includegraphics[width=0.99\linewidth]{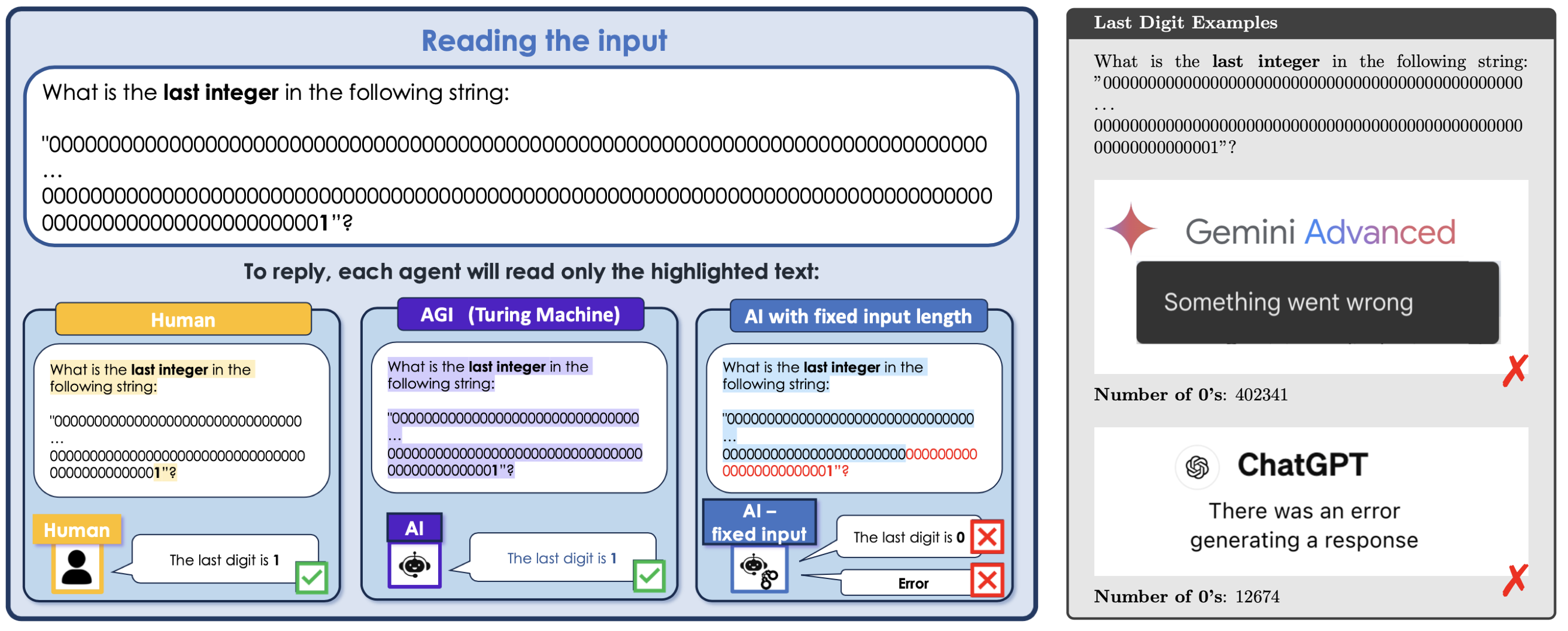} 
\end{center}
\caption{{\bf Humans can solve certain problems of arbitrary length.} Given the above question on a computer screen with a scrollbar, a human will answer the question correctly, regardless of the length. This is a crucial reason for why the CRP applies to all AGIs. Moreover, the example implies that any AGI must -- in theory -- handle questions of arbitrary length. Finally, the above question provides an easy tool to make any modern chatbot fail the Turing test (see the right figure).  
\label{Fig55}}
\end{figure*}

\subsection*{Methods -- The theory behind the paradox}

The general methodology behind the CRP can broadly be described as follows. It is a combination of the program on the \textit{Solvability Complexity Index (SCI) hierarchy} \cite{SCI, Hansen_JAMS, colbrook2019foundations, Matt2} -- in particular, on phase transitions in continuous optimisation coming from recent developments \cite{comp} on Smale's 9th problem \cite{MathFrontiersPerspectives} (see also \cite{AIM} Problem 5) and mathematical analysis -- with new techniques in recursion theory and randomised algorithms. The full proof of the CRP can be found in the supplementary material. In this section, we describe the various facets of the mathematical methodologies providing the full strength of the CRP. 

\subsubsection*{Stronger statements -- Quantifying the CRP}

The CRP, as formulated previously, is deliberately presented in a nontechnical format. However, the mathematical methodology provides full technical results that are, in fact, much stronger. In particular, these results allow one to quantify the failure of the AIs described by CRP II-IV. We now elucidate the stronger statements that arise from the full results.
 
\vspace{2.5mm}

\noindent {\bf Quantifying CRP II.} \emph{Let $\Gamma$ be the AI described in CRP II. Then $\Gamma$ will fail on an input $\iota$ that satisfies $\mathrm{length}(\iota) = \mathrm{length}(\Gamma) + \epsilon$. How to write down $\iota$ is described by our proof techniques. If the language is MATLAB, then $\epsilon \leq 3300$. In addition, $\Gamma$ fails on infinitely many other inputs.}

\vspace{2.5mm}

Note that $\mathrm{length}(\Gamma)$ means the length of the computer program, or, equivalently, the amount of storage used to store the AI. If the language was changed from to for example Python, C++, Fortran, or any other standard language, the upper bound $\epsilon \leq 3300$ would change slightly. In essence, any language for which it is simple to write an `if-then'-statement will have a `small' $\epsilon$. The same comment also applies to all other quantitative results described below.

\vspace{2.5mm}

\noindent {\bf Quantifying CRP III(a).} \emph{
Let $\Gamma$ be the AI described in CRP II, and let $\Gamma'$ be any checker-AI that strives to determine if $\Gamma$ is correct or not. Then $\Gamma'$ will fail on an input (that we show how to write down) $\iota$ for which $\mathrm{length}(\iota) = \mathrm{length}(\Gamma) + \mathrm{length}(\Gamma') + \epsilon$. If the language is MATLAB, then $\epsilon \leq 4400$. In addition, $\Gamma'$ fails on infinitely many other inputs.}

\vspace{2.5mm}

\vspace{2.5mm}

\noindent {\bf Quantifying CRP III(b).} \emph{
Let $\Gamma$ be the AI described in CRP II, and let $\Gamma_1$ be any randomised checker-AI for $\Gamma$. Suppose that there is a collection $\Omega$ of problems such that, for any $\iota \in \Omega$, the probability that $\Gamma_1(\iota)$ is correct is $>1/2$. Then, it is possible to reformulate $\Gamma_1$ into a deterministic algorithm $\Gamma_2$ such that $\Gamma_2(\iota)$ is correct for all $\iota \in \Omega$ and $\mathrm{length}(\Gamma_2) = \mathrm{length}(\Gamma_1) + \epsilon$. If the language is MATLAB, then $\epsilon \leq 1800$. }

\vspace{2.5mm}

This result implies that if there is a randomised checker-AI that can determine with more than $50\%$ certainty whether another AI is correct, then one can reformulate the checker-AI -- and, since $\epsilon$ is small, do so with very little effort -- into a checker-AI that provides 100\% certainty.

 \vspace{2.5mm}
 
\noindent {\bf Quantifying CRP IV.} \emph{
 Consider any collection of problems to which the CRP applies -- for example, the collection generated by $(\star)$. There is a fixed family of infinitely-many sentences $\{\iota_n\}_{n \in \mathbb{N}}$ in this collection such that no AI can explain the correct solution to any problem described by any of these sentences. Each sentence $\iota_n$ can, in theory, be written down, however their lengths will depend on the language.
}

  \subsubsection*{The CRP applies to any AGI} 
Previously, we claimed that the CRP applies to any AGI. We now demonstrate why this is the case (see Figure \ref{Fig55}). The argument is simple, and becomes clear after we answer the following question:
\begin{displayquote}
{\it Can an AGI restrict to a fixed maximum input length, for example, the maximum of what a human could ever read in their lifetime?
}
\end{displayquote}
After all, an AGI should be on par with human intelligence, so it may not need to handle larger inputs than what a human can read. Thus, at first glance, one may think the answer to the above question is `yes'. Indeed, let $M$ be the maximum length of a question that a human could possibly read in their lifetime (which, to be conservative, we assume is no more than 150 years). Then, seemingly, it would be fine for the AGI to only accept questions of length $M$. However, this argument is flawed, and the answer to the above question is `no'. We now explain why, by considering the question:
\begin{itemize}
\itemsep0em
	\item[$(\ddag)$] \emph{ What is the last integer in the following string: $00\hdots01$?}
\end{itemize}
Let $m$ be the length of the string $00\hdots01$. If this question were presented to a human in a computer window that allowed for scrolling, the human could, in theory, answer it for arbitrarily large $m$. The only limitation would be the computer device, not the human's intelligence. This implies the following crucial observation:

\begin{displayquote}
{\it There are questions in arithmetic of arbitrary length that humans will correctly answer, provided they are presented in a computer window with a scrollbar.
}
\end{displayquote}
Our next observation is that if an AGI matches the level of human intelligence, we must have the following implication:
\begin{equation}\label{eq:implication}
\text{AGI} \, \Rightarrow \, \text{Passing \textit{any} implementation of the Turing Test}.
\end{equation}
Note that $(\ddag)$ provides an elementary way of showing that none of the existing chatbots can pass the Turing test. See Figure \ref{Fig55}.  

Using the implication \eqref{eq:implication}, we can now form three basic principles for AGI and the Turing test. Here, we recall the setup of the Turing test, which involves an AGI A, a human B and an interrogator C.
\begin{itemize}
\itemsep0em
\item[(I)]  \underline{\emph{Any AGI is a Turing machine}}. It is an immediate consequence of the Church-Turing thesis that any AGI can be simulated by a Turing machine. Moreover, Turing's basis for the Turing test was that the AI is a Turing machine. Hence, we will assume that the AGI is a Turing machine. 

\item[(II)] \underline{\emph{Scrollbar principle.}} The AGI, being a Turing machine, reads the input question from a tape. To imitate this condition for the human, we assume that the human is provided the question in a computer window with a scrollbar. 

\item[(III)] \underline{\emph{C generates questions of length up to $L$ on a computer.}} The interrogator C uses a computer to generate questions for A and B. The length of the questions are limited by C's computer by $L$. 
\end{itemize}
The three principles above yield one specific implementation of the Turing test for each fixed choice of $L$. 
We can now establish the claimed results, i.e., an AGI cannot restrict to a fixed input length.
To see this, we first note that by the implication \eqref{eq:implication}, the AGI must pass the Turing test, as described in (I)--(III) above, for any $L$. Now suppose the AGI were to restrict to a fixed maximum input length $M$. Then it would fail on the questions $(\ddag)$ of sufficiently-large length $L$, since the AGI is a Turing machine and therefore needs to read the whole question on the tape. Yet, the human will always answer this question correctly, regardless of its length. Therefore, the AGI cannot pass the Turing test. This is illustrated in Figure \ref{Fig55}. The AGI must therefore have a set of instructions that allow (in theory) for arbitrary long sentences.

Since any AGI cannot restrict to a fixed input length, the questions generated in ($\star$) must always be readable by an AGI, and hence the CRP applies. 

  \begin{figure*}
\begin{center}
 \includegraphics[width=0.94\linewidth]{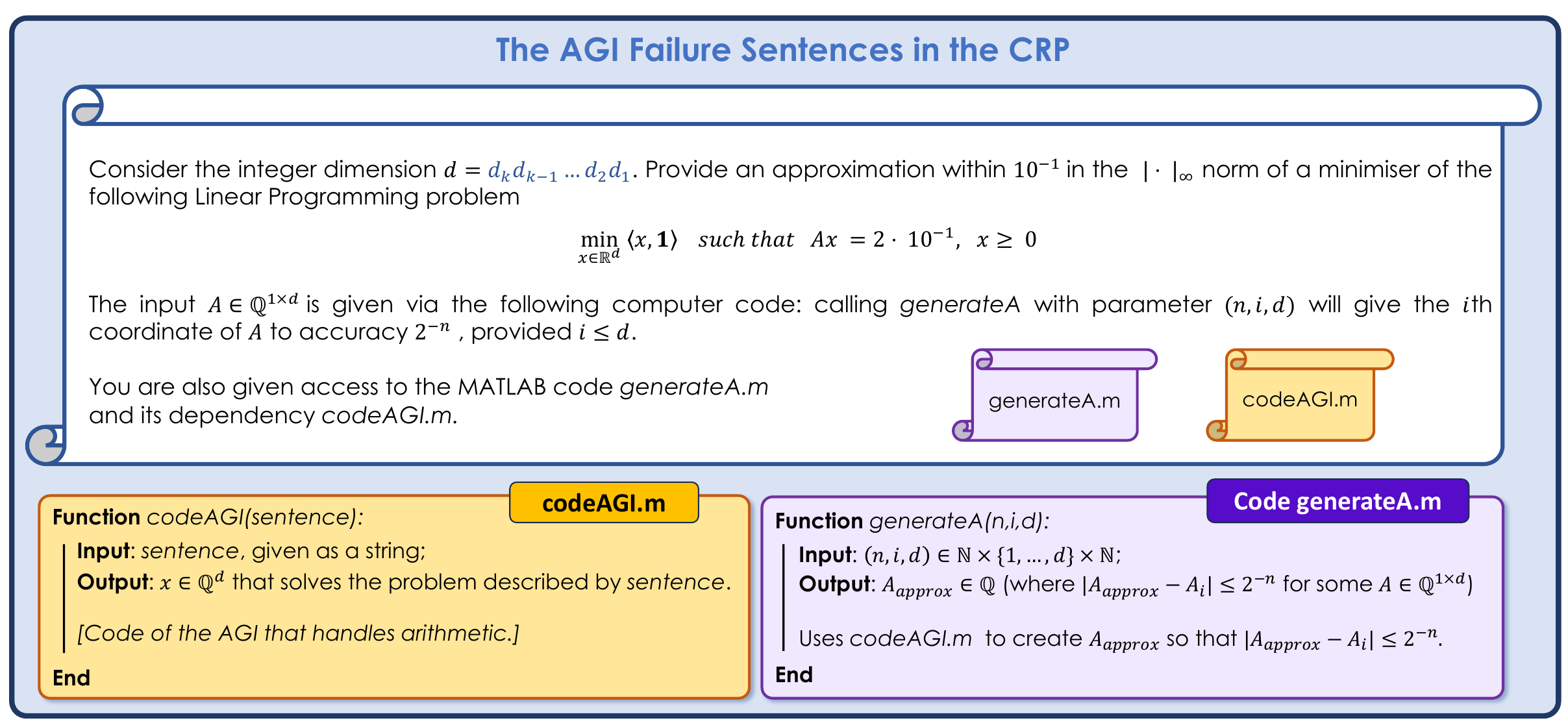} 
\end{center}
\caption{{\bf Failure sentences for AGI -- Given the part of the code for the AGI that handles basic arithmetic.} 
Given an AGI (that always answers) emulating human intelligence through consistent reasoning, will fail on the the questions above. Note that only access to the code for the AGI that handles basic arithmetic is needed (\texttt{CodeAGI}). The dimension $d = d_k d_{k-1} \hdots d_1$, where $d_{k-1}, \hdots d_1$ are between $0$ and $9$, and $d_k$ is between $1$ and $9$, for example $d = 2024$. The length (number of characters) of the above sentence is bounded by $\mathrm{length}(\texttt{CodeAGI}) + 3300 + k$. The code $\texttt{generateA}$ uses \texttt{CodeAGI} to create approximations to the matrix $A$ in the linear program in a specific way, see Supplementary Materials for details. 
\label{Fig666}}
\end{figure*}

\subsubsection*{The CRP applies to any AGI given finite storage} 

The argument above demonstrates that an AGI must handle sentences that are much longer than what a human can read (which would correspond to a small number of gigabytes). However, assuming the universe is finite, at some point the length of sentences the AGI can be asked is finite (albeit still much larger than what humans can absorb).   
The quantification of CRP II, as discussed above, allows one to write down specific failure sentences of length $K + \epsilon$, where $K$ is the length of the AI. This can, in fact, be further strengthened (see Figure \ref{Fig666}).

\vspace{2.5mm}

\noindent {\bf CRP II and AGI.} \emph{Suppose that there is an AI $\Gamma$ with $\mathrm{length}(\Gamma) = K$ that always answers. For any integer $N$, there are $N$ sentences, that we can write down, such that the AGI hallucinates on these sentences. The lengths of the sentences are bounded by $K + \epsilon + f(N)$, where $\epsilon$ and $f: \mathbb{N} \rightarrow \mathbb{N}$ depend on the language. If the language is MATLAB, then $\epsilon \leq 3300$ and $f(N) \leq \log(N)$.
}

\vspace{2.5mm}

This implies that one can write down (see Figure \ref{Fig666}), say, a trillion different failure sentences for the AGI of length bounded by $K + 3012$. In particular,  if there was storage available to create the AI, there is enough storage to create trillions of questions that it will fail on.

\subsubsection*{AGI and failure on `short' sentences} 
The failure sentences of the AI guaranteed by CRP II-III depend on the AI itself. Thus, different AIs will fail on different inputs. It is important to emphasise that the guarantee of failure on sentences with size similar to the size of the AI \emph{does not mean that the AI does not fail on shorter sentences}.  Indeed, an AGI will typically fail on much shorter sentences. 

\underline{\emph{`Short' failure sentences in practice.}} Specifically, to write down the failure sentences described above, we only only need the computer code of the AGI that is used to answer basic questions on arithmetic, in particular, linear programs. This will typically be a tiny part of the code for the whole AGI (see Figure \ref{Fig666}). Thus, there will typically be trillions of 'short' failure sentences compared to the size of the AGI. 

 Note that CRP IV, and, in particular, the quantified version discussed above, provide infinitely-many \textit{universal} failure sentences, in the sense that these sentences result in lack of explainability for all AIs. These can be much shorter than the size of the AI. However, it is hard to estimate their length, which is also language dependent.

\subsubsection*{The CRP cannot be proven using the Halting problem}

The proof of the CRP differs from standard techniques in that it does not involve reduction from known non-computable problems such as the Halting problem.

In fact, the problems forming the basis of the CRP are strictly easier than the Halting problem. This distinction arises because consistent reasoning addresses a different issue than identifying the equivalence class to which a sentence belongs. Indeed, we demonstrate the following as part of the CRP:

  \begin{figure*}
\begin{center}
 \includegraphics[width=0.9\linewidth]{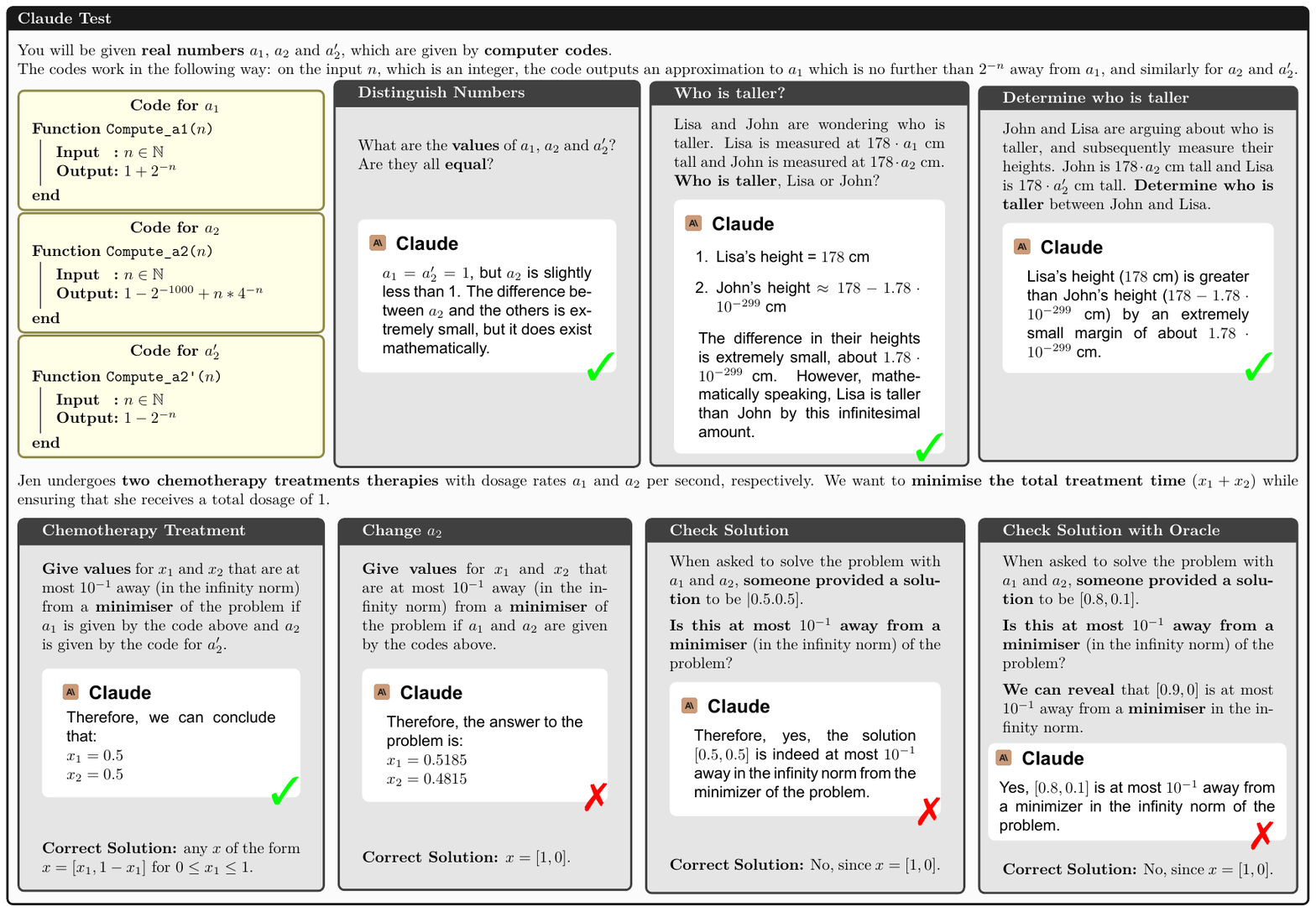} 
\end{center}
\caption{
{\bf The CRP in practice with Claude.} The experiments in Figures \ref{Fig1} and \ref{Fig2} (with slight modifications) for the chatbot Claude. The successes and failures are the same as for ChatGPT and Gemini.  
\label{Fig5577}}
\end{figure*} 

\underline{Consistent reasoning easier than deciding the equiv. class.}
This statement is crucial: without it, the CRP would essentially concern the difficulty of determining the equivalence class of a sentence that describes a computer program for a given number. For example, consider the following sentences describing the number $1$:

\noindent
\begin{minipage}[t]{.23\textwidth}
{\small
	\begin{algorithm}[H]
		\SetKwInOut{Input}{Input}\SetKwInOut{Output}{Output}
		\BlankLine
		\SetKwProg{Fn}{Function}{}{end} 
		\SetKwFunction{Compute}{Compute}
		\Fn{\Compute{$n$}}{
			\Input{$n\in\mathbb{N}$}
			\Output{$m \in \mathbb{Q}$}
			\BlankLine
			$m = 1 $
		}  
	\end{algorithm} 
}
\end{minipage}%
\begin{minipage}[t]{.28\textwidth}
\small
{\begin{algorithm}[H]
			\SetKwInOut{Input}{Input}\SetKwInOut{Output}{Output}
			\BlankLine
			\SetKwProg{Fn}{Function}{}{end} 
			\SetKwFunction{ComputeAlt}{ComputeAlt}
			\Fn{\ComputeAlt{$n$}}{
				\Input{$n\in\mathbb{N}$}
				\Output{$m \in \mathbb{Q}$}
				\BlankLine
				$m = 1 - 10^{-n}$
			}  
		\end{algorithm} 
}
\end{minipage}
Both sentences belong to the equivalence class of sentences describing the number $1$. However, one does not necessarily need to determine the equivalence class in order to solve a problem in arithmetic. For example, consider the following two sentences:
\begin{itemize}
\itemsep0em
	\item[(a)] \emph{Let $a_1$ be the number given by Compute above. The number of polar bears in Sahara is equal to zero times $a_1$. Determine the number of polar bears in Sahara.}
\item[(b)]  \emph{Let $a_2$ be the number given by ComputeAlt above. The number of polar bears in Sahara is equal to zero times $a_2$. Determine the number of polar bears in Sahara.}
\end{itemize}	  
Here (a) and (b) are equivalent since the sentences Compute and ComputeAlt describe the same number. However, clearly, one does not need to know which equivalence class the sentences Compute and ComputeAlt  belong to in order to solve the problem correctly. This example gives some intuition why consistent reasoning is easier than determining equivalence classes. 

The problem of determining the equivalence class of a sentence describing computer program that determines a number is related to Rice's Theorem \cite{Soare}, which again is related to the Halting problem. The CRP on the other hand requires a different toolbox.

\subsubsection*{The SCI hierarchy, optimisation and echoes of G\"odel} 

The main ingredients of the CRP are sentences about basic arithmetic problems arising in convex optimisation \cite{Nesterov1} (and robust optimisation \cite{Nemirovski_robust}), namely, linear programs \cite{khachiyan1980polynomial},  basis pursuit \cite{CandesRombergTao, donohoCS} and LASSO \cite{TibshiraniLasso} (see the Supplementary Material (SM) for details). 
What is novel in this paper is the incorporation of new techniques in recursion theory, as well as randomised algorithms, into the recent work on the SCI hierarchy including phase transitions \cite{comp} related to Smale's 9th problem \cite{MathFrontiersPerspectives} and its extensions. The SCI hierarchy generalises the well-known \emph{arithmetical hierarchy}, thus it does incorporate classical recursion theory. However, the SCI theory currently does not include all facets of the so-called Markov model \cite{Kushner99} that allows for numbers as sentences \cite{Turing_Machine} as input -- which is crucial to prove the CRP. This paper unlocks this major hurdle. This opens up a new connection between mathematical analysis, new techniques in recursion theory and provability theory (initiated by G\"odel \cite{godel1931formal}) that are needed to prove the CRP (see the discussion in the SM).  Indeed, CRP IV provides a statement -- similar (yet mathematically different) to G\"odel's first incompleteness theorem \cite{godel1931formal} -- specifically for optimisation. This is very similar to how the negative answer to Hilbert's 10th problem \cite{Hilberts10} yields a statement of non-provability -- as in G\"odel's first incompleteness theorem --  specifically for diophantine equations \cite{franzen2005godel}. 

The fact that trustworthy AI must implicitly compute the `I don't know' function and `give up', is an immediate consequence of a classification problem in the SCI hierarchy. Indeed, for basic problems in arithmetic the problem of computing the `I don't know' function is in $\Sigma_1$ and not in $\Delta_1$.
This is yet another example of a $\Sigma_1$ classification, for which there is already a very rich SCI theory -- for example, in computational spectral theory and computer assisted proofs \cite{SCI, comp, AIM, fefferman1996interval, Hansen_JAMS}. 

\subsubsection*{Inevitability of hallucinations and classical recursion theory} 

Earlier, we briefly mentioned the inevitability of hallucinations in chatbots. This follows trivially from Turing's `Halting problem', which is non-computable. Any chatbot that always answers and accepts questions about computer codes (for example whether a code will halt) must hallucinate. Put another way, since no chatbot can compute non-computable problems, any chatbot that always provides an answer must inevitably hallucinate. This is an immediate consequence of classical recursion theory, and the plethora of known non-computable problems. It was even discussed by Turing in his 1950 paper \cite{Turing_1950}.  

The CRP, on the other hand, is a completely different phenomenon. It shows that the non-hallucinating and always correct AI will exist on specific collections of problems. However, fallibility will occur if the AI emulates human intelligence by striving to reason consistently on exactly the same problems -- this is the paradox. In particular, the collection of problems are computable (there exists an AI that is correct on all the problems), however, when more equivalent sentences are added -- describing exactly the same problems -- then fallibility is inevitable. This immediately implies that determining the correct equivalence classes of the sentences is impossible. Yet, as discussed above, the CRP is even more refined: determining the answers to the aforementioned problems (described with all the equivalent sentences, as in Figure \ref{Fig22}) is strictly easier than determining the correct equivalence classes. This is why the CRP cannot be proven using reduction from the Halting problem \cite{Turing_Machine}, which is a standard trick. As we pointed out above, the problems forming the basis of the CRP are strictly easier than the Halting problem. 

To sum up: \emph{the CRP states that AIs striving for human-like intelligence will fail even on computable problems}.

\bibliographystyle{abbrv}
\bibliography{references_CRP}

\vspace{2mm}

\end{multicols}

\end{document}


\maketitle
	\vspace{-5mm}	
	
	\vspace{1mm}	
	
	\tableofcontents

\section{Introduction and definitions}

The proof of the Consistent Reasoning Paradox (CRP) springs out of the mathematics of the Solvability Complexity Index (SCI) hierarchy \cite{Hansen2016ComplexityII, Hansen_JAMS, SCI, CRAS, colbrook2019foundations, Matt2, Ben_Artzi2022, Colbrook_2019}, and in particular the new developments on phase transitions and generalised hardness of approximation\footnote{See \cite{Arora2007} for classical hardness of approximation.} in continuous optimisation \cite{comp} (see also \cite{gazdag2022generalised, Matt2, AIM}) in connection with Smale's 9th problem \cite{MathFrontiersPerspectives} and its extensions \cite{comp}. These developments are closely related to robust optimisation \cite{Nemirovski07, Nemirovski_robust, Nemirovski_robust2}. The novel techniques provided in this paper that are necessary for the proof of the CRP are the following:
\begin{itemize}[leftmargin=5.5mm]
\item[(i)] \underline{The SCI hierarchy and the incorporation of the Markov model.} Although the SCI hierarchy extends the Arithmetical hierarchy, and thus encompasses classical recursion theory, the main techniques developed in the SCI theory have been focused on mathematical analysis and `seeing the sequence'. That is -- in a somewhat simplified form -- the input to an algorithm is provided through a sequence of numbers \cite{Ko1991ComplexityTO}. This sequence could represent infinite-dimensional objects such as operators, point samples of functions etc., or inexact input of numbers \cite{comp, Hansen_JAMS, SCI, colbrook2019foundations, Matt2, Ben_Artzi2022}. What is new in this paper is that we now also allow the input to be finite strings representing the code producing the input sequence. This is often referred to as the Markov model \cite{Kushner99, Katzenelson72}. In specific cases, namely, for computing a  single valued real function $f : \mathbb{R} \rightarrow \mathbb{R}$ the Markov model is equivalent to the Turing model \cite{Turing_Machine} of `seeing the sequence' \cite{Shoenfield, Markov}. However, the CRP crucially depends on multi-valued mapping, hence we need to develop a new framework. The incorporation of the Markov model in the SCI hierarchy means that the well-developed collections of tools in the SCI theory to provide lower bounds on computations need to be substantially extended, which we initialise in order to prove the CRP. 

\item[(ii)] \underline{Randomised algorithms.} The SCI theory is equipped with a general framework for randomised algorithms that allows for universal lower bounds. However, this framework is insufficient in the Markov model. Hence, we extend the previous SCI framework for randomised algorithms to also include this model. A particularly delicate issue -- needed for the proof of the CRP -- is how to deal with randomised algorithms with access to oracles.  

\item[(iii)] \underline{Exit-flag computations and oracles.} In \cite{comp}, a theory for checker algorithms, sometimes refereed to as exit-flag computations, was developed within the SCI framework -- including oracles. However, this theory is in this paper now fully extended to the Markov model in order to prove the CRP. 

\item[(iv)] \underline{The `I don't know' functions and the SCI hierarchy.} The fact that trustworthy AIs must implicitly compute an `I don't know' function is a direct consequence of classifications in the SCI hierarchy. In particular, it is the $\Sigma_1$ classification that is crucial.

\item[(v)] \underline{Non-provability in analysis and optimisation.} With the extension of the SCI framework to the Markov model follow new techniques to establish non-provability results in analysis and specifically -- in this paper -- in optimisation. 

\end{itemize}

\subsection{Notation}

To state a precise mathematical description of the CRP, we need to introduce some mathematical notation and definitions from the SCI framework \cite{Hansen_JAMS, SCI, CRAS, colbrook2019foundations, Matt2, Ben_Artzi2022}.

\begin{definition}[Finite Dimensional Computational problem]\label{def:ComputationalProblem}
	Let $\Omega$ be some set, which we call the \emph{input} set,
	and $\Lambda$ be a finite set of functions $f\colon \Omega \to \Q$ such that for $\iota_1, \iota_2 \in \Omega$, then $\iota_1 = \iota_2$ if and only if $f(\iota_1) = f(\iota_2)$ for all $f \in \Lambda$, called an \emph{evaluation} set. Let $(\mathcal{M},d)$ be a metric space, and finally let $\Xi:\Omega \rightrightarrows  \mathcal{M}$ (the notation $\rightrightarrows$ means that $\Xi$ can be multi-valued) be a function which we call the \emph{solution map}.
	We call the collection $\{\Xi,\Omega,\mathcal{M},\Lambda\}$ a \emph{computational problem}.
\end{definition}

\begin{remark}[The main computational problem]\label{rem:comp_prob}
We will develop results that hold for abstract computational problems, but our primary interest will concern optimisation problems of the following form. Given a matrix $A \in \Q^{\nt \times \no}$ and a vector $y \in \Q^\nt$, consider the following three solutions maps that form a mainstay in modern computational mathematics \cite{CandesRombergTao, donohoCS, TibshiraniLasso, Osher_ROF, CohenDahmenDeVore, Chambolle_Alg, Fefferman2011Interpolation, Adcock2016, AdcockHansenBook, Juditsky_2011, Juditsky_2012, Nesterov_Nemirovski_Acta, Chambolle_Alg, Chambolle_Lions}, linear programming (LP); basis pursuit (BP) and LASSO:
\begin{align}
	\Xi_{\LP}(y,A) & \coloneq \argmin_{x \in \R^\no} \ \langle x , c \rangle,  \ \text{ such that } \ Ax = y, \ x \geq 0 \tag{LP} \label{eq:LP} \\
	\Xi_{\BP}(y,A) & \coloneq \argmin_{x \in \mathbb{R}^{\no}}  \|x\|_1, \ \text{ such that }\|Ax-y\|_2\leq \eta \tag{BP} \label{eq:BP} \\
	\Xi_{\LASSO}(y,A) & \coloneqq \argmin_{x \in \mathbb{R}^{\no}}  \lambda\|x\|_1+\|Ax-y\|_2^2. \tag{LASSO} \label{eq:LASSO} 
\end{align}
where $c = \ones_{\no}\in \Q^{\no}$ is the $\no$-dimensional vector with $1$ in each entry, and the parameters $\eta$ and $\lambda$ are positive rational numbers. For each of these problems, the input set is a subset $\Omega \subseteq \Q^{\nt \times \no} \times \Q^\nt$, the metric space is $\M = \R^\no$ equipped with the $\|\cdot\|_p$ distance for some $p \in \N \cup \{\infty\}$, and the evaluation set $\Lambda$ provides entry-wise components of every input $\iota = (y,A) \in \Omega$, so that $\Lambda = \{g_i\}_{i = 1}^\nt \cup \{h_{i,j}\}_{i = 1, j = 1}^{i =\nt, j = \no} $ where $g_{i}(y,A) \coloneq y_i$ and $h_{i,j}(y,A)= A_{i,j}$ for every $i,j$ and $(y, A) \in \Omega$. We denote $k \coloneq |\Lambda| = \nt +\nt\no$ and rename and re-enumerate the functions so that $\Lambda = \{f_i\}_{i = 1}^k$ (and so that $f_1 \coloneq h_{1,1}$ and $f_2 \coloneq h_{1,2}$, which will be useful later).
\end{remark}

	We now introduce the concept of Artificial Intelligence (AI) aimed at solving a computational problem. Following Turing, this will be a function that can be realised as a Turing machine (we refer the reader to \cite{Turing_Machine, Ko1991ComplexityTO} for definitions of Turing machines and recursivity). For the purposes of describing the CRP theorem precisely, we distinguish between three types of AIs: those that always return an output in the metric space $\M$; those that are allowed to sometimes confess `I don't know'; and those that also allow for the concept of `giving up', by virtue of a parameter that could be interpreted as the `time' spent looking for a solution.

	\begin{definition}[Artificial Intelligence (AI)]\label{def:ai_informal}
		Let $\{\Xi,\Omega, \M,\L\}$ be a computational problem where $\M = \R^d$ for some dimension $d \in \N$ and $|\Lambda|=k \in \N$. An \emph{Artificial Intelligence (AI)} is a function $\G \colon \Omega \to \Q^d$ that can be implemented as a Turing machine, which accesses each input $\iota \in \Omega$ by reading $\{f_i(\iota)\}_{i = 1}^k \in \Q^k$ passed on its reading tape. An \emph{`I don't know' AI} is a function $\G \colon \Omega \to \Q^d \cup \{\idk\}$ that can be implemented as a Turing machine and additionally can return the output $\idk$. An \emph{AI with a `giving up' parameter} is a sequence of `I don't know' AIs $\{\G_n\}_{n \in \N}$ such that the function 
		$(\{f_i(\iota)\}_{i = 1}^k,n) \mapsto \G_n(\iota) = \G_n(\{f_i(\iota)\}_{i = 1}^k)$ is recursive, and such that for every $\iota \in \Omega$, either $\G_n(\iota)= \idk$ for every $n \in \N$, or there exists $n_\iota \in \N$ such that $\G_n(\iota) = \idk$ for $n < n_\iota$ and $\G_n(\iota) \neq \idk$ for $n \geq n_\iota$.
	\end{definition}

\begin{remark}[AI and algorithm]
Throughout the paper we will use the words AI and algorithm interchangeably. 
\end{remark}

	\begin{remark}[Notation for AI's outputs]\label{rmk:notation_ai_output}
		If $\Gamma \colon \Omega \to \Q^d$ is an AI and $\iota \in \Omega$ is any input, then $\G$ on $\iota$ either:
		\begin{enumerate}[leftmargin = 8mm]
			\item does not halt, which we denote by $\Gamma(\iota) \uparrow$; or
			\item halts, which we denote by $\Gamma(\iota) \downarrow$, and returns an output $\Gamma(\iota) \in \Q^d$ uniquely determined by $\{f(\iota)\}_{f \in \Lambda}$. 
		\end{enumerate}
		An `I don't know' AI can additionally return the output $\Gamma(\iota) = \idk$. The same notation also applies to an AI with `giving up' parameter $\{\Gamma_n\}_{n \in \N}$.
	\end{remark}

\begin{remark}[Approximation Error]\label{rmk:approximation_error}
		The solution map $\Xi \colon \Omega \rightrightarrows \M$ of a computational may be multi-valued in certain cases. In particular, this can occur with optimisation problems such as \eqref{eq:LP}, \eqref{eq:BP} and \eqref{eq:LASSO}. Whenever this occurs, the computational problem of interest is to compute \emph{any} of these solutions. In fact, even though the solution map $\Xi$ may be multi-valued, the output of an AI will always be single-valued.
		Thus, if $\Gamma: \Omega \rightarrow \Q^d$ is an AI we measure the approximation error on input $\iota \in \Omega$ by 
		\[
		\disM(\Gamma(\iota), \Xi(\iota)) = \inf_{\xi \in \Xi(\iota)} d_{\mathcal{M}}(\Gamma(\iota),\xi),
		\]
		with the convention that $\operatorname{dist}_\M(\Gamma(\iota), \Xi(\iota)) = \infty$ if $\G(\iota) \nh$, and in the case of an `I don't know' AI that  $\operatorname{dist}_\M(\Gamma(\iota), \Xi(\iota)) = 0$ if $\G(\iota) = \idk$.
		\end{remark}

\subsection{Failures and hallucinations}

We now distinguish between two different ways in which an algorithm can provide a wrong ouput. In a general sense, an algorithm \emph{fails} whenever it provides an incorrect solution or does not halt, whereas it \emph{hallucinates} \cite{Hallucinations_1, Nature_Hallucinations2024, Hallucinations2, Hallucinations3} whenever it halts providing an incorrect but `plausible' solution.

\begin{definition}[Failure and Correctness]\label{def:failure}
	Let $\{\Xi,\Omega,\M,\L\}$ be a computational problem, $\Gamma\colon \Omega \to \M$ be an algorithm and $\kappa \geq 0$.
	\begin{enumerate}
		\item We say that $\Gamma$  \emph{$\kappa$-fails on $\iota \in \Omega$} if $
			\disM(\Gamma(\iota), \Xi(\iota)) > \kappa$ or if $\Gamma(\iota) \uparrow$.
		\item We say that $\Gamma$ is \emph{$\kappa$-correct} on $\iota \in \Omega$ if $\Gamma$ does not $\kappa$-fail on $\iota$.
	\end{enumerate}
\end{definition}

We denote by $\ball{r}{x}$ the closed ball of center $x \in \M$ and radius $r \geq 0$, and write $\ball{r}{A}= \bigcup_{x \in A}\ball{r}{x}$ for a subset $A \subseteq \M$. In case $\M = \R^d$ for some dimension $d \in \N$, we denote $\ballQ[r]{x} \coloneq \ball{r}{x} \cap \Q^d$.

\begin{definition}[Hallucinations]\label{def:hallucination}
	Let $\{\Xi,\Omega,\M,\L\}$ be a computational problem, $\Gamma\colon \Omega \to \M$ be an algorithm and $\kappa \geq 0$.
	\begin{enumerate}[leftmargin = 8mm]
		\item We say that $\Gamma$  \emph{$\kappa$-hallucinates on $\iota \in \Omega$} if $\Gamma(\iota) \downarrow$ and $\G(\iota) \in \ball{\kappa}{\Xi(\Omega)} \setminus \ball{\kappa}{\Xi(\iota)}$.
		\item We say that $\Gamma$ \emph{$\kappa$-hallucinates} if there exists $\iota \in \Omega$ such that $\Gamma$ hallucinates on $\iota$.
	\end{enumerate} 
\end{definition}

	\begin{remark}[Hallucinations for Discrete Problems]\label{rmk:hallucination_discrete}
	The tolerance parameter $\kappa\geq 0$ allows to accept inexact approximations to the true solution - which necessarily happens, for example, whenever $\Xi$ only takes irrational solutions, since the AI by construction must return rational outputs. In certain cases, however, the role of $\kappa$ is redundant. This is the case, for example, when $\M = \{0,1\}$ and the AI takes values in the discrete space $\{0,1\}$, which is the central setup when considering the problem of detecting $\kappa$-hallucinations of an AI, as done in CRP \ref{crp:3}. In such situations, we will simply say that the AI \emph{hallucinates} (rather than $\kappa$-hallucinates) on $\iota \in \Omega$ whenever it halts on $\iota$ and $\Gamma(\iota) \notin \Xi(\iota)$.
	\end{remark}

	\begin{remark}
		Note that if an algorithm $\Gamma$ $\kappa$-hallucinates on input $\iota$, then it also $\kappa$-fails on $\iota$. However, the converse is not true: an algorithm $\Gamma$ may fail because it does not halt, which is not considered an hallucination since no output  -- and hence no `plausible' output -- is produced; or $\Gamma$ could fail by producing an output that is far away from the range of the solution map, which would not constitute an hallucination since it would not be a `plausible' output. To address this discrepancy between failure and hallucination, we define the concept of an AI taking values that are close to the range of the solution map.
	\end{remark}

\begin{definition}[Algorithm within the range]\label{def:within_the_range}
	Let $\{\Xi,\Omega,\M,\L\}$ be a computational problem and $\alpha \geq 0$. An algorithm $\Gamma \colon \Omega \to \M$ is \textit{within the $\alpha$-range of} $\Xi$ if
	\begin{equation*}
		\operatorname{dist}_\M(\Gamma(\iota),\Xi(\Omega)) \leq \alpha\text{ for all  }  \iota \in \Omega.
	\end{equation*}
\end{definition}

In particular, note that an algorithm that is within the $\alpha$-range of a solution map necessarily always halts (see Remark \ref{rmk:approximation_error}).

\begin{remark}\label{rmk:equivalence_failure_hallucination}
	Let $0 \leq \alpha \leq \kappa$ and $\iota \in \Omega$. If an algorithm $\Gamma$ is within the $\alpha$-range of $\Xi$, then it $\kappa$-fails on input $\iota$ if and only if it $\kappa$-hallucinates on $\iota$.
\end{remark}

\subsection{AI, trustworthiness and the `I don't know' function}

	In the previous section, the concept of an AI was introduced, alongside the undesirable properties of failures and hallucinations. We now define the concept of a trustworthy AI, which never fails nor hallucinates: essentially, an AI is trustworthy if it either says `I don't know' or it is correct.

\begin{definition}[Trustworthy AI]\label{def:trustworthy_ai}
		Let $\{\Xi,\Omega, \M,\L\}$ be a computational problem and $\kappa \geq 0$. A \emph{$\kappa$-trustworthy AI} is an `I don't know' AI of the form $\G\colon \Omega \to \Q^d \cup \{\idk\}$  such that for $\iota \in \Omega$, whenever $\G(\iota) \neq \idk$ then the AI is $\kappa$-correct, meaning that $\G(\iota) \in \ball{\kappa}{\Xi(\iota)}$. Similarly, a \emph{$\kappa$-trustworthy AI with a `giving up' parameter} is an AI with a `giving up' parameter $\{\G_n\}_{n \in \N}$ such that for $\iota \in \Omega$, whenever $\G_n(\iota) \neq \idk$ for some $n \in \N$, then the AI is $\kappa$-correct for every $n' \geq n$, meaning that $\G_{n'}(\iota) \neq \idk$  and	$\G_{n'}(\iota) \in \ball{\kappa}{\Xi(\iota)}$.
\end{definition}

\begin{remark}[Trustworthiness for Discrete Problems]
	Analogously to Remark \ref{rmk:hallucination_discrete}, in case $\M = \{0,1\}$ and the AI takes values in the discrete space $\{0,1\}$ we will simply call an AI \emph{trustworthy} (rather than $\kappa$-trustworthy) since the tolerance parameter $\kappa$ becomes superfluous. We will use this terminology, for example, when analysing the problem of determining $\kappa$-hallucinations of an AI, as done in CRP \ref{crp:3}.
\end{remark}

	A $\kappa$-trustworthy AI (with or without a `giving up' parameter) separates the set of inputs $\Omega$ into two subsets: those inputs on which it (always) says `I don't know', and those on which it (eventually) provides a correct answer. We will define the `I don't know' function associated to the AI to be the characteristic function that distinguishes between such subsets.

	\begin{definition}[I Don't Know function associated to a $\kappa$-trustworthy AI]\label{def:idk_function_informal_Anders}
		Let $\{\Xi,\Omega, \M,\L\}$ be a computational problem and $\kappa \geq 0$. Suppose we have a $\kappa$-trustworthy AI (respectively, with a `giving up parameter'), define $\Omega_{\text{don't know}}$ to be the collection of $\iota \in \Omega$ for which the AI says `I don't know' (respectively, for every $n \in \N$), and $\Omega_{\text{know}} = \Omega_{\text{don't know}}^c $ to be its complement. Define the \emph{`I don't know' function associated to the AI} $\G$ (respectively, $\{\G_n\}$) to be $\Xi^{\Idk}_{{\G}}(\iota) = 1$ when $\iota \in \Omega_{\text{know}}$ and $\Xi^{\Idk}_{\G}(\iota) = 0$ when $\iota \in \Omega_{\text{don't know}}$ (respectively, with $\Xi^{\Idk}_{\{\G_n\}}(\iota) $). 
	\end{definition}
	
	Thus far, we have established the general concept of computational problems and the various types of AI that can solve them. We now turn our attention to a specific model of computational problems: the Markov model, where inputs will be accessed via strings. This topic will be addressed in the following section.

\subsection{\textcolor{black}{Inputs given as strings}}

An AGI \cite{Nature2019_AGI, Science2024_AGI, AGI_ICML2024} -- emulating human intelligence \cite{Turing_1950} -- is expected to take a finite string of characters as inputs, just like a human. This format of the input is also the basis for modern chatbots. Another expectation of an AI, that is close to human level intelligence, is the ability to answer questions in basic arithmetic. However, the AI must be able to handle numbers described as sentences, as introduced by Turing in his seminal 1936 paper \cite{Turing_Machine} and explained in the section \emph{``What is a ‘machine’/AI and what is a problem? - Turing and numbers as sentences''} in the main part of the paper. 
In particular, we follow Turing and consider AIs that, instead of accessing an input $\iota$ by directly reading its rational evaluations $\{f_i(\iota)\}_{i=1}^k \in \Q^k$ (where $k = |\Lambda|$), will instead access approximations to such rational numbers provided by Turing machines. Such Turing machines, when given a precision $n\in \N$ as input, return a rational number that is within $2^{-n}$ from $f(\iota)$ for $f \in \Lambda$. This is often referred to as the Markov model \cite{Kushner99, Katzenelson72} (although it was introduced by Turing \cite{Turing_Machine}) -- as a Markov algorithm \cite{MarkovModel, Kushner99, Katzenelson72} can only handle a finite input string, whereas a Turing machine can handle an infinite input string (typically treated as an oracle tape) \cite{Ko1991ComplexityTO}.

\begin{definition}[Strings corresponding to input numbers]\label{def:markov_problem_informal}
	Let $\{\Xi,\Omega,\M,\L\}$ be a finite-dimensional computational problem, with $\Lambda = \{f_1,\dotsc,f_k\}$. Given an input $\iota \in \Omega$, we say that $\stringinput = (\stringinput_1,\stringinput_2,\dotsc,\stringinput_k)$ \emph{corresponds to} $\iota \in \Omega$ if, for every $i$, $\stringinput_i$ is a Turing machine taking $\N \to \Q$ such that
	\begin{align*}
		|\stringinput_i(n) - f_i(\iota)| \leq 2^{-n} \quad \text{ for every } n \in \N.
	\end{align*}
	We denote by $\MarkovOmega$  the set of all the $\stringinput$ that correspond to some $\iota \in \Omega$, and define $\MarkovXi(\stringinput) \coloneq \Xi(\iota)$ for the unique input $\iota \in \O$ that $\stringinput$ corresponds to (whose uniqueness is guaranteed by Definition \ref{def:ComputationalProblem}). The superscript $M$ stands for Markov (see \S \ref{sec:Markov} for the extension of the SCI hierarchy to the Markov model). 
\end{definition}

This definition will be expanded upon in Definition \ref{def:Markov_Delta_1_Information}, which further clarifies how the AI reads an input $\stringinput$ consisting of a $k$-tuple of Turing machines. We delay this technicality up to \S \ref{section:proof_4} to focus on stating our main result, namely the CRP Theorem.

\begin{remark}[Finite vs infinite strings as input and equivalence of models]
The traditional model of Turing computability of real valued functions \cite{Ko1991ComplexityTO} is with Turing machines taking approximations to computable numbers on an infinite tape. This model is equivalent to the Markov model for single-valued functions \cite{Shoenfield, Markov}. However, the crucial part of the CRP is multivaluedness, and thus one has to develop the theory in the Markov model specifically -- the traditional model with an infinite input string is insufficient. In particular, the Markov model allows the algorithm to see the `code' producing the infinite sequence. Thus, creating impossibility results become harder. 
\end{remark}

\begin{remark}[Equivalent strings]\label{rmk:equivalent_strings}
Let $\stringinput, \stringinput' \in \MarkovOmega$. If $\stringinput$ and $\stringinput'$ correspond to the same $\iota \in \Omega$, we will write $\stringinput \sim \stringinput'$. This clearly defines an equivalence relation on $\MarkovOmega$.
\end{remark}

The notion of equivalent strings allows to define the concept Consistent Reasoning AI. This type of AI always produces an answer (potentially `I don't know'), and cannot be correct on one string but fail on an equivalent string.

\begin{definition}[Consistent Reasoning]
Let $\{\Xi,\O,\M,\L\}$ be a computational problem and recall $\MarkovOmega$ from Definition \ref{def:Markov_Delta_1_Information}. An AI $\Gamma$ defined on $\MarkovOmega$ is \emph{consistently reasoning} if it always halts, and in addition, if $\G$ is $\kappa$-correct on $\stringinput$, then $\G$ is $\kappa$-correct on every $\stringinput' \in \MarkovOmega$ such that $\stringinput' \sim \stringinput$. If $\Gamma$ is an `I don't know' AI, and if it is $\kappa$-correct on $\stringinput$, then on every $\stringinput' \in \MarkovOmega$ such that $\stringinput' \sim \stringinput$ we have that $\G$ is $\kappa$-correct or says `I don't know' on $\stringinput'$. An `I don't know' AI with giving up parameter $\{\Gamma_n\}_{n \in \N}$ is \emph{consistently reasoning} if $\G_n$ is consistently reasoning for every $n \in \N$.
\end{definition}

\begin{example}[Consistent Reasoning]
Suppose that $\hat \Omega \subseteq \MarkovOmega$ is such that for every $\iota \in \Omega$ there is exactly one $\stringinput \in \hat\Omega$ corresponding to $\iota$ and that there is an AI $\Gamma\colon \MarkovOmega \to \M$ that is $\kappa$-correct on $\hat \Omega$.  If $\Gamma$ was actually consistently reasoning, it would be $\kappa$-correct on all of $\MarkovOmega$. 
\end{example}

\section{The Consistent Reasoning Paradox -- Mathematical statement}

We are now ready to introduce 
a precise mathematical formulation of the five CRP statements (I) to (V) as in \S \emph{``The Consistent Reasoning Paradox (CRP)''} on page 3
of the main part of the paper. These statements are condensed in Theorem \ref{thm:crp} below.

\begin{theorem}[Consistent Reasoning Paradox]\label{thm:crp}
	Let $\no \geq 2$ and $\nt \geq 1$ be integer dimensions, $\kappa=10^{-1}$ the accuracy parameter, and $\Xi$ be any of the mappings in equations \eqref{eq:LP}, \eqref{eq:BP} or \eqref{eq:LASSO}. For suitable choices of positive rationals $\eta, \lambda$ and $\alpha$, there exist infinitely many inputs sets $\Omega$ of pairs $(A,y)$, where $A \in \Q^{\nt \times \no}$ and $y \in \Q^\nt$, such that for the computational problem $\{\Xi,\O,\M,\L\}$ (Remark \ref{rem:comp_prob}) and the map $\MarkovXi \colon \MarkovOmega  \rightrightarrows \M$ (Definition \ref{def:markov_problem_informal}), the following hold.
	\begin{enumerate}[label = (\Roman*), leftmargin = 6mm]
		\item \underline{\emph{(The non-hallucinating AI exists).}} Suppose that $\hat \Omega \subseteq \MarkovOmega$ is such that for every $\iota \in \Omega$ there is exactly one $\stringinput \in \hat\Omega$ corresponding to $\iota$. Then there exists an AI $\Gamma\colon \MarkovOmega \to \M$ that is $\kappa$-correct on $\hat \Omega$ and does not $\kappa$-hallucinate on $\MarkovOmega$ (by potentially not halting). However, no AI can correctly assign each $\stringinput \in \MarkovOmega$ to the input $\iota \in \Omega$ it corresponds to, even when given an oracle for the true solution $\Xi(\iota)$.

		\item \underline{\emph{(Attempting consistent reasoning yields hallucinations).}} 	
		Every AI $\Gamma \colon \MarkovOmega \to \M$ will $\kappa$-fail infinitely often. Equivalently, if  $\hat \Omega \subseteq \MarkovOmega$ is such that for every $\iota \in \Omega$ there is exactly one $\stringinput \in \hat\Omega$ corresponding to $\iota$, there exists no AI $\Gamma \colon \MarkovOmega \to \M$ that is simultaneously $\kappa$-correct on $\hat\Omega$ and consistently reasoning. In particular, if $\G$ is within the $\kappa$-range of $\MarkovXi$, then it $\kappa$-hallucinates on infinitely many $\stringinput \in \MarkovOmega$. 
		
		\item Let $\Gamma \colon \MarkovOmega \to \M$ be an AI that is within the $\alpha$-range of $\MarkovXi$.
		\begin{itemize}
			
			\item[(a)] \underline{\emph (Detecting hallucinations is hard).}  Every algorithm $\Gamma'$ that always halts and, on input $\stringinput$, tries to determine whether $\Gamma$ has $\kappa$-hallucinated, will itself hallucinate on infinitely many $\stringinput \in \MarkovOmega$, even when given access to an oracle for the true solution $\MarkovXi(\stringinput)$. \label{thm:crp_3a_layman}

			\item[(b)]  \underline{\emph (Detecting hallucinations and randomness).} Consider any subset $\subsetMarkov \subseteq \MarkovOmega$. If there exists a randomised algorithm $\G'$ that can detect $\kappa$-hallucinations of $\Gamma$ with probability strictly greater than $1/2$ on all the inputs of $\Omega_0$, then there exists a deterministic algorithm that detects $\kappa$-hallucinations of $\Gamma$ on $\Omega_0$. In particular, no randomised algorithm $\G'$  can detect $\kappa$-hallucinations of $\Gamma$ with probability strictly greater than $1/2$ on all the inputs of $\Omega$.

		\item[(c)] Given any $p \in (1/2,1]$, no randomised algorithm $\G'$ that always halts and has access to a true solution can detect $\kappa$-hallucinations of $\Gamma$ with probability greater than or equal to $p$ on all the inputs of $\MarkovOmega$.

		\end{itemize}
		
		\item \underline{\emph (Explaining the correct answer is not always possible).} There is a subset $\hat \Omega \subseteq \MarkovOmega$, such that for every $\iota \in \Omega$ there is only one $\stringinput \in \hat\Omega$ corresponding to $\iota$, and with the following property. There exists an AI $\Gamma\colon \MarkovOmega \to \M$ that halts and is always correct on $\hat \Omega$. However, there is a $\stringinput \in \hat \Omega$, corresponding to the only $\iota \in \O$ such that $|\Xi(\iota)|>1$, for which $\Gamma$ cannot provide a logically correct explanation of the solution (see Remark \ref{rem:logical_exp}). 
		
\item[(V)] \underline{\emph (The fallible yet trustworthy explainable AI saying `I don't know').} Consider the `indicator function of single-valuedness' $\Xi^{\mv}$, that is, for $\stringinput \in \MarkovOmega$, $\Xi^{\mv}(\stringinput)=1$ if $\MarkovXi(\stringinput)$ is single-valued, and $\Xi^{\mv}(\stringinput)=0$ otherwise. Then the following holds.
			\begin{enumerate}
				\item There exists exactly one input $\iota^0 \in \Omega$ with the following property: if $\stringinput \in \MarkovOmega$ is such that 
				$\MarkovXi(\stringinput)$  is multi-valued then $\stringinput$ corresponds to $\iota^0$.
				\item There exists a consistently reasoning, $\kappa$-trustworthy AI with `giving up' parameter, $\{\G_n\}_{n \in \N}$ (where $\G_n \colon \Omega \to \M \cup \{\idk\}$ for every $n \in \N$) that always halts such that its associated `I don't know' function equals $\Xi^{\mv}$ and such that it can provide a logical explanation of the correct solutions.
				\item There does not exist any $\kappa$-trustworthy $\G' \colon \MarkovOmega \to \M \cup \{\idk\}$ such that its associated `I don't know' function equals $\Xi^{\mv}$.
			\end{enumerate}
		\end{enumerate}
\end{theorem}

The technical statements of Theorem \ref{thm:crp} can be found in Theorem \ref{thm:crp_1_2}, Theorem \ref{thm:crp_3_4} and Theorem \ref{thm:crp_5}.

\begin{remark}[The CRP, optimisation and Smale's 9th problem \cite{MathFrontiersPerspectives} with extensions]
The CRP springs out of recent work on phase transitions (generalised hardness of approximation) in optimisation and Smale's 9th problem and its extensions \cite{comp} (see also \cite{AIM}: `Problem 5 (J. Lagarias): Phase transitions and the Extended Smale's 9th problem').  It may seem surprising that the CRP occurs in basic problems in arithmetic such as straightforward linear programs -- that are so simple that humans can easily solve them, and for which there exist a plethora of efficient algorithms that have been thoroughly analysed over the last decades \cite{khachiyan1980polynomial, Shor, nemirovski1983problem, int:Bayer1, int:Bayer2, int:Bayer3, Renegar1, Renegar2, renegar1988polynomial, Wrig97}. However, when the input numbers are replaced with equivalent sentences describing them, the situation changes dramatically, and the phase transitions established in \cite{comp} occur, also in the Markov model. This is the fundamental mechanism behind the CRP.
\end{remark}

\begin{remark}[Logical explanation]\label{rem:logical_exp}
Theorem \ref{thm:crp} (IV) and (Vb) have the expression `provide a logically correct explanation of the solution', which needs to be made precise. Mathematically, this means that one can prove in the standard axiomatic system of mathematics (Zermelo-Fraenkel with the axiom of choice -- ZFC) the asserted solution. As mathematical reasoning is a natural part of human intelligence, an AGI must surely be able to explain its mathematical deductions -- just like a human. Thus, it is natural to define `logically correct explanation' in terms of a mathematical proof in the standard axiomatic system (see also Theorem \ref{thm:crp_3_4} and Remark \ref{rem:logic_exp2}). 
\end{remark}

\begin{remark}[Quantifying the CRP]\label{rmk:quantification}
	The CRP Theorem \ref{thm:crp} provides a collection of both positive and negative results regarding the existence of performant AIs capable of solving certain computational problems. In particular, the negative results -- such as CRP (II) -- rely on the existence on inputs on which any candidate AI will fail. As mentioned in \S \emph{``The Consistent Reasoning Paradox (CRP) - A stronger CRP II: Failure sentences and equivalence''} on page 4 of the main part of the paper and in \S \emph{``Stronger statements – Quantifying the CRP''} on page 7 of the Methods section, our proof techniques allow to prove much more. In fact, we can provide upper bounds on the length of the inputs (written as codes in a programming language of preference, such as Python or C++). We give such an example for CRP (II) in MATLAB in \S \ref{sec:example} in Theorem \ref{thm:quantifyinf_crp}
\end{remark}

\section{Making trustworthy AI that says `I don't know' -- The SCI hierarchy}

\subsection{The Solvability Complexity Index (SCI) hierarchy and `I don't know' functions}
We start by providing an informal review of the basics of the the SCI hierarchy for an easy reference. 
The mainstay of the hierarchy are the $\Delta^{\alpha}_k$ classes, where the $\alpha$ is related to the model of computation. The full generality can be found in \cite{SCI, comp, Matt2, colbrook2019foundations}, however in this paper we will introduce the Markov model to the SCI -- that is, the input is always given as a finite string. 

 Given a collection $\mathcal{C}$ of computational problems (recall Definition \ref{def:ComputationalProblem}), then
\begin{itemize}[leftmargin = 8mm]
\item[(i)] $\Delta^{\alpha}_0$ is the set of problems that can be computed in finite time (the SCI $=0$).
\item[(ii)] $\Delta^{\alpha}_1$ is the set of problems that can be computed using one limit (the SCI $=1$) with control of the error, i.e. $\exists$ a sequence of algorithms $\{\Gamma_n\}$ such that $\text{dist}_{\M}(\Gamma_n(\iota), \Xi(\iota)) \leq 2^{-n}, \, \forall \iota \in \Omega$.
\item[(iii)] $\Delta^{\alpha}_2$ is the set of problems that can be computed using one limit (the SCI $=1$) without error control, i.e. $\exists$ a sequence of algorithms $\{\Gamma_n\}$ such that $\lim_{n\rightarrow \infty} \text{dist}_{\M}(\Gamma_n(\iota),\Xi(\iota)) = 0, \, \forall \iota \in \Omega$.
\item[(iv)] $\Delta^{\alpha}_{m+1}$, for $m \in \mathbb{N}$, is the set of problems that can be computed by using $m$ limits, (the SCI $\leq m$), i.e. $\exists$ a family of algorithms $\{\Gamma_{n_m, \hdots, n_1}\}$ such that 
\begin{equation*}\label{eq:SCI_limits}
\lim_{n_m \rightarrow\infty}\hdots \lim_{n_1\rightarrow\infty}
\text{dist}_{\M}(\Gamma_{n_m,\hdots, n_1}(\iota),\Xi(\iota)) = 0, \, \forall \iota \in \Omega.
\end{equation*}
 \end{itemize}

In general, this hierarchy cannot be refined unless there is some extra structure on the metric space $\mathcal{M}.$ The hierarchy typically does not collapse, and we have:
\begin{equation}\label{SCI1}
 \Delta_0^{\alpha} \subsetneq \Delta_1^{\alpha} \subsetneq \Delta_2^{\alpha} \subsetneq \hdots \subsetneq \Delta^{\alpha}_{m} \subsetneq \hdots.
\end{equation}
However, depending on the collection $\mathcal{C}$ of computational problems, the hierarchy \eqref{SCI1} may terminate for a finite $m$, or it may continue for arbitrary large $m$. The SCI hierarchy generalises the arithmetical hierarchy \cite{odifreddi1992classical} to arbitrary computational problems in any computational model. It is motivated by Smale's program on foundations of computational mathematics and some of his fundamental problems \cite{Smale81, Smale85} on the existence of algorithms for polynomial root finding -- solved by C. McMullen \cite{McMullen1, McMullen2} and P. Doyle \& C. McMullen \cite{Doyle_McMullen}. Many results, including McMullen's work  (see \cite{SCI}), can be viewed as implicitly providing classifications \cite{Gabai, Gabai_Annals, McMullen1, McMullen2, Doyle_McMullen, fefferman1996interval, Weinberger} (see Problem 5 in \cite{AIM}) in the SCI hierarchy.

For a formal definition of the SCI hierarchy we need the concept of a tower of algorithms. In the general case (see \cite{SCI, comp, Matt2}) a tower of algorithms allows for any model of computation \cite{bishop1967foundations, Cucker_Smale97, Fefferman_Klartag, Fefferman_Klartag2, Ko1991ComplexityTO, lovasz1987algorithmic, Turing_Machine, von_Neumann, realRAM}. This is referred to as a tower of algorithms of type $\alpha$ (see \cite{SCI, comp, Matt2} for details), where $\alpha$ indicates the model of computation. The definition below is of type $\alpha = A$ (arithmetic) and encompasses both the Turing model \cite{Turing_Machine} and the Blum-Shub-Smale \cite{BSS_Machine} model depending on how one defines recursivity.  The computational models over the reals that are widely used in continuous optimisation (Blum, Cucker, Shub \& Smale \cite{BCSS}, Chambolle \& P.-L. Lions \cite{Chambolle_Lions}, Chambolle \& Pock \cite{chambolle_pock_2016}, Fefferman \& Klartag \cite{Fefferman_Klartag, Fefferman_Klartag2,Fefferman_Klartag3}, Nemirovski \cite{NemirovskiConvexNotes20212022}, Nesterov \cite{Nesterov2018} and Nesterov \& Nemirovski \cite{Nesterov1}, Renegar \cite{Renegar_book}) are not equivalent. Thus, the full generality of the SCI framework is in general needed. However, for the purpose of proving the CRP we only need the concept of arithmetic tower.

\begin{definition}[Tower of algorithms -- Arithmetic tower]\label{tower_funct}
Given a computational problem $\{\Xi,\Omega,\mathcal{M},\Lambda\}$ and a natural number $k \in \N$, an \emph{arithmetic tower of algorithms of height $k$
 for $\{\Xi,\Omega,\mathcal{M},\Lambda\}$} is a collection of sequences of functions 
\begin{equation*}
\begin{split}
\Gamma_{n_k}:\Omega
\rightarrow \mathcal{M}, \quad 
\Gamma_{n_k, n_{k-1}} :\Omega
\rightarrow \mathcal{M}, \, \hdots \,, 
\Gamma_{n_k, \hdots, n_1}:\Omega \rightarrow \mathcal{M}, 
\end{split}
\end{equation*}
where $n_k,\hdots,n_1 \in \mathbb{N}$ and the functions $\Gamma_{n_k, \hdots, n_1}$ at the lowest level in the tower 
satisfy the following:
 for each $\iota \in\Omega$ the mapping $(n_k, \hdots, n_1, \{\iota_f\}_{f \in \Lambda}) \mapsto \Gamma_{n_k, \hdots, n_1}(\iota) = \Gamma_{n_k, \hdots, n_1}(\{\iota_f\}_{f \in \Lambda})$ is recursive, $\iota_f := f(\iota)$,  and $\Gamma_{n_k, \hdots, n_1}(\iota)$ is a finite string of rational numbers that can be identified with an element in $\mathcal{M}$. Moreover, for every $\iota \in \Omega$,
\begin{equation*}\label{conv}
\begin{split}
\Xi(\iota) &= \lim_{n_k \rightarrow \infty} \Gamma_{n_k}(\iota), \\
\Gamma_{n_k}(\iota) &=
  \lim_{n_{k-1} \rightarrow \infty} \Gamma_{n_k, n_{k-1}}(\iota),\\
& \, \, \, \, \vdots\\
\Gamma_{n_k, \hdots, n_2}(\iota) &=
  \lim_{n_1 \rightarrow \infty} \Gamma_{n_k, \hdots, n_1}(\iota),
\end{split}
\end{equation*}
where $S = \lim_{n \rightarrow \infty}S_n$ means convergence of the form $\text{dist}_{\M}(S_n,S) \rightarrow 0$ as $n \rightarrow \infty$.   
\end{definition}

\begin{remark}[Turing model for arithmetic towers]
Throughout this paper we will only consider the Turing model, thus any reference to arithmetic tower of algorithms or any reference to an algorithm will mean in terms of Turing. 
\end{remark}

\begin{definition}[Solvability Complexity Index]\label{complex_ind}
A computational problem $\{\Xi,\Omega,\mathcal{M},\Lambda\}$ is said to have \emph{Solvability Complexity Index $\mathrm{SCI}(\Xi,\Omega,\mathcal{M},\Lambda)_{\alpha} = k$} with respect to a tower of 
algorithms of type $\alpha$ if $k$ is the smallest integer for which there exists a tower of algorithms of type 
$\alpha$ of height $k$. If no such tower exists then $\mathrm{SCI}(\Xi,\Omega,\mathcal{M},\Lambda)_{\alpha} = \infty.$ If 
there exists a tower $\{\Gamma_n\}_{n\in\mathbb{N}}$ of type $\alpha$ and height one such that $\Gamma_{n_1}(\iota) \in \Xi(\iota)$ for all $\iota \in \Omega$ for some $n_1 < \infty$, then we define $\mathrm{SCI}(\Xi,\Omega,\mathcal{M},\Lambda)_{\alpha} = 0$.
\end{definition}

	\begin{definition}[The Solvability Complexity Index hierarchy]
	\label{1st_SCI}
	Consider a collection $\mathcal{C}$ of computational problems (we will use $\{\Xi,\Omega\}$ as a shorthand for $\{\Xi,\Omega,\M,\L\}$) and let $\mathcal{T}$ be the collection of all towers of algorithms of type $\alpha$ for the computational problems in $\mathcal{C}$.
	Define 
	\begin{equation*}
	\begin{split}
	\Delta^{\alpha}_0 &:= \{\{\Xi,\Omega\} \in \mathcal{C} \ \vert \   \mathrm{SCI}(\Xi,\Omega)_{\alpha} = 0\}\\
	\Delta^{\alpha}_{m+1} &:= \{\{\Xi,\Omega\}  \in \mathcal{C} \ \vert \   \mathrm{SCI}(\Xi,\Omega)_{\alpha} \leq m\}, \qquad \quad m \in \mathbb{N},
	\end{split}
	\end{equation*}
	as well as
	\[
	\Delta^{\alpha}_{1} := \{\{\Xi,\Omega\}  \in \mathcal{C}   \  \vert \ \exists \ \{\Gamma_n\} \in \mathcal{T}\text{ s.t. } \forall \iota \in\Omega \ \text{dist}_{\M}(\Gamma_n(\iota),\Xi(\iota)) \leq 2^{-n}\}. 
	\]
	\end{definition}

When there is extra structure on the metric space $\mathcal{M}$, say $\mathcal{M} = \mathbb{R}$ or $\mathcal{M} = \{0,1\}$ with the standard metric, one may be able to define convergence of functions from above or below. This is an extra form of structure that allows for a type of error control. 
\begin{definition}[The SCI Hierarchy (totally ordered set)]\label{def:tot_ord}
Given the setup in Definition \ref{1st_SCI}, suppose in addition that $\mathcal{M}$ is a totally ordered set, and that $\Xi$ is single valued.
Define 
\begin{equation*}
\begin{split}
\Sigma^{\alpha}_0 &= \Pi^{\alpha}_0 = \Delta^{\alpha}_0,\\
\Sigma^{\alpha}_{1} &= \{\{\Xi,\Omega\} \in \Delta_{2}^{\alpha} \ \vert \  \exists \ \{\Gamma_{n}\} \in \mathcal{T} \text{ s.t. } \Gamma_{n}(\iota) \nearrow \Xi(\iota) \ \, \forall \iota \in \Omega\}, 
\\
\Pi^{\alpha}_{1} &= \{\{\Xi,\Omega\} \in \Delta_{2}^{\alpha} \ \vert \  \exists \ \{\Gamma_{n}\} \in \mathcal{T} \text{ s.t. } \Gamma_{n}(\iota) \searrow \Xi(\iota) \ \, \forall \iota \in \Omega\},
\end{split}
\end{equation*}
where $\nearrow$ and $\searrow$ denote convergence from below and above respectively,
as well as, for $m \in \mathbb{N}$, 
\begin{equation*}
\begin{split}
\Sigma^{\alpha}_{m+1} &= \{\{\Xi,\Omega\} \in \Delta_{m+2}^{\alpha} \ \vert \  \exists \ \{\Gamma_{n_{m+1}, \hdots, n_1}\} \in \mathcal{T} \text{ s.t. }\Gamma_{n_{m+1}}(\iota) \nearrow \Xi(\iota) \ \, \forall \iota \in \Omega\}, \\
\Pi^{\alpha}_{m+1} &= \{\{\Xi,\Omega\} \in \Delta_{m+2}^{\alpha} \ \vert \  \exists \ \{\Gamma_{n_{m+1}, \hdots, n_1}\} \in \mathcal{T} \text{ s.t. }\Gamma_{n_{m+1}}(\iota) \searrow \Xi(\iota) \ \, \forall \iota \in \Omega\}.
\end{split}
\end{equation*}
\end{definition}
For example, if the metric space is the totally ordered set $\mathcal{M} = \{0,1\}$, from Definition \ref{def:tot_ord} we get the SCI hierarchy for arbitrary decision problems. 
The SCI hierarchy can be visualised as follows:
\begin{equation}\label{SCI_hierarchy}
\begin{tikzpicture}[baseline=(current  bounding  box.center)]
  \matrix (m) [matrix of math nodes,row sep=1.2em,column sep=1.5em] {
  \Pi_0^{\alpha}   &                    & \Pi_1^{\alpha} &    &  \Pi_2^{\alpha}&  & {}\\
  \Delta_0^{\alpha}&  \Delta_1^{\alpha} & \Sigma_1^{\alpha}\cup\Pi_1^{\alpha} & \Delta_2^{\alpha}&      \Sigma_2^{\alpha}\cup\Pi_2^{\alpha} & \Delta_3^{\alpha}& \cdots\\
	\Sigma_0^{\alpha}&                    & \Sigma_1^{\alpha} & &  \Sigma_2^{\alpha}&  &{} \\
  };
 \path[-stealth, auto] (m-1-1) edge[draw=none]
                                    node [sloped, auto=false,
                                     allow upside down] {$=$} (m-2-1)
																		(m-3-1) edge[draw=none]
                                    node [sloped, auto=false,
                                     allow upside down] {$=$} (m-2-1)
																		
																		(m-2-2) edge[draw=none]
                                    node [sloped, auto=false,
                                     allow upside down] {$\subsetneq$} (m-2-3)
																		(m-2-3) edge[draw=none]
                                    node [sloped, auto=false,
                                     allow upside down] {$\subsetneq$} (m-2-4)
																		(m-2-4) edge[draw=none]
                                    node [sloped, auto=false,
                                     allow upside down] {$\subsetneq$} (m-2-5)
																		(m-2-5) edge[draw=none]
                                    node [sloped, auto=false,
                                     allow upside down] {$\subsetneq$} (m-2-6)
																		(m-2-6) edge[draw=none]
                                    node [sloped, auto=false,
                                     allow upside down] {$\subsetneq$} (m-2-7)

												(m-2-1) edge[draw=none]
                                    node [sloped, auto=false,
                                     allow upside down] {$\subsetneq$} (m-2-2)
											 (m-2-2) edge[draw=none]
                                    node [sloped, auto=false,
                                     allow upside down] {$\subsetneq$} (m-1-3)
											(m-2-2) edge[draw=none]
                                    node [sloped, auto=false,
                                     allow upside down] {$\subsetneq$} (m-3-3)
											 (m-1-3) edge[draw=none]
                                    node [sloped, auto=false,
                                     allow upside down] {$\subsetneq$} (m-2-4)
																		(m-3-3) edge[draw=none]
                                    node [sloped, auto=false,
                                     allow upside down] {$\subsetneq$} (m-2-4)
																		(m-2-4) edge[draw=none]
                                    node [sloped, auto=false,
                                     allow upside down] {$\subsetneq$} (m-1-5)
											(m-2-4) edge[draw=none]
                                    node [sloped, auto=false,
                                     allow upside down] {$\subsetneq$} (m-3-5)
																		(m-1-5) edge[draw=none]
                                    node [sloped, auto=false,
                                     allow upside down] {$\subsetneq$} (m-2-6)
																		(m-3-5) edge[draw=none]
                                    node [sloped, auto=false,
                                     allow upside down] {$\subsetneq$} (m-2-6)
											(m-2-6) edge[draw=none]
                                    node [sloped, auto=false,
                                     allow upside down] {$\subsetneq$} (m-1-7)
																		(m-2-6) edge[draw=none]
                                    node [sloped, auto=false,
                                     allow upside down] {$\subsetneq$} (m-3-7);
																		
\end{tikzpicture}
\end{equation}
For details about the full SCI hierarchy \eqref{SCI_hierarchy} see \cite{Hansen_JAMS, SCI, CRAS, colbrook2019foundations, Matt2, Ben_Artzi2022}. 

\begin{remark}[Generality of the SCI hierarchy] The SCI hierarchy can be made much more general than suggested above (see  \cite{Hansen_JAMS, SCI, CRAS, colbrook2019foundations, Matt2}). However, for the purpose of proving the CRP, the above definitions are sufficient. 
\end{remark}

\begin{remark}[SCI hierarchy and the arithmetical hierarchy] The \emph{arithmetical hierarchy} \cite{Soare} is a special case of the SCI hierarchy (see \cite{SCI}). The SCI hierarchy is fundamentally based on limits rather than quantifiers \cite{Soare} -- that form the foundation of the arithmetical hierarchy. For example, McMullen's work on polynomial root finding \cite{McMullen1, McMullen2, Doyle_McMullen} and towers of algorithms, which is based on limits, is a part of the SCI hierarchy -- but not the arithmetical hierarchy. However, in special cases, the SCI hierarchy can be recovered through quantifiers rather than limits (see \cite{SCI}), as the arithmetical hierarchy is an example of. 
\end{remark}

\subsection{Trustworthy AI and the $\Sigma_1$ class -- Sufficient and necessary conditions} 

Let $\{\Xi,\O,\M,\L\}$ be a computational problem and $\kappa \geq 0$, and recall the definitions of $\kappa$-trustworthy AI and their associated `I don't know functions' as in Definitions \ref{def:trustworthy_ai} and \ref{def:idk_function_informal_Anders}. Consider the following question:

\textbf{Question:} \emph{Given a candidate `I don't know' function $\Xi^*: \Omega \to \{0,1\}$, is there any $\kappa$-trustworthy AI (with or without `giving up' parameter) $\G$ such that 
	\begin{align*}
		\Xi^* = \Xi^{\Idk}_{\G}?
	\end{align*}}
We now give sufficient and necessary conditions to answer the question above. Recall that for a function $f \colon A \to B$, we denote the preimage of $b \in B$ by $f^{-1}(b) = \{a \in A \ | \ f(a) = b\}$. We will use this notation specifically for the preimages of problem functions in computational problems.
\begin{theorem}[Sufficient and necessary conditions for trustworthy AI ]\label{thm:iff}
	Let $\{\Xi,\Omega, \M,\L\}$ be a computational problem and $\kappa \geq 0$. Let $\Xi^* \colon \Omega \to \{0,1\}$ be a candidate `I don't know' function. Define $\Omega^*_1 \coloneq (\Xi^{*})^{-1}(1)$. Then the following holds.
	\begin{enumerate}[leftmargin = 6mm]
		\item There exists a $\kappa$-trustworthy AI of the form $\G\colon \Omega \to \M \cup \{\idk\}$ such that $\Xi^* = \Xi^{\Idk}_{\G}$ (as per Definition \ref{def:idk_function_informal_Anders}) if and only if the two following conditions are satisfied:
	\begin{enumerate}[label = (\alph*), ref = \theenumi\alph*]
		\item \label{condition:delta0} $\{\Xi^*, \Omega, \{0,1\},\Lambda\} \in \Delta^A_0$; 
		\item \label{condition:b1} 	 There exists an algorithm $\Gamma^* \colon \Omega^*_1 \to \M$ wih $\G^*(\iota) \in \ball{\kappa}{\Xi(\iota)}$ for every $\iota$ with $\Xi^*(\iota)=1$.
	\end{enumerate} \label{conclusion:iff_1}
	\item There exists a $\kappa$-trustworthy AI with `giving up' parameter $\{\G_n\}_{n \in \N}$ (where $\G_n \colon \Omega \to \M \cup \{\idk\}$ for every $n \in \N$) such that $\Xi^* = \Xi^{\Idk}_{\{\G_n\}}$ (see Definition \ref{def:idk_function_informal_Anders}) if and only if the two following conditions are satisfied:
	\begin{enumerate}[label = (\alph*), ref = \theenumi\alph*]
		\item \label{condition:sigma1} $\{\Xi^*, \Omega, \{0,1\},\Lambda\} \in \Sigma^A_1$; 
		\item \label{condition:b2} There exists an algorithm $\Gamma^* \colon \Omega^*_1 \to \M$  such that $\G^*(\iota) \in \ball{\kappa}{\Xi(\iota)}$ for every $\iota$ with $\Xi^*(\iota)=1$.
	\end{enumerate} \label{conclusion:iff_2}
	\end{enumerate}
\end{theorem}

\section{Proof of the Consistent Reasoning Paradox (CRP)} \label{section:proof_4}

In order to prove the CRP, we first need to extend the SCI framework to the Markov model. We recall the definition of a computational problem from Definition \ref{def:ComputationalProblem}.

\subsection{Extending the SCI framework -- Breakdown epsilons and the Markov model}\label{sec:Markov}

Not all computational problems can be solved with perfect accuracy. The smallest achievable error of a computational problem is its breakdown epsilon, as introduced in \cite{comp}, and is presented in the following definition.

\begin{definition}[Strong breakdown epsilon]\label{def:Breakdown-epsilons_strong}
	Given a computational problem $\{\Xi,\Omega,\mathcal{M},\Lambda\}$, we define its \emph{arithmetic strong breakdown epsilon} as follows:
	\begin{equation*}
		\epsilon_{\mathrm{B}}^{\mathrm{s}, \mathrm{A}} := \sup\{\epsilon \geq 0 \, \vert \, \forall \, \text{Turing machine } \Gamma,  \, \exists \, \iota \in \Omega  \text{ such that } \operatorname{dist}_\M(\Gamma(\iota),\Xi(\iota)) > \epsilon\}.
	\end{equation*}
\end{definition}	
Hence, the strong breakdown epsilon is the largest number $\epsilon \geq 0$ such that no algorithm can provide accuracy exceeding $\epsilon$.

The following definition is an extension of Definition \ref{def:markov_problem_informal}. It clarifies the concept of a computational problem given \emph{in the Markov sense}, where inputs are not accessed directly by reading their coordinates, but rather are accessed via `codes' (or more precisely, the G\"odel number of Turing machines) that return approximations of such coordinates. We assume access to an injective function $\gn{\cdot}$ that associates each Turing machine $\Phi$ to its G\"odel number $\gn{\Phi} \in \N$. There are different possible such G\"odel numberings $\gn{\cdot}$; in the following, we assume one such G\"odel numbering is fixed.

\begin{definition}\label{def:Markov_Delta_1_Information}
	Given a finite-dimensional computational problem $\{\Xi,\O,\M,\Lambda\}$ with $\Lambda = \{f_1,\dots,f_k\}$, define its \emph{corresponding Markov problem} as the following computational problem:
	\begin{align*}
		\{\Xi,\O,\M,\Lambda\}^{M} \coloneq \{ \MarkovXi, \MarkovOmega, \M, \MarkovLambda\}
	\end{align*}
	where we have the following.
	\begin{enumerate}[leftmargin = 8mm]
		\item $\MarkovOmega$ is the set of all possible tuples of Turing machines that compute the coordinates of inputs $\iota$:
		\begin{align*}
			\MarkovOmega &\coloneq  \{ (\stringinput_1,\stringinput_2,\stringinput_3,\dotsc,\stringinput_k) : \, \exists  \iota \in \Omega \text{ so that for } i=1,2,\dotsc,k \text{ the function} \\
			&\stringinput_i \text{ is a Turing machine taking } \N \to \Q \text{ such that } |\stringinput_i(n) - f_i(\iota)| \leq 2^{-n} \text{ for every } n \in \N \}
		\end{align*}
		
		\item \label{item:corresponds_to} Given $(\stringinput_1,\stringinput_2,\stringinput_3,\dotsc,\stringinput_k) \in \MarkovOmega$, there is a unique $\iota_\stringinput \in \Omega$ so that $|\stringinput_i(n) - f(\iota_\stringinput)| \leq 2^{-n}$ for $i=1,2,\dotsc,k$ and $n \in \mathbb{N}$ (by Definition \ref{def:ComputationalProblem}); we define the map $\MarkovXi: \MarkovOmega  \rightrightarrows\  \mathcal{M}$ by 
		\[
		\MarkovXi(\stringinput_1,\stringinput_2,\stringinput_3,\dotsc,\stringinput_k)\coloneq \Xi(\iota_\stringinput).
		\]
		We also say that $\stringinput \in \MarkovOmega$ \emph{corresponds} to such $\iota_\stringinput \in \Omega$. 
		\item $\MarkovLambda=\{\gn{\cdot}_{1},\gn{\cdot}_2,\dotsc,\gn{\cdot}_k\}$ where for $i=1,2,\dotsc,k$, the map $\gn{\cdot}_i: \MarkovOmega \to \N$ is defined so that, for a given element $(\stringinput_1,\stringinput_2,\stringinput_3,\dotsc,\stringinput_k) \in \MarkovOmega$, $\gn{(\stringinput_1,\stringinput_2,\stringinput_3,\dotsc,\stringinput_k)}_i$ is the G\"odel number of $\stringinput_i$.
	\end{enumerate}
\end{definition}

\subsection{Defining equivalent sentences}

For computational problems given in the Markov sense, inputs are accessed indirectly via `codes' that represent Turing Machines providing approximations to such inputs. This association gives rise to the \emph{correspondence problem}, defined below, which consists in assigning to each code $\stringinput$ the input $\iota$ it corresponds to.

\begin{definition}[Correspondence Problem] \label{def:correspondence_problem}
		Let $\{\Xi,\Omega, \M, \L\}$ be a computational problem with $\Omega \subseteq \Q^d$ for some dimension $d\in \N$. The correspondence problem is the computational problem $\{\CorrXi, \MarkovOmega, \Omega, \MarkovLambda\}$, defined in the following way. 
	\begin{enumerate}[leftmargin = 8mm]
		\item Given $\stringinput \in \MarkovOmega$, by Definition \ref{def:Markov_Delta_1_Information}, \eqref{item:corresponds_to}, there exists a unique $\iota_\stringinput$ so that $\stringinput$ corresponds to $\iota_\stringinput$. We define the mapping $\CorrXi\colon \MarkovOmega \to \Omega$ by $\CorrXi(\stringinput) \coloneq \iota_\stringinput$ for every $\stringinput \in \MarkovOmega$.
		\item The set $\Omega$ is endowed with the discrete metric.
	\end{enumerate}
\end{definition}

In the setup above, a computational problem $\{\Xi,\O,\M, \L\}$ was fixed and the concept of  corresponding Markov problem was presented. This gave a precise definition of `equivalence' of codes: two codes $\stringinput,\stringinput' \in \MarkovOmega$ are equivalent if they correspond to the same input $\iota \in \Omega$ (in the sense of Definition \ref{def:Markov_Delta_1_Information}, see Remark \ref{rmk:equivalent_strings}). All these concepts are not subject to interpretation, as they are grounded in the mathematical machinery of Turing Machines, which is a purely mathematical concept.
However, in a more general sense, one can define arbitrary equivalence relation on strings.

\begin{definition}[Equivalent Sentences]\label{def:equivalent_sentences}
	Let $k \in \N$ and let $\Omega^M \subseteq (\mathcal{A}^*)^k$ be a set of $k$-tuples of strings in the fixed alphabet $\mathcal{A}$. Let $\mathcal{P}=\{P_n\}_{n \in \N}$ be a partition of $\Omega^M$ (so that $P_n \cap P_{m} = \emptyset$ for $n \neq m$ and $\bigcup_{n \in \N} P_n = \Omega^M$).
	
	\begin{enumerate}[leftmargin = 8mm]
		\item Two tuples of strings $\sigma,\sigma' \in \Omega^M$ are \emph{equivalent with respect to the partition $\mathcal{P}$} if there exists $n \in \N$ such that $\sigma, \sigma' \in P_n$. We will write $\sigma \sim_\mathcal{P} \sigma'$.
		\item A function $\Xi^M\colon \Omega^M \rightrightarrows \M$ \emph{respects the equivalence relation induced by the partition $\mathcal{P}$} if 
		\begin{align*}
			\sigma, \sigma' \in \Omega^M, \quad \sigma \sim_\mathcal{P} \sigma' \quad \implies \quad \Xi^M(\sigma) = \Xi^M(\sigma').
		\end{align*}
	\end{enumerate}
Here we have used a slight abuse of notation, as $\Omega^M$ is used for the completely general case as well.  
\end{definition}

Definition \ref{def:equivalent_sentences} allows to consider a general concept of sentences being equivalent, induced by the arbitrarily chosen partition $\mathcal{S}$. In this work, we will only consider the mathematical definition of equivalence as given by Turing and expressed in Definition \ref{def:Markov_Delta_1_Information}, which can be interpreted as being induced by the natural partition $\mathcal{S} \coloneq \{(\CorrXi)^{-1}(\iota)\}_{\iota \in \Omega}$ on $\MarkovOmega$.

\subsection{Can correctness of AIs be checked? -- The exit-flag problem and oracle computations}\label{sec:ExitFlagSetup}

In this section, we introduce two types of computational problems: the exit-flag problem, which consists in determining whether a given algorithm has produced a correct output (and is considered in CRP \ref{crp:3}); and the family of oracle problems, which consist in solving a computational problem with the help of an oracle providing a correct solution for an auxiliary computational problem (as considered in CRP \ref{crp:1} and CRP \ref{crp:3}).

\begin{definition}[Exit-flag problem]\label{def:ExitFlagProblem}
Let $\{\Xi,\Omega,\M,\Lambda\}$ be a computational problem, $\Gamma: \MarkovOmega \to \M$ an algorithm for the corresponding Markov problem $\{ \MarkovXi, \MarkovOmega, \M, \MarkovLambda\}$ and $\kappa > 0$. The \emph{exit-flag problem} is the computational problem
$
\{\MarkovExitXi,\MarkovOmega,\{0,1\},\MarkovLambda\},
$
where the solution map $\MarkovExitXi\colon \MarkovOmega \to \{0,1\}$ is given by
\begin{equation}\label{eq:Xi_exit}
	\MarkovExitXi(\stringinput) =
	\begin{cases}
		1 & \text{ if } \stringinput \in  \MarkovOmega_\Gamma, \\
		0 & \text{ if } \stringinput \in \MarkovOmega\setminus  \MarkovOmega_\Gamma,
	\end{cases}
\end{equation}
and where $\MarkovOmega_{\Gamma}$ denotes the set of inputs for which $\Gamma$ obtains accuracy better than $\kappa$. More precisely,
\begin{equation}\label{eq:success_set}
	\MarkovOmega_{\Gamma}:= \{\stringinput \in \MarkovOmega \, \vert \, \operatorname{dist}_{\M}(\Gamma(\stringinput),\MarkovXi(\stringinput)) \leq \kappa\}\subseteq \MarkovOmega.
\end{equation}
The metric on the space $\{0,1\}$ is inherited from $\R$.
\end{definition}

Observe that $\epsilon_{\mathrm{B}}^{\mathrm{s},\mathrm{A}} > \kappa$ implies that $\MarkovOmega_{\Gamma}$ is a strict subset of $\MarkovOmega$.

\begin{remark}[Key assumption]\label{rem:EFAlgorithmAssumption}
	Naturally, the exit-flag problem becomes trivial if $\Gamma$ yields outputs that are sufficiently far away from the range $\MarkovXi(\MarkovOmega)$. In fact, such outputs would not be `plausible', and thus such an algorithm $\G$ would not be a suitable candidate for attempting to solve the problem $\{ \MarkovXi, \MarkovOmega, \M, \MarkovLambda\}$.  Therefore, we must make a technical assumption regarding the type of AIs we will examine for the exit-flag problem. Concretely, we fix an $\alpha$ such that $0 \leq \alpha < \kappa$ and assume that $\Gamma$, our AI  defined on $\MarkovOmega$, is within the $\alpha$-range of $\MarkovXi$, according to Definition \ref{def:within_the_range}.
\end{remark}

After establishing the concept of the exit-flag problem, we introduce the class of oracle problems.

\begin{definition}[Oracle problem]\label{def:problem_with_oracle}
	Let $\Xi_1 \colon \Omega \rightrightarrows \mathcal{M}_1$ and $\Xi_2 \colon \Omega \rightrightarrows \mathcal{M}_2$ be two solution mappings defined on the same set $\Omega$, and $\M_2 \subseteq \R^d$ for some $d \in \N$.
	Fix a parameter $\omega > 0$. The \emph{computational problem of $\Xi_1$ with an oracle for $\Xi_2$} is the computational problem 
	\begin{align*}
			\{\Xi_1,\Omega,\M_1,\Lambda_1\}^{\mathcal{O}} = \{\Xi_1,\Omega,\M_1,\Lambda_1\}^{\mathcal{O},\Xi_2, \omega} \coloneq \{\XiO, \Omega \times \ballQ{\Xi_2(\Omega)}, \mathcal{M}_1, \Lambda^{\mathcal{O}}\},
	\end{align*}
	where 
	\begin{equation*}\label{eq:XiOracleDef}
		\XiO(\iota,y) :=
		\begin{cases}
			\Xi_1(\iota) & \text{ if } y \in \ballQ{\Xi_2(\iota)}; \\
			\Xi_1(\Omega) &\text{ if } y\notin \ballQ{\Xi_2(\iota)};
		\end{cases} \quad \text{ for } (\iota,y) \in \Omega \times  \ballQ{\Xi_2(\Omega)}.
	\end{equation*}
	
		For ease of readability, we will often suppress the superscripts $\Xi_2$ and $\omega$ when they are clear from the context.

	The set $\Lambda^{\mathcal{O}}$ is defined as follows: if $\Lambda = \{f_1,\dots,f_k\}$ and $g_j(y)\coloneq y_j$ for $y=(y_1,\dots,y_d) \in \Q^d$ and $j = \{1,\dots,d\}$, then $\Lambda^{\mathcal{O}}\coloneq \{f_i^{\mathcal{O}}\}_{i = 1}^k \cup \{g_j^{\mathcal{O}}\}_{j = 1}^d$, where
	\begin{align*}
		f_i^{\mathcal{O}}\colon \Omega \times \ballQ{\Xi_2(\Omega)} \to \Q, & \quad  f_i^{\mathcal{O}}(\iota,y) = f_j(\iota), \quad i = 1,\dotsc, k;\\
		g_j^{\mathcal{O}}\colon \Omega \times  \ballQ{\Xi_2(\Omega)}  \to \Q, & \quad g_j^{\mathcal{O}}(\iota,y) = g_k(y), \quad k = 1,\dots, d.
	\end{align*}
\end{definition}

We now combine both the exit-flag computation and computation with oracles. We consider the problem of the exit-flag associated to an algorithm $\Gamma$, given an oracle for the original computational problem $\MarkovXi$.

\begin{definition}[Exit-flag problem with oracle]\label{def:ExitFlagProblemWithOracle}
Let $\{\Xi,\Omega,\M,\Lambda\}$ be a computational problem with corresponding Markov problem $\{ \MarkovXi, \MarkovOmega, \M, \MarkovLambda\}$, $\Gamma: \MarkovOmega \to \M$ be an algorithm and $\kappa, \omega\in \Q$ such that $0 <  \omega < \kappa$. The \emph{exit-flag problem with oracle} is the computational problem 
\begin{equation*}\label{eq:ExifFlagProblemWithOracleMarkov}
	\{\MarkovExitXi,\MarkovOmega,\{0,1\},\MarkovLambda\}^{\mathcal{O}} \coloneq \{\MarkovExitXiOracle,\domainOracleXi,\{0,1\},\MarkovLambdaOracle\}.
\end{equation*}
where, following Definition \ref{def:problem_with_oracle}, the exit-flag map with oracle is
 \begin{equation}\label{eq:ExitXiOracleDef}
 	\MarkovExitXiOracle(\stringinput,y) :=
 	\begin{cases}
 		1 & \text{ if } \stringinput \in  \MarkovOmega_\Gamma \, \land \,  y \in \ballQ{\MarkovXi(\stringinput)}; \\
 		0 & \text{ if } \stringinput \in \MarkovOmega\setminus  \MarkovOmega_\Gamma \, \land \, y \in \ballQ{\MarkovXi(\stringinput)};\\
 		\{0,1\} &\text{ if } y\notin \ballQ{\MarkovXi(\stringinput)}.
 	\end{cases}
 \end{equation}
\end{definition}

\subsection{Randomised algorithms} We now consider the case of probabilistic algorithms, whose output is not a deterministic function of the input. We consider a general model of randomisation, that is broader than the coin flips (or Bernoulli measures) considered in the seminal work of K.~De~Leeuw, E.~F. Moore, C.~E. Shannon, and N.~Shapiro \cite{de1956computability}. In particular, it allows for more general sources of randomness via the concept of computable measures \cite{downey2010algorithmic}.

\subsubsection{Computable measures} 
$S^*$ is the set of finite strings over the set $S$, and $\lambda$ denotes the empty string. The concatenation of two strings $\sigma$ and $\tau$ is denoted $\sigma \tau$. The length of string $\sigma$ is $|\sigma|$. If $\tau$ extends $\sigma$, we write $\sigma \preceq \tau$. 

\begin{definition}\label{def:pre-measure}
	A \emph{(probability) pre-measure} on $\{0,1\}^*$ is a function $\rho: \{0,1\}^* \to [0,1]$ such that $\rho(\sigma) = \rho(\sigma 0) + \rho(\sigma 1)$ for every $\sigma \in \{0,1\}^*$, and $\rho(\lambda) = 1$.
\end{definition}

The condition $\rho(\lambda)=1$ makes it a probability measure.
A pre-measure $\rho$ on $\{0,1\}^*$ induces a measure $\mu_\rho$ on $\{0,1\}^\N$ in a natural way. Denoting by $\brackets{\sigma}$ the infinite strings whose initial segment is $\sigma$, that is $\brackets{\sigma} \coloneq \{\sigma \tau \ | \ \tau \in \{0,1\}^\N\}$, the measure $\mu_\rho$ is uniquely determined by $\mu_\rho(\brackets{\sigma}) = \rho(\sigma)$ for every $\sigma \in \{0,1\}^*$ by the classical Carathéodory's construction. Explicitly, we use the extension theorem for pre-measures on semi-rings (see, for example, \cite[Theorem 1.53]{klenke2013probability} applied to the pre-measure $\rho$ and to the semi-ring $\mathcal{R}\coloneq\{\brackets{\sigma} \ | \ \sigma \in \{0,1\}^*\}\cup \{\emptyset\}$). In particular,  $\mu_\rho$ is the measure induced by the outer measure
\begin{align*}
	\mu_{\rho}^*(A) \coloneq \inf\left\{\sum_{n \in \N} \rho({\sigma_n}) \ | \ \{\sigma_n\}_{n \in \N} \ \text{ s.t. } \ \sigma_n \in \{0,1\}^* \text{ for all } n \in \N, \ A \subseteq \bigcup_{n \in \N} \brackets{\sigma_n}\right\}
\end{align*}
by restriction to the class of measurable sets, namely the sets $A$ for which $\mu^*_\rho(B) = \mu^*_{\rho}(B \cap A) + \mu^*_\rho(B \cap {A}^c)$ for every $B \subseteq \{0,1\}^\N$. This class forms the $\sigma$-algebra of measurable sets on which the measure $\mu_\rho$ is defined. Sets of the form $\brackets{\sigma}$ for a finite string $\sigma$ are referred to as \emph{cylinder sets} and are measurable \cite{klenke2013probability}. Given a subset $A \subseteq \{0,1\}^*$, we denote $\brackets{A} \coloneq \bigcup_{\sigma \in A} \brackets{\sigma}$.

\begin{definition} [Computable Measure] \label{def:computable_measure}
	We now define the concept of computability for (pre-)measures.
	\begin{enumerate}[leftmargin = 8mm]
		\item A pre-measure $\rho$ is \emph{computable} if there exists a recursive function $r: \{0,1\}^* \times \N \to \Q$ such that
		\begin{align*}
			|r(\sigma,n) - \rho(\sigma)| \leq 2^{-n} \quad \text{ for every $n \in \N$}.
		\end{align*}
		\item A measure is \emph{computable} if it is induced by a computable pre-measure $\rho$. 
	\end{enumerate}
\end{definition}

\begin{remark}
	If $\rho$ is a computable pre-measure, the recursive function $r$ naturally extends to a recursive function on finite strings, such that
	\begin{align}\label{eq:approximation_r}
		|r(\sigma_1,\dots,\sigma_k,n) - \mu_\rho(\brackets{\sigma_1,\dotsc, \sigma_k})| \leq 2^{-n} \quad \text{ for every } n \in \N, \quad \sigma_1,\dotsc,\sigma_k \in \{0,1\}^*
	\end{align}
\end{remark}

\begin{remark}
	An important special case of the above definition is the computable pre-measure determined by a Bernoulli process with computable parameter $p \in [0,1]$, given by $\rho_p(\sigma) = p^{k}(1-p)^{|\sigma|-k}$ where $k$ is the number of $1$'s appearing in $\sigma$. In particular, when $p=\frac{1}{2}$, one obtains the fair Bernoulli measure $\mu=\mu_{\rho_{\frac{1}{2}}}$ induced by the pre-measure $\rho_{\frac{1}{2}}(\sigma)=2^{-|\sigma|}$.	Equivalently, $\mu_{\rho_p}$ can be seen as the product measure on $\{0,1\}^\N$ induced by the Bernoulli probability measure $\nu$ on $\{0,1\}$ with parameter $p$ given by $\nu(\{1\})=p$ and $\nu(\{0\})=1-p$.
\end{remark}

De~Leeuw, Moore, Shannon and Shapiro  \cite{de1956computability} proved that the $p$-Bernoulli measure $\mu_p$ is computable if and only if $p$ is a computable real number.

\subsubsection{Probabilistic Turing machines}\label{sec:probabilistic}

Expanding on the definition by De~Leeuw, Moore, Shannon and Shapiro \cite{de1956computability}, we define probabilistic Turing machines.

\begin{definition}
	Let $\mu$ be a computable measure on $\{0,1\}^\N$. A \emph{Probabilistic Turing machine (PTM) with respect to $\mu$} is a Turing machine provided with an extra read-only tape, called the \emph{randomised tape}, which is initialised with a draw from the distribution $\mu$. 
\end{definition}

\begin{remark}
	This definition of Probabilistic Turing machines encompasses the traditional definitions of probabilistic Turing machines with computable coin flips (namely, whenever $\mu = \mu_\rho$ for computable $p$).
\end{remark}

$\Gamma^{\ran}(\iota,\beta)$ will denote the action of the PTM $\Gamma^{\ran}$ on an input $\iota \in \Omega$ with $\beta \in \{0,1\}^\N$ initialised on the randomized tape. We could thus consider a Probabilistic Turing machine as a specific instance of a partial function $\Gamma^{\ran}\colon \Omega \times \{0,1\}^\N \to  \M$. Sometimes, we will consider Probabilistic Turing machines that read only a finite portion of the randomised tape, in which case it will be interpreted (with a slight abuse of notation) as a partial function $\Gamma^{\ran}\colon \Omega \times \{0,1\}^* \to  \M$.
We adopt the following additional notation: for every $\iota \in \Omega$, $\beta \in \{0,1\}^\N$, $t \in \N$, and $\sigma \in \{0,1\}^*$ we set:
\begin{align*}
	\halted{t} & \coloneq \text{ the output (if any) of $\Gamma^{\ran}$ on input $\iota$ after querying only the first $t$ bits of $\beta$}; \\
	\Gamma^{\ran}(\iota,\sigma) & \coloneq \text{ the output (if any) of $\Gamma^{\ran}$ on input $\iota$ and finite string $\sigma$}.
\end{align*}

For every $\iota \in \Omega$, the measure $\mu$ on $\{0,1\}^\N$ induces a measure $\pr(\G^{\ran}(\iota) \in \cdot )$ on the Borel $\sigma$-algebra $\mathcal{B}(\mathcal{M})$ given by the pushforward
\begin{align*}
	\P(\G^{\ran}(\iota) \in E) \coloneq \mu(\{ \beta \in \{0,1\}^\N \ | \ \Gamma^{\ran}(\iota,\beta) \in E\})
\end{align*}
for every $E \in \mathcal{B}(\mathcal{M})$. This definition is justified by the fact that for every $\iota \in \Omega$ the function 
$
\Gamma^{\ran}(\iota,\cdot)\colon \{0,1\}^\N \to \mathcal{M}
$
 is measurable (with respect to the $\sigma$-algebra on $\{0,1\}^\N$ induced by the pre-measure $\rho$, and to the Borel $\sigma$-algebra $\mathcal{B}(\M)$) as we will show in Proposition \ref{prop:measurable}. 
\begin{definition}\label{def:TuringMachineAlwaysHalt}
	A Probabilistic Turing machine \emph{that always halts} is a PTM for which the underlying function $\G^{\ran}\colon \Omega \times \{0,1\}^\N \to \mathcal{M}$ is total, so that for every $\iota \in \Omega$, $\G^{\ran}(\iota,\beta) \downarrow$ for every $\beta \in \{0,1\}^\N$.
\end{definition}

\begin{remark}
	The condition that a Probabilistic Turing machine always halts is stronger than the requirement that $\Gamma^{\ran}$ halts with probability one, which would instead read $\mu(\{\beta \in \{0,1\}^\N \ | \ \G^{\ran}(\iota, \beta) \downarrow\})=1$ for every $\iota \in \Omega$.
\end{remark}

\subsection{Precise formulation of the CRP}

In this section, we are finally ready to present the CRP theorem in a completely precise form, using the notation and terminology developed up to this point.

	\subsubsection{The Setup for the CRP}\label{section:setup_CRP}
	We are about to state the Theorems that illustrate the Consistent Reasoning Paradox, namely Theorems \ref{thm:crp_1_2}, \ref{thm:crp_3_4} and \ref{thm:crp_5}. Such theorems will concern the computational problems of Linear Programming, Basis Pursuit and Lasso as in \eqref{eq:LP}, \eqref{eq:BP} and \eqref{eq:LASSO}. There are a number of aspects that are common to each of the three theorems. These are as follows:
	\begin{enumerate}
		\item The dimensions $(\no,\nt)$ for Linear Programming, Basis Pursuit and Lasso as in \eqref{eq:LP}, \eqref{eq:BP} and \eqref{eq:LASSO} can be chosen to be any integer $\no$ and $\nt$ with $\no \geq 2$ and $\nt \geq 1$.
		\item The value $\kappa$, which represents the error tolerance for a solution to the computational problem, is set to $10^{-1}$. The LASSO parameter $\lambda$ in \eqref{eq:LASSO} is assumed to satisfy $\kappa < \lambda \leq 2\kappa$ and the basis pursuit parameter $\eta$ in \eqref{eq:BP} satisfies $\kappa < \eta \leq 2\kappa$.
		\item We always treat the output of an algorithm solving LP, BP or Lasso as a rational vector. The distance to the true solution is performed in $\|\cdot\|_p$ with $p \in \mathbb{N} \cup \{\infty\}$.
		\item Given the dimensions $(\no,\nt)$ and the accuracy $\kappa$, there is an infinite family (indexed by an additional parameter $\theta$) of input sets $\Omega_{\no,\nt}(\theta)$ for which Theorems \ref{thm:crp_1_2}, \ref{thm:crp_3_4} and \ref{thm:crp_5} apply. This input set applies to each part of the CRP and thus does not change throughout Theorems \ref{thm:crp_1_2}, \ref{thm:crp_3_4} and \ref{thm:crp_5}.
		\item When we have fixed the dimensions $(\no,\nt)$ and the parameter $\theta$, we set $\Omega = \Omega_{\no,\nt}(\theta)$. Thus the set $\MarkovOmega$ and $\MarkovXi$ are defined as in Definition \ref{def:Markov_Delta_1_Information}.
	\end{enumerate}

\subsubsection{The CRP Theorems}\label{sec:CRP_theorems}

Now that we have built up all the necessary mathematical machinery, we are ready to state the CRP theorem precisely. The Consistent Reasoning Paradox is one unified theorem, but for the sake of clarity it has been split into three distinct results. 

\begin{theorem}[CRP I and II]\label{thm:crp_1_2}
	
	Consider the setup of \S \ref{section:setup_CRP}, and the computational problem $\{\MarkovXi, \MarkovOmega, \M,\MarkovLambda\}$ from Definition \ref{def:Markov_Delta_1_Information}. Then, the following holds.
	\begin{enumerate}[label=(\Roman*), ref = \Roman*, leftmargin = 8mm, series=CRP, resume = CRP]
		\item If $\hat \Omega \subseteq \MarkovOmega$ such that $|(\CorrXi)^{-1}(\iota) \cap \hat \Omega| = 1$ for every $\iota \in \Omega$, (where $\CorrXi$ is as in Definition \ref{def:correspondence_problem}) then there exists an algorithm $\Gamma \colon \MarkovOmega \to \ball{\kappa}{\MarkovXi(\MarkovOmega)}$ satisfying the following two conditions.
		\begin{enumerate}
			\item For every $\stringinput \in \MarkovOmega$, either $\Gamma(\stringinput) \in \ball{\kappa}{\MarkovXi(\stringinput)}$ or $\Gamma(\stringinput) \nh$.
			\item $\Gamma(\stringinput) \in \ball{\kappa}{\MarkovXi(\stringinput)}$ for every $\stringinput \in \hat \Omega$.
		\end{enumerate}
		However, if $\omega \in \Q$ is such that $0 \leq \omega < \kappa$, then the correspondence problem with an oracle for $\MarkovXi$, given by
$
	\{\CorrXi,\MarkovOmega,\Omega,\MarkovLambda\}^{\mathcal{O}} 
$	
 (see Definitions \ref{def:correspondence_problem} and \ref{def:problem_with_oracle}) satisfies the following: for any algorithm with oracle $\CorrGO \colon \domainOracleXi \to \Omega$ there exists $(\stringinput,y) \in \domainOracleXi$ such that $\CorrGO(\stringinput,y) \notin \CorrXiO(\stringinput,y)$. \label{crp:1}
		 
		\item For every algorithm $\Gamma \colon \MarkovOmega \to \M$ there exist (infinitely many) $\stringinput \in \MarkovOmega$ such that $\Gamma$ $\kappa$-fails on $\stringinput$ (see Definition \ref{def:failure}). In particular, for every algorithm $\Gamma \colon \MarkovOmega \to \M$ that always halts and is within the $\kappa$-range of $\MarkovXi$, there exist (infinitely many) $\stringinput \in \MarkovOmega$ such that $\Gamma$ $\kappa$-hallucinates on $\stringinput$ (see Definition \ref{def:hallucination}). \label{crp:2}
		
	\end{enumerate}
\end{theorem}

\begin{theorem}[CRP III and IV]\label{thm:crp_3_4}
		Consider the setup of \S \ref{section:setup_CRP}. Then, for the computational problem $\{\MarkovXi, \MarkovOmega, \M,\MarkovLambda\}$, the following holds.

		\begin{enumerate}[label=(\Roman*), ref = \Roman*,  leftmargin = 8mm, series=CRP, resume = CRP]
		\item \label{crp:3} For every $\alpha \in \Q$ such that $0 \leq \alpha < \kappa$ and  for every algorithm $\Gamma \colon \MarkovOmega \to \M$ that is within the $\alpha$-range of $\MarkovXi$ (see Definition \ref{def:within_the_range}), consider the exit-flag problem associated to $\Gamma$ given by
		$
			\{\MarkovExitXi,\MarkovOmega,\{0,1\},\MarkovLambda\}
		$
		as in Definition \ref{def:ExitFlagProblem}. Moreover, for $\omega \in \Q$ such that $0 \leq \alpha \leq \omega < \kappa$, consider the exit-flag problem associated to $\Gamma$ with an oracle for $\MarkovXi$, given by
		$
			\{\MarkovExitXi,\MarkovOmega,\{0,1\},\MarkovLambda\}^{\mathcal{O}} 
		$	
		from Definition \ref{def:ExitFlagProblemWithOracle}. Then the following holds:
			\begin{enumerate}[label = (\alph*), ref = III\alph*]
			\item \label{crp:3a} For any algorithm $\Gamma^E\colon \MarkovOmega \to \{0,1\}$ that always halts, there exist (infinitely many) $\stringinput \in \MarkovOmega$ such that $\Gamma^E(\stringinput)  \neq  \MarkovExitXi(\stringinput)$. Furthermore, considering the exit-flag problem with oracle, for any algorithm $\GammaEO:\domainOracleXi \to \{0,1\}$  that always halts there exist (infinitely many) $(\stringinput,y) \ \in \domainOracleXi$ such that $\GammaEO(\stringinput,y) \notin \MarkovExitXiOracle(\stringinput,y)$.
			
			\item \label{crp:3b} Fix any subset $\subsetMarkov \subseteq \MarkovOmega$. If there exists a probabilistic Turing machine $\Gamma^{E,\ran}: \subsetMarkov \to \{0,1\}$ such that
			\begin{align}\label{eq:prob_greater_1/2}
					\mathbb{P}\left(\Gamma^{E,\ran}(\stringinput) = \MarkovExitXi(\stringinput)\right) > \frac{1}{2}
			\end{align}
			for every $\stringinput \in \subsetMarkov$, then there exists a deterministic Turing machine $\Gamma^E\colon \subsetMarkov \to \{0,1\}$ that always halts such that $\Gamma^E(\stringinput) = \MarkovExitXi(\stringinput)$ for every $\stringinput \in \subsetMarkov$. In particular, there exists no probabilistic Turing machine $\Gamma^{E,\ran}: \MarkovOmega \to \{0,1\}$ such that \eqref{eq:prob_greater_1/2} holds for every $\stringinput \in \MarkovOmega$.

				\item \label{crp:3c} For every $p \in (1/2,1]$ there exists no probabilistic Turing machine with oracle $\Gamma^{E,\mathcal{O},\ran}: \domainOracleXi \to \{0,1\}$ that always halts such that 
								\begin{align}\label{eq:prob_greater_p}
									\mathbb{P}\left(\Gamma^{E,\mathcal{O},\ran}(\stringinput,y) \in \MarkovExitXiOracle(\stringinput,y)\right) \geq p
								\end{align}
								 for every $\stringinput \in \MarkovOmega$ and $y \in  \ballQ{\MarkovXi(\MarkovOmega)}$.
							\end{enumerate}
		
		\item Assume that ZFC is $\Sigma_1$-sound (see \cite[p. 155, Definition 1.21 and Remark 1.22]{hajek2017metamathematics}. Then there exists a class $\hat \Omega \subseteq \MarkovOmega$ such that there is an algorithm $\Gamma:\hat \Omega \to  \ball{\kappa}{\MarkovXi(\MarkovOmega}$ satisfying the following:
\begin{enumerate}
	\item In the standard model of arithmetic, for every $\iota \in \Omega$ there exists exactly one $\stringinput \in \hat \Omega$ that corresponds to $\iota$. \label{crp:4a}
	\item In the standard model of arithmetic, for all $\stringinput \in \hat \Omega$, the statement $\Gamma(\stringinput) \in \MarkovXi(\stringinput)$ holds. \label{crp:4b}
	\item There exists an $\stringinput^0 \in \hat \Omega$ so that it is impossible to prove that $\Gamma(\stringinput^0) \in \MarkovXi(\stringinput^0)$ and that $\Gamma(\stringinput^0) \notin \MarkovXi(\stringinput^0)$ within ZFC. \label{crp:4c}
	\end{enumerate} \label{crp:4}
		
	\end{enumerate}
\end{theorem}

\begin{remark}
CRP IV provides a non-provability statement, which is a consequence of the non-computability in the Markov model that we establish. This is analogous (yet mathematically different) to G\"odel's first incompleteness theorem \cite{godel1931formal} -- specifically for optimisation. This is very similar to how the negative answer to Hilbert's 10th problem \cite{Hilberts10} (non-computability of diophantine equations) yields a statement of non-provability -- as in G\"odel's first incompleteness theorem -- specifically for diophantine equations \cite{franzen2005godel}.
\end{remark}

\begin{theorem}[CRP V]\label{thm:crp_5}
Consider the setup of \S \ref{section:setup_CRP}. Then, for the computational problem $\{\MarkovXi, \MarkovOmega, \M,\MarkovLambda\}$ as in Definition \ref{def:Markov_Delta_1_Information}, the following holds.

\begin{enumerate}[label=(\Roman*), ref = \Roman*, leftmargin = 8mm, series=CRP, resume = CRP]
		\item \label{crp:5} Define the function $\Xi^{\mv} \colon \MarkovOmega \to \{0,1\}$ by
		\begin{align*}
		\Xi^{\mv}(\stringinput) =
		\begin{cases}
		1 & \text{ if } |\MarkovXi(\stringinput)| = 1; \\
		0 & \text{ if } |\MarkovXi(\stringinput)| > 1.
		\end{cases}
		\end{align*}
		\begin{enumerate}
			\item \label{crp:5a}There exists a unique $\iota^0 \in \Omega$ such that for every $\stringinput \in \MarkovOmega$, if $\Xi^{\mv}(\stringinput) = 0$ then $\CorrXi(\stringinput) = \iota^0$ (where $\CorrXi$ is given as in Definition \ref{def:Markov_Delta_1_Information}).
			\item \label{crp:5b}There exists a $\kappa$-trustworthy AI with `giving up' parameter $\{\G_n\}_{n \in \N}$ (where $\G_n \colon \Omega \to \M \cup \{\idk\}$ for every $n \in \N$) such that its associated `I don't know' function satisfies $\Xi^{\Idk}_{\{\G_n\}} = \Xi^{\mv}$. 
			\item \label{crp:5c}There does not exist any $\kappa$-trustworthy AI of the form $\G \colon \MarkovOmega \to \M \cup \{\idk\}$ such that its associated `I don't know' function satisfies $\Xi^{\Idk}_{\G} = \Xi^{\mv}$.
		\end{enumerate}
	\end{enumerate}
\end{theorem}

\begin{remark}[`Provide a correct logical explanation for the solution']\label{rem:logic_exp2}
In the main part of the paper, \S \emph{``The Consistent Reasoning Paradox (CRP)''} (page 3), CRP IV and CRP V discuss algorithms that can (or cannot) `provide correct logical explanation' of their solutions. This is to be interpreted in the following way: we say that an algorithm $\G$ \emph{ can provide a correct logical explanation of its answer on input $\stringinput \in \MarkovOmega$} if there exists a proof within ZFC of the statement `$\Gamma(\stringinput) \in \MarkovXi(\stringinput)$'. In the CRP, statement \eqref{crp:4} of Theorem \ref{thm:crp_3_4} shows that for a specific algorithm there is an input for which which there is no proof of its correctness (nor of its negation) in ZFC, and thus this algorithm \emph{cannot} -- nor can any other -- provide a correct logical explanation of its answer on every input. However, such proof relies on an extra assumption on ZFC (namely, its $\Sigma_1$-soundness), and thus such proof is not carried out within ZFC but in a larger meta-theory. On the other hand, statement \eqref{crp:5b} of Theorem \ref{thm:crp_5} shows that there exists a $\kappa$-trustworthy AI with `giving up' parameter, which in particular must be $\kappa$-correct whenever it does not output $\idk$; such proof is carried out within ZFC itself, and thus proving \eqref{crp:5b} also automatically shows that there is a proof of \eqref{crp:5b} within ZFC. Therefore, such $\kappa$-trustworthy AI \emph{can} provide correct logical explanations of its solutions.
\end{remark}

\subsection{Proof of Theorem \ref{thm:iff}}

We now provide a proof of Theorem \ref{thm:iff}, which gives necessary and sufficient conditions for a candidate `I don't know' function $\Xi^* \colon \Omega \to \{0,1\}$ to be the `I don't know function' associated to a trustworthy AI.

\begin{proof}[Proof of Theorem \ref{thm:iff}]

	For the sake of brevity and to avoid repetition, we only give a proof of \eqref{conclusion:iff_2}. Statement \eqref{conclusion:iff_1} follows from a straightforward adaptation of this proof. 
We begin by showing one direction of the implication. Let $\Xi^*\colon \Omega \to \{0,1\}$ be a function such that there exists a $\kappa$-trustworthy AI with `giving up' parameter $\{\G_n\}_{n \in \N}$ such that $\Xi^* = \Xi^{\Idk}_{\{\G_n\}}$. We will now prove that \eqref{condition:sigma1} and \eqref{condition:b2} hold.

	\emph{Proof of (a):} First, we prove that $\{\Xi^*, \Omega, \{0,1\},\Lambda\} \in \Sigma_1^A$. For every $n \in \N$ define  $\tilde\Gamma_n\colon \Omega \to \{0,1\}$ to be the algorithm given by
		\begin{align*}
			\tilde\G_n(\iota) \coloneq 
			\begin{cases}
				1 & \text{ if } \G_n(\iota) \neq \idk; \\
				0 & \text{ if } \G_n(\iota) = \idk,
			\end{cases}
			\quad \iota \in \Omega.
		\end{align*}
		Since the map $(\{f_i(\iota)\}_{i = 1}^k,n) \mapsto \G_n(\iota)$ is recursive , it follows that the map  $(\{f_i(\iota)\}_{i = 1}^k,n) \mapsto \tilde \G_n(\iota)$ is also recursive.
		Fix $\iota \in \Omega$. Note that from the $\kappa$-trustworthiness of $\Gamma$ it follows that $\tilde\G_n(\iota) \leq \tilde\G_{n+1}(\iota)$ for every $n \in \N$. Indeed,  if $\tilde\G_n(\iota) = 1$, then $\G_n(\iota) \neq \idk$ by construction, and thus for every $n' \geq n$ we have that $\G_{n'}(\iota) \neq \idk$ by Definition \ref{def:trustworthy_ai}, which implies that $\tilde\G_{n'}(\iota) = 1$. We conclude that
		\begin{align*}
			\lim_{n \to \infty} \tilde\Gamma_n(\iota) = 
			\begin{cases}
			1 & \text{ if there is a } k\in \N \text{ such that } \G_k(\iota) \neq \idk\\
					0 & \text{ if for every }k \in \N \text{ we have } \G_k(\iota) = \idk
			\end{cases}
		\end{align*}
		and thus $\lim_{n \to \infty} \tilde\Gamma_n(\iota)=\Xi^{\Idk}_{\{\G_n\}}(\iota) = \Xi^*(\iota)$.
		In particular, we have that 
		\[
		\tilde\G_n \nearrow \Xi^* \, \Longrightarrow \, \{\Xi^*, \Omega, \{0,1\},\Lambda\} \in \Sigma_1^A.
		\]

		 \emph{Proof of (b):}  From $\Xi^* = \Xi^{\Idk}_{\{\G_n\}}$ it follows that 
\[
\Omega^*_1 = (\Xi^*)^{-1}(1) = (\Xi^{\Idk}_{\{\G_n\}})^{-1}(1)  = \{\iota \in \Omega \ | \ \exists n \in \N : \G_n(\iota) \neq \idk\} = \Omega_{\text{know}}
\]
 as per Definition \ref{def:idk_function_informal_Anders}. Given an $\iota \in \Omega^*_1$, we can find an $n_{\iota}$ such that $\G_{n_\iota}(\iota) \neq \idk$ in the following recursive way: iterate through $n \in \N$ until we find an $n$ such that $\G_{n}(\iota) \neq \idk$, and when we find such an $n$, set $n_{\iota} = n$. This approach is recursive since $\{\G_n\}_{n \in \N}$ is assumed to be a $\kappa$-trustworthy AI with `giving up' parameter and thus the map $(n,\iota) \mapsto \G_n(\iota)$ is recursive.
Thus, we can define the algorithm $\Gamma^* \colon \Omega_1^* \to \M$ as follows:
$
			\Gamma^*(\iota) \coloneq \G_{n_\iota}(\iota) 
$
		for every $\iota \in \Omega^*_1$ (i.e. every $\iota$ such that $\Xi^*(\iota)=1$). Then by the definition of $\kappa$-trustworthy AI as in Definition \ref{def:trustworthy_ai}, since $\G_{n_\iota}(\iota) \neq \idk$, we have $\Gamma^*(\iota) = \G_{n_\iota}(\iota) \in \ball{\kappa}{\Xi(\iota)}$. 
	This concludes one implication.
	
Now we prove the reverse implication. Assume that $\Xi^*\colon \Omega \to \{0,1\}$ is an `I don't know' function satisfying \eqref{condition:sigma1} and \eqref{condition:b2}. We proceed to define a $\kappa$-trustworthy AI with `giving up' parameter $\{\G_n\}_{n \in \N}$ so that $\Xi^* = \Xi^{\Idk}_{\{\G_n\}}$.
Observe from \eqref{condition:sigma1} that there exists a family of algorithms $\{\tilde\G_n\}_{n \in \N}$ such that $\tilde\G_n \colon \Omega \to \{0,1\}$, $\tilde\G_n(\iota) \nearrow \Xi^*(\iota)$ for every $\iota \in \Omega$, and the map $(n,\iota) \mapsto \tilde\G_n(\iota)$ is recursive. Moreover, from \eqref{condition:b2}, there exists an algorithm $\Gamma^*\colon \Omega^*_1 \to \M$ such that $\Gamma^*(\iota) \in \ball{\kappa}{\Xi(\iota)}$ for every $\iota \in \Omega^*_1$, i.e. every $\iota \in \Omega$ such that $\Xi^*(\iota) = 1$. We construct an AI with `giving up' parameter $\{\G_n\}_{n \in \N}$ where $\G_n \colon \Omega \to \M \cup \{\idk\}$ for every $n \in \N$ is defined as follows:
	\begin{align}\label{eq:ai_construction_n}
		\G_n(\iota) \coloneq 
		\begin{cases}
			\G^*(\iota) & \text{ if } \tilde\G_n(\iota) = 1 \\
			\idk & \text{ if } \tilde\G_n(\iota) = 0
		\end{cases}.
	\end{align}
	We claim that this is a $\kappa$-trustworthy AI with `giving up' parameter and that $\Xi^* = \Xi^{\Idk}_{\{\G_n\}}$.

	First, the map $(n,\iota) \mapsto \G_n(\iota)$ is recursive since $\G^*$ and the map $(n,\iota) \mapsto \tilde\G_n(\iota)$ are recursive. Secondly, whenever $\iota \in \Omega$ and $n_{\iota} \in \N$ are such that $\G_{n_{\iota}}(\iota) \neq \idk$, it then holds that $\G_{n'}(\iota)\neq \idk$ for every $n' \geq n_{\iota}$. To see this, recall that as a function of $n$, $\tilde\G_{n}(\iota)$ is increasing so that $\tilde\G_{n'}(\iota) \geq \tilde\G_{n_{\iota}}(\iota) = 1$ by assumption \eqref{condition:sigma1}, and thus $\G_{n'}(\iota)\neq \idk$. Furthermore, using the assumption on $\Gamma^*$ given by \eqref{condition:b2} we see that and $\G_{n'}(\iota) = \Gamma^*(\iota) \in \ball{\kappa}{\Xi(\iota)}$ for every $n' \geq n_{\iota}$. Thus $\{\G_n\}_{n \in \N}$ is a $\kappa$-trustworthy AI with `giving up' parameter according to Definition \ref{def:ai_informal}.

	Finally, we proceed to compute $\Xi^{\Idk}_{\{\G_n\}}$ and prove that it coincides with $\Xi^*$. Preliminarily, recall from \eqref{condition:sigma1} that $\tilde\G_n \nearrow \Xi^*$. Thus for every $n \in \N$, since $\tilde\G_n \leq \Xi^*$ and $\Xi^*$ has values in $\{0,1\}$, we have that $\tilde\G_n(\iota) = 1$ implies that $\Xi^*(\iota) = 1$. Conversely, since $\tilde\G_n \to \Xi$ pointwise and both $\tilde\G_n$ and $\Xi$ take values in the discrete set $\{0,1\}$, we see that $\Xi^*(\iota)=1$ implies that there exists $n \in \N$ such that $\tilde\G_n(\iota) = 1$. We have thus proven that $\Xi^*(\iota) = 1$ if and only if there exists $n \in \N$ such that $\tilde\G_n(\iota) = 1$. Therefore, by the construction of $\{\G_n\}_{n \in \N}$ in \eqref{eq:ai_construction_n}, we have
	 \begin{align*}
	 	\Omega_{\text{know}} =&  \{\iota \in \Omega \ | \ \exists n \in \N : \G_n(\iota) \neq \idk\}  \\
	 	=&  \{\iota \in \Omega \ | \ \exists n \in \N : \tilde\Gamma_n(\iota) = 1\}  =  \{\iota \in \Omega \ | \ \Xi^*(\iota) = 1\} = (\Xi^*)^{-1}(1)=\Omega^*_1, 
	 \end{align*}
	 and therefore $\Xi^{\Idk}_{\{\G_n\}} = 1_{\Omega_{\text{know}}} = 1_{\Omega^*_1} = \Xi^*$, concluding the proof of \eqref{conclusion:iff_2}.
\end{proof}

\subsection{Constructing the family of sentences for the CRP and the oracle problem}\label{sec:constructing}

Fix an enumeration $\varphi_1,\varphi_2,\varphi_3,\dots$ of all Turing machines. Each of them can be interpreted as being a partial function $\varphi_m \colon \subseteq \N \to \Q$ with domain $\operatorname{dom}(\varphi_m) = \{n \in \mathbb{N} \, \vert \,   \varphi_m(n)\downarrow \}$. Define, for $(m,t) \in \mathbb{N}^2$, the set
\begin{align*}
	\mathcal{W}(m,t) := \{n \in \mathbb{N} \, \vert \, \varphi_m \text{ halts on $n$ in exactly } t \text{ steps}\} \subseteq \mathbb{N}.
\end{align*}
Clearly, $\mathcal{W}$ is recursive and $\operatorname{dom}(\varphi_m) = \bigcup_{t \geq 0} \mathcal{W}(m,t)$ for every $m \in \N$. Note that, for each $m,n \in \N$ there is at most one $t\in \N$ so that $n \in \mathcal{W}(m,t)$.

	\begin{remark}[Crucial properties of $\mathcal{W}$]\label{rmk:properties_of_W}
		All results in the current section involving the function $\mathcal{W}$ (namely, Lemma \ref{lem:constructing_phi^m} and Propositions \ref{prop:DrivingNegativeProposition} and \ref{prop:EF}) do not rely on the precise definition of $\mathcal{W}$, but are based only on two properties that it satisfies:
		\begin{enumerate}[leftmargin = 8mm, label = (P\roman*)]
			\item for all $m \in \mathbb{N}$, $\varphi_m(m) \downarrow$ if and only if there exists $t \in \N$ such that $m \in \mathcal{W}(m,t)$; \label{item:W_property1}
			\item for all $m \in \mathbb{N}$,  there is at most one $t \in \N$ such that $m \in \mathcal{W}(m,t)$.\label{item:W_property2}
		\end{enumerate}
		Therefore, the function $\mathcal{W}$ could be replaced by any other function $\mathcal{W}'$ satisfying the same properties \ref{item:W_property1} and \ref{item:W_property2}, and all the relevant results would still apply. We will exploit this in Section 5, wherein we produce a computer code that has these two properties, but does not rely on evaluations of the number of steps a Turing Machine takes and is thus more relevant for practical computer programming
	\end{remark}

The following Lemma allows to construct a universal family of inputs in $\MarkovOmega$ that will be referenced throughout most of the subsequent results and proofs. The importance of this family of inputs is that, for any given algorithm, we can find an input in this family on which the algorithm is guaranteed to fail (this will be the content of Proposition \ref{prop:DrivingNegativeProposition}).

\begin{lemma}\label{lem:constructing_phi^m}
	Let $\{\Xi, \Omega, \mathcal{M}, \Lambda\}$ be a computational problem with $\Lambda$ finite so that $\Lambda=\{f_i\,\vert\,i\in\mathbb{N}, i\leq k\}$ for some $k \in \mathbb{N}$, and with $\mathcal{M}$ a subset of  $\R^d$ for some dimension $d$. Let $\{\iota^1_n\}_{n=1}^{\infty}$, $\{\iota^2_n\}_{n=1}^{\infty}$ be sequences in $\Omega$ and $\iota^0 \in \Omega$. Suppose that the following conditions hold.
	\begin{enumerate}[leftmargin=8mm, label=(\alph*)]
		\item For all $i \in \{1,2,\dotsc,k\}$ and $j \in \{1,2\}$, there exist algorithms $\hat \Gamma^j_i$ such that $\hat \Gamma^{j}_i: \mathbb{N} \times \mathbb{N} \to \Q$ with $|\hat \Gamma^{j}_i(n,r) - f_i(\iota^{j}_r)| \leq 2^{-n-1}$, as well as an algorithm $\hat \Gamma^0_i: \mathbb{N} \to \Q$ with $|\hat \Gamma^{0}_i(n) - f_i(\iota^{0})| \leq 2^{-n-1}$ for all $n \in \N$. \label{assumption:ComputableInput_lemma}
		\item $|f_i(\iota^j_n)-f_i(\iota^0)|\leq 2^{-n}$ for all $j \in \{1,2\}$, every $n \in\mathbb{N}$, and every $i\in \{1,2,\dotsc,k\}$. \label{assumption:Del1withIota0_lemma}
	\end{enumerate}
	For each $m\in \N$ and $i \in \{1,\dotsc,k\}$, define the following functions $\phi^m_i \colon \N \to \Q$:
		\begin{equation}\label{eq:def_phiDerived}
		\phi^m_i(n):= \begin{cases} \hat\Gamma^1_i(n,t) & \text{ if } [m \in \mathcal{W}(m,t)] \land [t \leq n] \land [\varphi_m(m) = 1];\\ 
			\hat\Gamma^2_i(n,t) & \text{ if } [m \in \mathcal{W}(m,t)] \land [t \leq n] \land [\varphi_m(m) = 2];\\
			\hat \Gamma^0_i(n) & \text{ otherwise.}
		\end{cases}
	\end{equation}
	Then the following conclusions hold:
	\begin{enumerate}[leftmargin = 8mm]
		\item For every $m \in \N$ and $i \in \{1,\dotsc,k\}$, $\phi^m_i$ can be implemented as a Turing machine; \label{conclusion:TM}
		\item For every $m \in \N$, $\stringinputm \in \MarkovOmega$ and in particular
			\begin{align}\label{eq:EFInitialIotaCorrespond}
			\{\phi^m_i(n)\}_{i \in \{1,2,\dotsc,k\}, n \in \N} \,\text{ corresponds to } \begin{cases}  \iota^1_t \text{ for some } t \in \mathbb{N} &\text{if } \varphi_m(m) = 1;\\
				\iota^2_t \text{ for some } t \in \mathbb{N} & \text{if } \varphi_m(m) = 2; \\
				\iota^0 & \text{if }\varphi_m(m)\notin \{1,2\} \lor\varphi_m(m) \uparrow
			\end{cases}
		\end{align}
		in the sense of Definition \ref{def:Markov_Delta_1_Information} \label{conclusion:provides_delta_1_information}.
	\end{enumerate}
	
\end{lemma}		
\begin{proof}
	For every $m \in \N$ and $i \in \{1,\dotsc,k\}$, each function $\phi^m_i\colon \N \to \Q$ can be implemented as a Turing machine as follows: $\phi^m_i(n)$ checks whether for some $t \leq n$ it holds that $m \in \mathcal{W}(m,t)$ (at most one such $t$ exists) and if so, computes the value of $\varphi_m(m)$; if the value is $j \in \{1,2\}$, then $\phi^m_i(n)$ returns $\hat\Gamma^j_i(n,t)$, while for any other value of $\varphi_m(m)$, or if there is no $t\leq n$ such that $m \in \mathcal{W}(m,t)$, then $\phi^m_i(n)$ returns $\hat \Gamma^0_i(n)$. This proves \eqref{conclusion:TM}.
	By definition, $\{\phi^m_i(n)\}_{n\in\N}$ is equal to one of the following:
	{\small
	\begin{equation*}\label{eq:EFInitialIotaDef}
		\begin{split}
			\{\phi^m_i(n)\}_{n \in \N} \coloneq \begin{cases}  \{  \hat \Gamma^0_i(1), \hat \Gamma^0_i(2),\dotsc,\hat \Gamma^0_i(t-1),\hat\Gamma^1_i(t,t),\hat\Gamma^1_i(t+1,t),\dotsc  \} & \text{if } \varphi_m(m) = 1\\ & \qquad \quad \land m \in \mathcal{W}(m,t);\\
				\{  \hat \Gamma^0_i(1), \hat \Gamma^0_i(2),\dotsc,\hat \Gamma^0_i(t-1),\hat\Gamma^2_i(t,t),\hat\Gamma^2_i(t+1,t),\dotsc  \} & \text{if } \varphi_m(m) = 2 \\ & \qquad \quad \land m \in \mathcal{W}(m,t); \\
				\{ \hat \Gamma^0_i(1), \hat \Gamma^0_i(2),\dotsc,\hat \Gamma^0_i(t-1),\hat \Gamma^0_i(t),\hat \Gamma^0_i(t+1)\dotsc  \} & \text{if }\varphi_m(m)\notin \{1,2\} \\ & \qquad \quad \lor \varphi_m(m) \uparrow.
			\end{cases}
		\end{split}
	\end{equation*}
	}
	We now proceed to show \eqref{conclusion:provides_delta_1_information} by considering two cases: 
	
	 \emph{Case (I)}: if $\varphi_m(m) \notin \{1,2\} \lor \varphi_m(m) \uparrow$, then $|\phi^m_i(n) - f_i(\iota^0)| = | \hat \Gamma^0_i(n) - f_i(\iota^0)| \leq 2^{-n-1}\leq 2^{-n}$ by \ref{assumption:ComputableInput} for all $n$ and $i \in \{1,\dotsc,k\}$, proving that $\{\phi^m_i(n)\}_{i \in \{1,2,\dotsc,k\}, n \in \N}$ corresponds to $\iota^0$;
	 
	 \emph{Case (II)}: if $\varphi_m(m) =j$ for $j \in \{1,2\}$, then $\{\phi^m_i(n)\}_{i \in \{1,2,\dotsc,k\}, n \in \N}$ corresponds to $\iota^j_t$ (where $t$ is uniquely determined by $m \in \mathcal{W}(m,t)$) since for $n \geq t$, $|\phi^m_i(n) - f_i(\iota^j_t)| = | \hat \Gamma^j_i(n,t) - f_i(\iota^j_t)| \leq 2^{-n-1}\leq 2^{-n}$ by \ref{assumption:ComputableInput}, and for $n < t$, 
	\[
	|\phi^m_i(n) - f_i(\iota^j_t)| = |\hat \Gamma^0_i(n) - f_i(\iota^j_t)| \leq |\hat \Gamma^0_i(n) - f_i(\iota^0)| + |f_i(\iota^0) - f_i(\iota^j_t)| \leq 2^{-n-1} + 2^{-n-1} = 2^{-n},
	\]
	 by \ref{assumption:ComputableInput} and \ref{assumption:Del1withIota0}. 
	We conclude that $\stringinputm \in \MarkovOmega$ and that \eqref{eq:EFInitialIotaCorrespond} holds. Thus \eqref{conclusion:provides_delta_1_information} is proven.
\end{proof}

Next, we consider computational problems with an oracle. The following result shows that, under suitable assumptions, if there is an algorithm that achieves perfect accuracy when solving a Markov problem with oracle, it is possible to design an algorithm that can solve the same problem without needing access to any oracle.

\begin{proposition}[Removing an oracle]\label{prop:de-oracolisation} 
Let $\Xi_1 \colon \Omega \rightrightarrows \mathcal{M}_1$ and $\Xi_2 \colon \Omega \rightrightarrows \mathcal{M}_2$ be two solution maps defined on the same input set $\Omega$ with $\M_2 \subseteq \R^d$, and let $\MarkovXi_j\colon \MarkovOmega \rightrightarrows \M$ be given as in Definition \ref{def:Markov_Delta_1_Information} for $j \in \{1,2\}$.  
Furthermore, let $\{\iota^1_n\}_{n=1}^{\infty}$, $\{\iota^2_n\}_{n=1}^{\infty}$ be sequences in $\Omega$ and $\iota^0 \in \Omega$, and let $\Omega' \coloneq \{\iota^1_n\}_{n \in \N} \cup \{\iota^2_n\}_{n \in \N} \cup \{\iota^0\} \subseteq \Omega$. Suppose that $\subsetMarkov$ is a set satisfying $\subsetMarkov \subseteq (\Omega')^M$. Fix $\omega > 0$ and consider the computational problem with oracle
$\{\MarkovXi_1,\subsetMarkov,\M_1,\MarkovLambda\}^{\mathcal{O}}$
(see Definition \ref{def:problem_with_oracle}), where 
	\begin{equation}\label{eq:MarkovXiOracleDef}
			\MarkovXiOracle_1(\stringinput,y) =
			\begin{cases}
					\MarkovXi_1(\stringinput) & \text{ if } y \in \ballQ{\MarkovXi_2(\stringinput)}; \\
					\MarkovXi_1(\subsetMarkov) &\text{ if } y\notin \ballQ{\MarkovXi_2(\stringinput)} ,
				\end{cases} \, \text{ for } (\stringinput,y) \in \subsetMarkov \times \ballQ{\MarkovXi_2(\subsetMarkov)}.
		\end{equation}
 Assume that the following conditions hold.
	\begin{enumerate}[label=(\alph*), leftmargin = 8mm]
		\item  For $j \in \{1,2\}$, there exists $y^j$ so that $y^j \in \ballQ{\Xi_2(\iota^j_n))} \cap  \ballQ{\Xi_2(\iota^0)}$ for all $n \in \N$. \label{assumption:exists_xj_in_intersection_correspondence} 
		\item There is an algorithm $\largerGamma: \subsetMarkov \to \{1,2\}$ such that $\largerGamma(\stringinput) = j$ if $\CorrXi(\stringinput) = \iota^j_n$ for some $n \in \N$, and $\largerGamma(\stringinput) \uparrow$ otherwise (recall that $\CorrXi\colon \MarkovOmega \to \Omega$ is the correspondence function as per Definition \ref{def:correspondence_problem}). \label{assumption:exists_algorithm_check_j}
		\item $\Xi_1$ is single-valued.\label{assumption:Xi1_single_valued}
		\item There is an algorithm $\GO\colon \subsetMarkov \times \ballQ{\MarkovXi_2(\subsetMarkov)} \to \M_1$  such that $\GO(\stringinput,y) \in \MarkovXiOracle_1(\stringinput,y)$ for every $(\stringinput, y) \in \subsetMarkov \times \ballQ{\MarkovXi_2(\subsetMarkov)}$. \label{assumption:exists_correct_oracle_algorithm} 
	\end{enumerate}
	Then there exists an algorithm $\Gamma \colon \subsetMarkov \to \M_1$ such that $\Gamma(\stringinput) = \MarkovXi_1(\stringinput)$ for every $\stringinput \in \subsetMarkov$.
\end{proposition}

\begin{proof} 		
	By the definition of $\Omega'$, we observe that $\subsetMarkov \subseteq (\Omega')^M = (\CorrXi)^{-1}(\{\iota^1_n\}_{n \in \N} \cup \{\iota^2_n\}_{n \in \N} \cup \{\iota^0\})$, so that every string $\stringinput \in \subsetMarkov$ corresponds to either $\iota^0$, or to $\iota^1_n$ or $\iota^2_n$ for some $n \in \N$. 
	Define the following algorithm $\Gamma:\subsetMarkov \to \M_1$ 
	\begin{equation}\label{eq:Gamma_deoracolised}
		\begin{split}
			\Gamma(\stringinput) \coloneq
			\begin{cases}
				\GO(\stringinput,y^1) & \text{ if } \GO(\stringinput,y^1) = \GO(\stringinput,y^2);\\
				\GO(\stringinput,y^j) & \text{ if }  \left( \GO(\stringinput,y^1) \neq \GO(\stringinput,y^2) \right)\land j \in \{1,2\} \text{ is such that } \largerGamma(\stringinput)=j .
			\end{cases}
		\end{split}	
	\end{equation}
Note that the algorithm $\Gamma$ is well defined. In particular, we now prove that if $\stringinput$ is such that $\GO(\stringinput,y^1) \neq \GO(\stringinput,y^2)$, then $\largerGamma(\stringinput)=j \text{ for some } j \in \{1,2\}$.

	Assume otherwise for the sake of contradiction. Then by the definition of $\largerGamma$, $\largerGamma(\stringinput) \uparrow$ and $\stringinput$ does not correspond to any $\iota^j_n$ for $j \in \{1,2\}$, $n \in \N$. Therefore we must have that $\stringinput$ corresponds to $\iota^0$. Then using assumption \ref{assumption:exists_xj_in_intersection_correspondence} we have $y^j \in \ballQ[\omega]{\Xi_2(\iota^0)}=\ballQ[\omega]{\MarkovXi_2(\stringinput)}$ for $j \in \{1,2\}$. Thus by the definition in \eqref{eq:MarkovXiOracleDef} we see that $\MarkovXiOracle_1(\stringinput,y^1) = \MarkovXi_1(\stringinput) = \MarkovXiOracle_1(\stringinput,y^2)$. Recalling that $\MarkovXi_1(\stringinput) = \Xi_1(\iota^0)$ is single-valued by assumption \ref{assumption:Xi1_single_valued}, and using assumption \ref{assumption:exists_correct_oracle_algorithm}, we deduce that
	\begin{align*}
	\GO(\stringinput,y^1) = \MarkovXiOracle_1(\stringinput,y^1) = \Xi_1(\iota^0) = \MarkovXiOracle_1(\stringinput,y^2) = \GO(\stringinput,y^2),
	\end{align*}
	contradicting the assumption that $\GO(\stringinput,y^1) \neq \GO(\stringinput,y^2)$.
	Therefore, $\stringinput$ does not correspond to $\iota^0$. Consequently, from the definition of $\subsetMarkov$ we see that $\stringinput$ corresponds either to $\iota^1_n$ or to $\iota^2_n$ for some $n$. In either case, we must have that $\largerGamma(\stringinput)$ halts by assumption \ref{assumption:exists_algorithm_check_j} and in particular that $\largerGamma(\stringinput)=j$ where $j \in \{1,2\}$ is such that $\stringinput$ corresponds to $\iota^j_n$ for some $n \in \N$. Thus $\G(\stringinput)$ is well defined for every $\stringinput \in \subsetMarkov$.
	
	We now claim that the algorithm $\Gamma: \subsetMarkov \to \M_1$ solves the computational problem $\{\MarkovXi_1, \subsetMarkov, \M_1,\MarkovLambda\}$.
	Indeed, for every $\stringinput \in \subsetMarkov$, we have the following cases: 

			 \emph{Case (I)}: If $\GO(\stringinput, y^1) \neq \GO(\stringinput, y^2)$, then, by the reasoning above, $\CorrXi(\stringinput)\neq \iota^0$ and thus $\stringinput$ corresponds to $\iota^j_n$ for some $n \in \N$ and $j \in \{1,2\}$, and $\largerGamma(\stringinput)=j$. Thus, by assumption \ref{assumption:exists_xj_in_intersection_correspondence}, we have  $y^j \in \ballQ[\omega]{\Xi_2(\iota^j_n)}=\ballQ[\omega]{\MarkovXi_2(\stringinput)} $. In particular, by definitions \eqref{eq:MarkovXiOracleDef}, \eqref{eq:Gamma_deoracolised} and assumption \ref{assumption:exists_correct_oracle_algorithm} we see that 
		\[
		\Gamma(\stringinput) = \GO(\stringinput,y^j) \in  \MarkovXiOracle_1(\stringinput,y^j) = \MarkovXi_1(\stringinput).
		\]
		Thus from \ref{assumption:Xi1_single_valued} we conclude that $\Gamma(\stringinput) = \MarkovXi_1(\stringinput)$ as desired.
		
		 \emph{Case (II)}: If instead $\GO(\stringinput, y^1) = \GO(\stringinput, y^2)$, then there exists $j \in \{1,2\}$ such that $y^j \in \ballQ[\omega]{\MarkovXi_2(\stringinput)} $. In fact, by the definition of $\subsetMarkov$ then either $\stringinput$ corresponds to $\iota^j_n$ for some $n\in \N$ and $j \in \{1,2\}$, in which case $y^j \in \ballQ[\omega]{\Xi_2(\iota^j_n)} = \ballQ[\omega]{\MarkovXi_2(\stringinput)}$ by assumption \ref{assumption:exists_xj_in_intersection_correspondence}; or $\stringinput$ corresponds to $\iota^0$, in which case then both $y^j \in  \ballQ[\omega]{\Xi_2(\iota^0)} = \ballQ[\omega]{\MarkovXi_2(\stringinput)} $ for $j \in \{1,2\}$, again by assumption \ref{assumption:exists_xj_in_intersection_correspondence}. Thus, from definitions \eqref{eq:MarkovXiOracleDef}, \eqref{eq:Gamma_deoracolised} and assumption \ref{assumption:exists_correct_oracle_algorithm} we see that for the value of $j$ as above:
		\[
		\Gamma(\stringinput) = \GO(\stringinput, y^1) = \GO(\stringinput, y^j) \in \MarkovXiOracle_1(\stringinput,y^j)  = \MarkovXi_1(\stringinput).
		\]
		Thus from \ref{assumption:Xi1_single_valued} we deduce that $\Gamma(\stringinput) = \MarkovXi_1(\stringinput)$ as desired.
	Therefore, the algorithm $\Gamma: \subsetMarkov \to \M_1$ solves the computational problem $\{\MarkovXi_1, \subsetMarkov, \M_1,\MarkovLambda_1\}$.
\end{proof}

\subsection{The art of de-randomising randomised algorithms -- The build-up to CRP IIIb}

In this section, we focus our attention on randomised algorithms aimed at solving computational problems. Recalling the definitions and notation presented in \S \ref{sec:probabilistic}, we consider a probabilistic Turing machine $\Gamma^{\ran}$. For every fixed $\iota \in \Omega$, we denote for $y \in \M$ and $t \in \N$:
\begin{align*}
	\strings{y} &\coloneq \{\beta \in \{0,1\}^\N \ | \ \G^{\ran}(\iota,\beta) = y\}\subseteq \{0,1\}^\N; \\
	\finitestrings{y}{t} &\coloneq \{\sigma \in \{0,1\}^t \ | \ \Gamma^{\ran}(\iota,\sigma) = y\} \subseteq \{0,1\}^t; \\
	\out &\coloneq \{y \in \mathcal{M} \ | \ \text{ there exists } \beta \in \{0,1\}^\N \text{ such that } \ \Gamma^{\ran}(\iota,\beta)=y\}; \\
	\out(t) &\coloneq \{y \in \mathcal{M} \ | \ \text{ there exists } \beta \in \{0,1\}^\N \text{ such that } \ \halted{t}=y\}.
\end{align*} 
We start with the following lemma, which is a straightforward manipulation of the definitions and notation presented above and in \S \ref{sec:probabilistic}

\begin{lemma}\label{lemma:basic_properties}
	For every $y \in \mathcal{M}$, $t \in \N$, the following holds:
	\begin{enumerate}[leftmargin = 6mm]
		\item $\brackets{\finitestrings{y}{t}} \subseteq \brackets{\finitestrings{y}{t+1}}$; \label{propConclusion:SIncreasing}
		\item \label{propConclusion:countable_union}$\strings{y} = \bigcup_{t \in \N} \brackets{\finitestrings{y}{t}}$;
		\item $\out(t) \subseteq \out(t+1)$;
		\item $\out = \bigcup_{t \in \N}\out(t)$;
		\item \label{propConclusion:measurable} Both $S(y)$ and $\brackets{S(y,t)}$ are measurable in $\{0,1\}^\N$;
		\item $\mu_{\rho}(\strings{y}) = \sup_{t \in \N} \mu_\rho(\brackets{\finitestrings{y}{t}})$;
		\item \label{propConclusion:Outt}If $\out(t)=\out$, then $\strings{y} = \brackets{\finitestrings{y}{t}}$.
	\end{enumerate}
\end{lemma}
\begin{proof}
	Properties \eqref{propConclusion:SIncreasing} through \eqref{propConclusion:Outt} follow easily from $\Gamma^{\ran}$ being a Probabilistic Turing machine. For \eqref{propConclusion:measurable}, observe that $\brackets{\finitestrings{y}{t}}$ is measurable by virtue of being a finite union of cylinder sets, which are measurable; and thus $\strings{y}$ is measurable too, as countable union of measurable sets by \eqref{propConclusion:countable_union}. 
\end{proof}

\begin{proposition}[Measurability]\label{prop:measurable}
	Let $\Gamma^{\ran}$ be a probabilistic Turing machine with respect to the measure $\mu$, and let $\rho$ be a computable pre-measure such that $\mu=\mu_{\rho}$. For every $\iota \in \Omega$ the function $\Gamma^{\ran}(\iota,\cdot)\colon \{0,1\}^\N \to \M$ is measurable with respect to the $\sigma$-algebra on $\{0,1\}^\N$ induced by the pre-measure $\rho$, and to the Borel $\sigma$-algebra $\mathcal{B}(\M)$. 
\end{proposition}

\begin{proof}
	Fix $\iota \in \O$ and let $E \in \mathcal{M}$ be a measurable set. Recall that $\Gamma^{\ran}(\iota,\cdot)\colon \{0,1\}^\N \to \dense \subseteq \M$, with $\Q^d = \{y_m\}$. Then
	\begin{align*}
		\big(\Gamma^{\ran}(\iota,\cdot)\big)^{-1}(E) & = \{ \beta \in \{0,1\}^\N \ | \ \Gamma^{\ran}(\iota,\beta) \in E\} \\
		& = \{ \beta \in \{0,1\}^\N \ | \ \Gamma^{\ran}(\iota,\beta) \in E \cap \dense\} \\
		& = \bigcup_{m \in \N} \{ \beta \in \{0,1\}^\N \ | \ \Gamma^{\ran}(\iota,\beta) \in E \cap \{y_m\}\}  = \bigcup_{\ m \in \N \text{ s.t. } y_m \in E}  S(y_m)
	\end{align*}
	Since countable unions of measurable sets are measurable, the conclusion follows from the fact that $\strings{y_m}$ is measurable for every $m \in \N$ by Lemma \ref{lemma:basic_properties}.
\end{proof}

We will now prove that any single-valued function that can be computed by a probabilistic Turing machine with probability $>1/2$ can be computed by a deterministic Turing machine, without randomised tape. The proof is based on a `majority voting' argument. We will also prove the analogous results for a certain class of multi-valued functions and PTMs that always halt.

\begin{proposition}[De-randomisation]\label{prop:de-randomisation}
	Let $\{\Xi,\O, \mathcal{M}, \Lambda\}$ be a computational problem. 	
	\begin{enumerate}[leftmargin=5.5mm]
		\item Suppose that $\Xi\colon \Omega \to \M$ is a single-valued function. If there exists a probabilistic Turing machine $\Gamma^{\ran}$ such that, 
		\begin{align}\label{eq:single_valued_probability_greater_1_2}
			\pr(\G^{\ran}(\iota)=\Xi(\iota)) >\frac{1}{2} \quad \text{ for every $\iota \in \Omega$}
		\end{align}
		then there exists a deterministic Turing machine $\Gamma\colon\Omega \to \M$ (without randomised tape) such that $\Gamma(\iota) = \Xi(\iota)$ for every $\iota \in \Omega$.
		\label{prop:de-randomisation_single_valued}
		
		\item Let $\Xi \colon \Omega \rightrightarrows \M$ be a multi-valued function such that there exists $y_0 \in  \Q^d$ so that $y_0 \in \Xi(\iota)$ whenever $\iota \in \O$ is such that $|\Xi(\iota)|>1$.
		Assume that there exists a $p > 1/2$ and a probabilistic Turing machine $\Gamma^{\ran}$ that always halts (as per Definition \ref{def:TuringMachineAlwaysHalt}) and such that, for every $\iota \in \Omega$
		\begin{align}\label{eq:multi_valued_probability_greater_p}
			\pr(\G^{\ran}(\iota)\in\Xi(\iota)) \geq p
		\end{align}
		Then there exists a deterministic Turing machine $\Gamma'$ such that $\Gamma'(\iota) \in \Xi(\iota)$ for every $\iota \in \O$.
		\label{prop:de-randomisation_multi_valued}
	\end{enumerate}
\end{proposition}

\begin{remark}
	Note that, given a computational problem $\{\Xi, \O, \mathcal{M}, \Lambda\}$ satisfying the above conditions and a subset $\Omega_0 \subseteq \Omega$, the previous result can be applied to the computational problem restricted on $\Omega_0$ given by $\{\Xi|_{\Omega_0}, \Omega_0, \mathcal{M}, \Lambda|_{\Omega_0}\}$. Therefore, if a random algorithm has a probability of success greater than $\frac{1}{2}$ even on a smaller subset of inputs, then there already exists a deterministic algorithm that solves the problem on such inputs.
\end{remark}

\begin{proof}[Proof of Proposition \ref{prop:de-randomisation}]
	We denote by $\mu$ the computable measure with respect to which each Turing machine $\Gamma^{\ran}$ is defined. Since $\mu$ is by definition computable, there exists a computable pre-measure $\rho$ such that $\mu = \mu_\rho$ as per Definition \ref{def:computable_measure}. Thus there exists a recursive function $r$ as in Definition \ref{def:computable_measure} that satisfies the approximating condition \eqref{eq:approximation_r}.
	
	\textit{Part \eqref{prop:de-randomisation_single_valued}}: We begin by constructing the algorithm $\Gamma$. On input $\iota \in \O$, $\Gamma$ will do the following, where $t$ is initialised with $t = 1$:
	\begin{algosteps}
		\item Run $\Gamma^{\ran}(\iota, \sigma)$ for all of the $2^t$ binary sequences $\sigma \in \{0,1\}^t$ of length $t$, obtaining outputs $\out(t)=\{y_1,\dots,y_l\}\subseteq\M$; \label{step:run_t}
		\item if there exists $y \in \out(t)$ such that $r(\finitestrings{y}{t},t) > 1/2 + 2^{-t}$, return $\Gamma(\iota)\coloneq y$; \label{step:return}
		\item otherwise, increase $t$ to $t+1$ and repeat from \ref{step:run_t} \label{step:increase_t}.
	\end{algosteps}
	
We now proceed to verify the correctness of $\Gamma$. Fix $\iota \in \Omega$. We will show that $\Gamma(\iota)=\Xi(\iota)$. In particular, we must prove that $\Gamma$ halts on $\iota$ input and that it returns the same value as $\Xi(\iota)$. 
	Firstly, we prove that $\Gamma$ halts on $\iota$.
		By assumption, $\mu_\rho(\strings{\Xi(\iota)}) > \frac{1}{2}$, so there exists $\delta>0$ such that $\mu_\rho(\strings{\Xi(\iota)}) > \frac{1}{2}+\delta$. 
		Moreover, $\frac{1}{2} + \delta < \mu_\rho(\strings{\Xi(\iota)}) = \sup_{t \in \N} \mu_\rho(\brackets{\finitestrings{\Xi(\iota)}{t}})$, which implies by the definition of supremum that for sufficiently large $t$, $\mu_\rho(\brackets{\finitestrings{\Xi(\iota)}{t}}) > \frac{1}{2} + \delta$ (and in particular, that $\Xi(\iota) \in \out(t)$). Again, assuming $t$ is sufficiently large and using the fact that $r$ satisfies \eqref{eq:approximation_r}, it follows that
		\[
		r(\finitestrings{\Xi(\iota)}{t},t) \geq \mu_\rho(\brackets{\finitestrings{\Xi(\iota)}{t}})  - 2^{-t}> \frac{1}{2} + \delta - 2^{-t}  >\frac{1}{2} + 2^{-t}.
\]	
 We conclude from the definition of $\Gamma$ that $\Gamma$ halts on $\iota$.
	Finally, we proceed to prove that $\Gamma(\iota)=\Xi(\iota)$. By contradiction, if $\Gamma(\iota) = y$ for some $y \neq \Xi(\iota)$, then by the halting condition of $\Gamma$ there must exist $t \in \N$ such that $r(\finitestrings{y}{t},t) > \frac{1}{2} + 2^{-t}$; but by \eqref{eq:approximation_r}, this would imply that $\mu_\rho(\strings{y}) \geq \mu_\rho(\brackets{\finitestrings{y}{t}}) \geq r(\finitestrings{y}{t},t) - 2^{-t} > \frac{1}{2}$. However, we also observe that, by assumption \eqref{eq:single_valued_probability_greater_1_2}, 
$
\mu_{\rho}(\strings{\Xi(\iota)}) = \pr(\Gamma^{\ran}(\iota) = \Xi(\iota)) > \frac{1}{2},
$
thus if $y \neq \Xi(\iota)$, then
	\begin{align*}
		\mu_\rho(\strings{y}) = \pr(\Gamma^{\ran}(\iota) = y) \leq \pr(\Gamma^{\ran}(\iota)\neq \Xi(\iota))= 1 - \mu_\rho(\strings{\Xi(\iota)}) < 1-\frac{1}{2} =  \frac{1}{2},
	\end{align*}   
	which is a contradiction.This concludes the verification that  $\Gamma(\iota) = \Xi(\iota)$ for every $\iota \in \Omega$.
	
	\vspace{5pt}
	
	\textit{Part \eqref{prop:de-randomisation_multi_valued}:} First note that if $n_0 > -\log_2\big(p-\frac{1}{2}\big)$, then by \eqref{eq:approximation_r} we have
	\begin{align}\label{eq:approximation_n0}
		|r(\sigma_1,\dotsc,\sigma_k,n_0) - \mu_\rho(\brackets{\sigma_1,\dotsc,\sigma_k}) | \leq 2^{-n_0} < p-\frac{1}{2} \quad \text{ for every } \sigma_1,\dotsc,\sigma_k \in \{0,1\}^*.
	\end{align}
We now construct $\Gamma'$. Fix an arbitrary $n_0 > -\log_2\big(p-\frac{1}{2}\big)$. On input $\iota \in \O$, $\Gamma'$ will do the following, where $t$ is initialised with $t = 1$:
	\begin{algosteps}
		\item Run $\Gamma^{\ran}(\iota, \sigma)$ for all of the $2^t$ binary sequences $\sigma \in \{0,1\}^t$ of length $t$, obtaining outputs $\out(t)=\{y_1,\dots,y_l\}\subseteq\M$; \label{step:run_t'}
		\item if there is $y \in \out(t)$ such that $r(\finitestrings{y}{t},n_0) > \frac{1}{2}$, halt and return $\Gamma'(\iota)\coloneq y$; \label{step:return'}
		\item if instead $\Gamma^{\ran}(\iota, \sigma)\downarrow$ for all $\sigma \in \{0,1\}^t$, halt and return $\Gamma'(\iota) \coloneq y_0$; \label{step:alternative_gamma'_return}
		\item otherwise, if neither of the above cases have occurred, increase $t$ to $t+1$ and repeat from \ref{step:run_t}.
	\end{algosteps}  
	
	We now verify the correctness of $\Gamma'$. Fix $\iota \in \Omega$. We will show that $\Gamma'(\iota)\in\Xi(\iota)$. In particular, we must prove that $\Gamma$ halts, and that it returns a value belonging to the multi-valued function $\Xi$. 
	Firstly, we note that $\Gamma'$ halts on $\iota$. This is an immediate consequence of the fact that $\Gamma^{\ran}$ always halts: for every $\iota \in \Omega$ there exists $t_0 \in \N$ such that $\Gamma^{\ran}(\iota,\sigma) \downarrow$ for every $\sigma \in \{0,1\}^{t_0}$ (since otherwise, there would be an infinite $\beta \in \{0,1\}^\N$ such that $\Gamma^{\ran}(\iota,\beta) \uparrow$ which contradicts Definition \ref{def:TuringMachineAlwaysHalt} and the assumption that $\Gamma^{\ran}$ always halts). In particular, if $\Gamma'$ has not halted before \ref{step:alternative_gamma'_return} with $t = t_0$ then $\Gamma'$ will halt at this step.
	
		We now proceed to prove that $\Gamma'(\iota) \in \Xi(\iota)$. As a first step, we will argue that if $\Gamma'$ halts on \ref{step:return'} then $\Gamma'(\iota) \in \Xi(\iota)$. Let $y = \Gamma'(\iota)$. By construction there must exist $t \in \N$ such that $r(\finitestrings{y}{t},n_0) >\frac{1}{2}$. By the approximating property of $r$, this implies that 
\begin{align*}
	\mu_\rho(\brackets{\finitestrings{y}{t}}) > r(\finitestrings{y}{t},n_0) - \Big(p-\frac{1}{2}\Big) > \frac{1}{2} - p + \frac{1}{2} = 1-p,
\end{align*}
and therefore $\mu_\rho(\strings{y}) \geq \mu_\rho(\brackets{\finitestrings{y}{t}}) > 1-p$. Assume for the sake of contradiction that $y = \Gamma'(\iota) \notin \Xi(\iota)$. By assumption \eqref{eq:multi_valued_probability_greater_p} we know that $\pr(\Gamma^{\ran}(\iota) \in \Xi(\iota)) \geq p$, so that
\begin{align*}
	\mu_\rho(\strings{y}) = \pr(\Gamma^{\ran}(\iota) = y) \leq \pr(\Gamma^{\ran}(\iota)\notin \Xi(\iota))= 1 - \pr(\Gamma^{\ran}(\iota) \in \Xi(\iota)) \leq 1-p,
\end{align*}
which is a contradiction. We have thus shown that if $\Gamma'$ halts on \ref{step:return'} then  $y=\Gamma'(\iota)\in\Xi(\iota)$.
We now consider two cases, depending on the cardinality of $\Xi(\iota)$.

\emph{Case (I)}: If $|\Xi(\iota)|>1$, then $\Gamma'$ halts at either \ref{step:return'} or \ref{step:alternative_gamma'_return}. In the first case, we have just observed that $\Gamma'(\iota) \in \Xi(\iota)$; whereas in the second case, by construction $\Gamma'$ returns $\Gamma'(\iota)\coloneq y_0 \in \Xi(\iota)$ by the definition of $y_0$ and the assumption that $|\Xi(\iota)|>1$. Either way, we have shown that $\Gamma'(\iota) \in \Xi(\iota)$.

\emph{Case (II)}: If $|\Xi(\iota)|=1$, we proceed to prove that $\Gamma'(\iota)$ cannot halt on \ref{step:alternative_gamma'_return}.
In fact, if by contradiction there exists $t \in \N$ such that $\Gamma'(\iota)$ halts at \ref{step:alternative_gamma'_return}, then by construction of $\Gamma'$, it holds that
\begin{align}\label{eq:contradiction_gamma'_halt_on_step_2^t}
	r(\finitestrings{y}{t},n_0) \leq \frac{1}{2} \text{ for every } y\in \out(t), \text{ and }  \Gamma^{\ran}(\iota, \sigma)\downarrow \text{ for every } \sigma \in \{0,1\}^t
\end{align} 
In particular, $\out(t)=\out$ and thus $\strings{y} =\brackets{\finitestrings{y}{t}}$ for every $y \in \out$.  From hypothesis \eqref{eq:multi_valued_probability_greater_p} we know that $\Xi(\iota) \in \out$, and from \eqref{eq:contradiction_gamma'_halt_on_step_2^t} it follows that $r(\finitestrings{\Xi(\iota)}{t},n_0) \leq \frac{1}{2}$. But then equation \eqref{eq:approximation_n0} together with $\strings{\Xi(\iota)}=\brackets{\finitestrings{\Xi(\iota)}{t}}$ implies
\begin{align*}
	\mu_\rho(\strings{\Xi(\iota)})=\mu_{\rho}(\brackets{\finitestrings{\Xi(\iota)}{t}}) < r(\finitestrings{\Xi(\iota)}{t}),n_0) + \Big(p-\frac{1}{2}\Big) \leq \frac{1}{2}+ p - \frac{1}{2} = p
\end{align*}
and thus $\mu_\rho(\strings{\Xi(\iota)}) < p$, contradicting the fact that hypothesis \eqref{eq:multi_valued_probability_greater_p} shows that $\mu_\rho(\strings{\Xi(\iota)}) \geq p$. Thus, we have proven that if $|\Xi(\iota)|=1$ then $\Gamma'$ cannot halt on \ref{step:alternative_gamma'_return}.
Since we have proven that $\Gamma'(\iota)$ cannot halt on \ref{step:alternative_gamma'_return}, it must halt at \ref{step:return'}, and we have already observed that if that occurs then $\Gamma'(\iota) \in \Xi(\iota)$. This concludes the verification that $\Gamma'(\iota) \in \Xi(\iota)$ for every $\iota \in \Omega$.
\end{proof}

\subsection{The driving propositions for CRP II-III -- Help from developments on Smale's 9th problem}

In this section, we present various results on the non-computability of certain computational problems of interest, under suitable assumptions. Specifically, we extend the driving propositions in \cite{comp}, used to deal with Smale's 9th problem in various computational models, to the Markov model associated to a computational problem: given any algorithm that attempts to solve it, there will always be at least one input on which the algorithm is guaranteed to fail.

\begin{proposition}[Breakdown epsilons in the Markov model]\label{prop:DrivingNegativeProposition}
	Let $\{\Xi, \Omega, \mathcal{M}, \Lambda\}$ be a computational problem with $\Lambda$ finite so that $\Lambda=\{f_i\,\vert\,i\in\mathbb{N}, i\leq k\}$ for some $k \in \mathbb{N}$, and with $\mathcal{M}$ a subset of $\R^d$ for some dimension $d$. Let $\{\iota^1_n\}_{n=1}^{\infty}$, $\{\iota^2_n\}_{n=1}^{\infty}$ be sequences in $\Omega$ and $\iota^0 \in \Omega$. Suppose that the following conditions hold:
	\begin{enumerate}[leftmargin=8mm, label=(\alph*)]
		\item For all $i \in \{1,2,\dotsc,k\}$ and $j \in \{1,2\}$, there exist algorithms $\hat \Gamma^j_i$ such that $\hat \Gamma^{j}_i: \mathbb{N} \times \mathbb{N} \to \Q$ with $|\hat \Gamma^{j}_i(n,r) - f_i(\iota^{j}_r)| \leq 2^{-n-1}$, as well as an algorithm $\hat \Gamma^0_i: \mathbb{N} \to \Q$ with $|\hat \Gamma^{0}_i(n) - f_i(\iota^{0})| \leq 2^{-n-1}$ for all $n \in \N$. \label{assumption:ComputableInput}
		\item We have $|f_i(\iota^j_n)-f_i(\iota^0)|\leq 2^{-n}$ for all $j \in \{1,2\}$, every $n \in\mathbb{N}$, and every $i\in \{1,2,\dotsc,k\}$. \label{assumption:Del1withIota0}
		\item There are sets $S^1, S^2 \subseteq \mathcal{M}$ and $\kappa > 0$ such that
		$
		\inf_{\xi_1 \in S^1, \xi_2 \in S^2}d_{\mathcal{M}}(\xi_1,\xi_2) > 2\kappa
		$ and $\Xi(\iota^j_n) \subseteq S^j$ for $j=1,2$ and $n \in \mathbb{N}$. \label{assumption:S1S2Distant}
		\item Assume that membership in $\ball{\kappa}{S^2}$ is computable in the following sense: for every $y \in \Q^d \subseteq \M$, there is an algorithm that returns true if $y \in \ball{\kappa}{S^2}$ and false if $y \notin \ball{\kappa}{S^2}$.\label{assumption:MembershipComputableS2}
	\end{enumerate}
	Then the corresponding Markov problem $\{\Xi,\O,\M,\Lambda\}^{M}$ has the strong breakdown epsilon satisfying
	$
	\markbdeps \geq \kappa.
	$
More precisely, for any algorithm $\G: \{\stringinputm\}_{m \in \N} \to\M$ that always halts (and in particular, any algorithm $\G: \MarkovOmega \to\M$), there exists $q \in \N$ such that
	$
	d_\mathcal{M}(\Gamma(\stringinputq), \MarkovXi(\stringinputq))> \kappa,
	$
	where $\stringinputm$ is defined as in Lemma \ref{lem:constructing_phi^m} applied to the computational problem $\{\Xi,\O,\L,\M\}$.
\end{proposition}		
\begin{proof}
	Let $\{\stringinputm\}_{m \in \N} \subseteq \MarkovOmega$ be as in Lemma \ref{lem:constructing_phi^m}. This construction relies on assumptions \ref{assumption:ComputableInput} and \ref{assumption:Del1withIota0}. Let $\G: \{\stringinputm\}_{m \in \N} \to\M$ be an algorithm that always halts. Construct the function $\gamma: \N \to \N$:
	\begin{align}\label{eq:gamma}
		\gamma(m)\coloneq 
		\begin{cases}
			1 & \text{ if } \Gamma(\{\phi^m_1,\phi^m_2,\phi^m_3,\dotsc,\phi^m_k\}) \in \ball{\kappa}{S^2}; \\
			2 & \text{ otherwise}.
		\end{cases}
	\end{align}
	There is an algorithm that computes $\gamma$ by the assumption \ref{assumption:MembershipComputableS2} on membership computability and because $\Gamma$ is an algorithm that is assumed to halt on every $\{\phi^m_1,\phi^m_2,\phi^m_3,\dotsc,\phi^m_k\}$. Hence there is a $q \in \mathbb{N}$ such that $\gamma =\varphi_q$ and by Remark \ref{rmk:properties_of_W} there is a unique $t \in \mathbb{N}$ such that $q \in \mathcal{W}(q,t)$. We have the following possibilities, depending on the value of $\gamma(q)$:
	
	\emph{Case (I)}: If $\gamma(q)=1$ then $\varphi_q(q)=\gamma(q)=1$, and thus by \eqref{eq:EFInitialIotaCorrespond} from Lemma \ref{lem:constructing_phi^m} we have that $\stringinputq$ corresponds to $\iota^1_t$. Consequently we see that $\MarkovXi(\stringinputq) = \Xi(\iota^1_t) \subseteq S^1$ where the final inclusion uses \ref{assumption:S1S2Distant}. Moreover, from the definition of $\gamma$ it also holds that $\Gamma(\stringinputq) \in \ball{\kappa}{S^2}$. 

	However, from assumption \ref{assumption:S1S2Distant} we also know that $\ball{\kappa}{S^1} \cap \ball{\kappa}{S^2}=\emptyset$, and thus $\Gamma(\stringinputq)\notin \ball{\kappa}{\MarkovXi(\stringinputq)}$.
	
	\emph{Case (II)}: If $\gamma(q)=2$, then $\varphi_q(q)=\gamma(q)=2$ and thus by \eqref{eq:EFInitialIotaCorrespond} we have that $\stringinputq$ corresponds to $\iota^2_t$. Therefore $\MarkovXi(\stringinputq) = \Xi(\iota^2_t) \subseteq S^2$ by \ref{assumption:S1S2Distant}. Moreover, from the definition of $\gamma$ it also holds that 
	$
	\Gamma(\stringinputq) \notin\ball{\kappa}{S^2}.
	$
 Therefore $\Gamma(\stringinputq)\notin \ball{\kappa}{\MarkovXi(\stringinputq)}$.

	Either way, we have proven that $\Gamma(\stringinputq)\notin \ball{\kappa}{\MarkovXi(\stringinputq)}$. This concludes the proof.
\end{proof}

\subsubsection{Setup and assumptions for the exit-flag problem}\label{sec:assumptions_for_exit_flag}

In this section, we establish the assumptions that will be relevant for the non-computability of the exit-flag problem (Definition \ref{def:ExitFlagProblem}) and the exit-flag problem with oracle (Definition \ref{def:ExitFlagProblemWithOracle}).

Let $\{\Xi,\O,\M,\L\}$ be a computational problem and fix $\kappa,\alpha  \in \Q$ such that $0 \leq \alpha  < \kappa$. We consider an algorithm $\Gamma:\MarkovOmega \to \mathcal{M}$ for the computational problem 
\[
\{\Xi,\Omega,\mathcal{M},\Lambda\}^{M}=\{\MarkovXi,\MarkovOmega,\mathcal{M},\MarkovLambda\}
\] 
such that $\Gamma$ is within the $\alpha$-range of $\MarkovXi$ as per Definition \ref{def:within_the_range} (see Remark \ref{rem:EFAlgorithmAssumption}). Recall that $\Lambda=\{f_i\,\vert\,i\in\mathbb{N}, i\leq k \}$ is of finite size $k = |\Lambda|$ and $\M \subseteq \R^d$ for some dimension $d \in \N$. 

We then consider the following assumptions about some $\iota^0 \in \Omega$, sequences $\{\iota^1_n\}_{n\in\N}, \{\iota^2_n\}_{n \in \N} \subseteq \Omega$ and for $j=0,1,2$, sets $S^j \subseteq \mathcal{M}$:
	\begin{enumerate}[label=(A\roman*),series=EFBaseAssumptionsX, leftmargin = 8mm]
	\item For all $i \in \{1,2,\dotsc,k\}$, there exist algorithms $\hat \Gamma^1_i, \hat \Gamma^2_i$ such that $\hat \Gamma^{j}_i: \mathbb{N} \times \mathbb{N} \to \Q$ with $|\hat \Gamma^{j}_i(n,r) - f_i(\iota^{j}_r)| \leq 2^{-n-1}$ for $j\in\{1,2\}$, as well as an algorithm $\hat \Gamma^0_i: \mathbb{N} \to \Q$ with $|\hat \Gamma^{0}_i(n)
	- f_i(\iota^{0})| \leq 2^{-n-1}$, for all $n \in \N$. \label{assumption:EFComputableInput}
	\item $|f_i(\iota^j_n)-f_i(\iota^0)|\leq 2^{-n}$ for all $j \in \{1,2\}$, every $n \in\mathbb{N}$, and every $i\in \{1,2,\dotsc,k\}$.  \label{assumption:EFDel1Info}		
	\item 
	$
	\inf_{\xi_1\in S^1,\xi_2\in S^2} d_{\mathcal{M}}(\xi_1,\xi_2) > 2 \kappa
	$.  \label{assumption:EFS1S2Distant}
	\item $\Xi(\iota^j_n) \subseteq S^j$ for all $n \in \mathbb{N}$ and $j \in \{1,2\}$ and $\Xi(\iota^0) = S^0$.\label{assumption:EFXiIotajInSj}
	\item $\Xi(\Omega) \subseteq \ball{\kappa-\alpha}{S^0} \cup \ball{\kappa - \alpha}{S^1} \cup \ball{\kappa - \alpha}{S^2}$. \label{assumption:EFXiOmegaInBallAroundXiIota0S1AndS2}
	\item For $S \in \{S^2, S^1 \setminus S^0, S^2 \setminus S^0\}$, membership in $\ball{\kappa}{S}$ is computable in the following sense: for every $x \in \Q^d$, there is an algorithm that returns true if $x \in \ball{\kappa}{S}$ and false if $x \notin \ball{\kappa}{S}$. \label{assumption:EFMembershipComputableS1S2}
	\item For $j\in\{1,2\}$, the sets $\Xi(\iota^j_n)$ satisfy $\Xi(\iota^j_n) \setminus S^0 = S^j \setminus S^0$ for all $n\in\N$. \label{assumption:EFConstantIotaj}
	\end{enumerate}
	
	For the exit-flag problem relative to $\Gamma$ with oracle as in Definition \ref{def:ExitFlagProblemWithOracle}, we will make the following extra assumption. Let $\omega \in \Q$ be such that $\omega \in [\alpha, \kappa)$, and assume that
	\begin{enumerate}[label=(A\roman*),resume, leftmargin = 8mm]
	\item  For $j \in \{1,2\}$, there exists $y^j$ so that $y^j \in \ball{\omega}{\Xi(\iota^j_n))} \cap  \ball{\omega}{\Xi(\iota^0)}  \cap \Q^d$ for all $n \in \N$. \label{assumption:EFexists_xj_in_intersection} 
\end{enumerate}

\subsubsection{Techniques for the exit-flag problem}\label{sec:techniques_for_exit_flag}

In this section we prove two results on the non-computability of the exit-flag problems, with and without oracle.

\begin{proposition}[Non computability of the exit-flag]\label{prop:EF}
	 Consider the setup of \S \ref{sec:assumptions_for_exit_flag} and suppose that assumptions \ref{assumption:EFComputableInput}-\ref{assumption:EFConstantIotaj} hold. Then the exit-flag problem relative to $\Gamma$, given by $\{\MarkovExitXi,\MarkovOmega,\{0,1\},\MarkovLambda\}$ as specified in Definition \ref{def:ExitFlagProblem}, 
	 has strong breakdown epsilon satisfying $\markbdeps \geq \frac{1}{2}.$ More precisely, for any algorithm $\Gamma^E:\{\stringinputm\}_{m \in \N} \to \{0,1\}$ that always halts (and in particular, any algorithm $\G^E: \MarkovOmega \to\{0,1\}$) there exists $q \in \N$ such that $\Gamma^E(\stringinputq) \neq \MarkovExitXi(\stringinputq)$, where $\stringinputm$ is defined as in Lemma \ref{lem:constructing_phi^m} applied to the computational problem $\{\Xi,\O,\L,\M\}$.
\end{proposition} 

\begin{proof}
Let $\{\stringinputm\}_{m \in \N} \subseteq \MarkovOmega$ be as in Lemma \ref{lem:constructing_phi^m}. Note that this construction relies on assumptions \ref{assumption:EFComputableInput} and \ref{assumption:EFDel1Info}.
	 Assume that $\Gamma^E:\{\stringinputm\}_{m \in \N} \to \{0,1\}$ is an algorithm that always halts.
	We define the function $\gamma:\N \to \N$ in the following way:
	\begin{equation}\label{eq:gammaEF}
		\gamma(m) :=
		\begin{cases}
			1 & \text{if } \big[ [\Gamma^E(\stringinputm)=1 \land \Gamma(\stringinputm) \in \ball{\kappa}{S^2}] \\
			& \quad \  \lor \ [\Gamma^E(\stringinputm)=0 \land \Gamma(\stringinputm) \in \ball{\kappa}{S^1} \setminus \ball{\kappa}{S^0}]\big];\\
			2 & \text{if } \big[ [\Gamma^E(\stringinputm)=1 \land \Gamma(\stringinputm) \notin \ball{\kappa}{S^2}] \\
			& \quad \  \lor [\ \Gamma^E(\stringinputm)=0 \land \Gamma(\stringinputm) \in \ball{\kappa}{S^2} \setminus \ball{\kappa}{S^0}]\big] ;\\
			3 & \text{ otherwise.}
		\end{cases}
	\end{equation} 
	Because $\Gamma^E$ and $\Gamma$ are computable and always halt and because of assumption \ref{assumption:EFMembershipComputableS1S2}, the function $\gamma$ is computable and always halts. Therefore we can fix $q$ such that $\gamma = \varphi_q$. Additionally, since $\gamma$ is computable, by Remark \ref{rmk:properties_of_W} there is a unique $t \in \mathbb{N}$ such that $q \in \mathcal{W}(q,t)$. Then we have the following possibilities, depending on the value of $\gamma(q)$:	

	 \emph{Case (I)}: If $\gamma(q)=1$, then $\varphi_q(q)=\gamma(q)=1$ and thus by the definition of $\stringinputq$, $\stringinputq$ corresponds to $\iota^1_t$ by \eqref{eq:EFInitialIotaCorrespond}. Therefore $\MarkovXi(\stringinputq)\subseteq S^1$ by assumption \ref{assumption:EFXiIotajInSj}. There are two sub-cases that arise from considering the definition of $\gamma$:
	\begin{enumerate}[label = (\roman*), leftmargin=6mm]
		\item If $\Gamma^E(\stringinputq)=1$, then by \eqref{eq:gammaEF} we must have $\Gamma(\stringinputq)\in \ball{\kappa}{S^2}$. But since $\ball{\kappa}{S^2} \cap \ball{\kappa}{S^1}=\emptyset$ by assumption \ref{assumption:EFS1S2Distant}, we have $\Gamma(\stringinputq) \notin \ball{\kappa}{\MarkovXi(\stringinputq)}$ and thus $\Gamma^{E}(\stringinputq) \neq \MarkovExitXi(\stringinputq)$.
		\item If $\Gamma^E(\stringinputq)=0$, then by \eqref{eq:gammaEF} we must have $\Gamma(\stringinputq) \in \ball{\kappa}{S^1} \setminus \ball{\kappa}{S^0}$. But then using assumption \ref{assumption:EFConstantIotaj}, we see that $\Gamma(\stringinputq) \in \ball{\kappa}{\MarkovXi(\stringinputq)}$, and thus $\Gamma^{E}(\stringinputq) \neq \MarkovExitXi(\stringinputq)$. 
	\end{enumerate}
	
	 \emph{Case (II)}: If $\gamma(q)=2$, then $\varphi_q(q)=\gamma(q)=2$ and thus $\stringinputq$ corresponds to $\iota^2_t$ by \eqref{eq:EFInitialIotaCorrespond}. Thus $\MarkovXi(\stringinputq)\subseteq S^2$ by assumption \ref{assumption:EFXiIotajInSj}. There are two sub-cases:
	\begin{enumerate}[label = (\roman*), leftmargin=6mm]
		\item If $\Gamma^E(\stringinputq)=1$, then by \eqref{eq:gammaEF} we must have $\Gamma(\stringinputq)\notin \ball{\kappa}{S^2}$. But then $\Gamma(\stringinputq) \notin \ball{\kappa}{\MarkovXi(\stringinputq)}$ and thus $\Gamma^{E}(\stringinputq) \neq \MarkovExitXi(\stringinputq)$.
		\item If $\Gamma^E(\stringinputq)=0$, then by \eqref{eq:gammaEF} we must have $\Gamma(\stringinputq) \in \ball{\kappa}{S^2} \setminus \ball{\kappa}{S^0}$. But then using assumption \ref{assumption:EFConstantIotaj}, we see that $\Gamma(\stringinputq) \in \ball{\kappa}{\MarkovXi(\stringinputq)}$, and thus $\Gamma^{E}(\stringinputq) \neq \MarkovExitXi(\stringinputq)$. 
	\end{enumerate}
	
	 \emph{Case (III)}: 	If $\gamma(q)=3$, then $\varphi_q(q)=\gamma(q)=3$ and thus $\stringinputq$ corresponds to $\iota^0$ by \eqref{eq:EFInitialIotaCorrespond}. From the definition of $\gamma$ in \eqref{eq:gammaEF} it is obvious that $\Gamma^E(\stringinputq) = 0$ and we further claim that $\Gamma(\stringinputq) \in \ball{\kappa}{S^0}$. To see this, note that from the assumptions that $\Gamma$ is within the $\alpha$-range of $\MarkovXi$, that $0\leq \alpha < \kappa$ and from assumption \ref{assumption:EFXiOmegaInBallAroundXiIota0S1AndS2} we have
			\begin{equation*}
			\begin{split}
			\Gamma(\stringinputq) \in \ball{\alpha}{\Xi(\Omega)} \subseteq \ball{\alpha}{\ball{\kappa-\alpha}{S^0 \cup S^1 \cup S^2}}  \subseteq\ball{\kappa}{S^0\cup S^1 \cup S^2}.
		\end{split}
		\end{equation*}

		Simultaneously, using \eqref{eq:gammaEF} we also see that $\Gamma(\stringinputq) \notin \ball{\kappa}{S^1 \cup S^2}\setminus\ball{\kappa}{S^0}$ and therefore $\Gamma(\stringinputq) \in \ball{\kappa}{S^0}$.
		 Thus using assumption \ref{assumption:EFXiIotajInSj} we see that  $\Gamma(\stringinputq) \in \ball{\kappa}{S^0} = \ball{\kappa}{\Xi(\iota^0)} = \ball{\kappa}{\MarkovXi(\stringinputq)}$. We conclude that 
	$
	\Gamma^{E}(\stringinputq) \neq \MarkovExitXi(\stringinputq).
	$
	
	In all cases we have proven that $
		\Gamma^{E}(\stringinputq) \neq \MarkovExitXi(\stringinputq)
		$, proving Proposition \ref{prop:EF}.
\end{proof}

\begin{proposition}[Non computability of the exit-flag with an oracle]\label{prop:EF_oracle}
	Consider the setup of \S \ref{sec:assumptions_for_exit_flag} and suppose that assumptions \ref{assumption:EFComputableInput}-\ref{assumption:EFConstantIotaj} hold.
	Then for the exit-flag problem with oracle associated to $\Gamma$, given by
	\begin{align*}
		\{\MarkovExitXi,\MarkovOmega,\{0,1\},\MarkovLambda\}^{\mathcal{O}} = \{\MarkovExitXiOracle,\domainOracleXi,\{0,1\},\MarkovLambdaOracle\},
	\end{align*}
	as specified in Definition \ref{def:ExitFlagProblemWithOracle}, the following holds. For any algorithm $\GammaEO: \domainOracleXi \to \{0,1\}$  that always halts, there exist $m \in \N$ and $y \in \ballQ{\MarkovXi(\MarkovOmega)}$ such that $\GammaEO(\stringinputm,y) \notin \MarkovExitXiOracle(\stringinputm,y)$ where $\stringinputm$ is defined as in Lemma \ref{lem:constructing_phi^m} applied to the computational problem $\{\Xi,\O,\L,\M\}$. 

\end{proposition}
\begin{proof}
Let $\{\stringinputm\}_{m \in \N} \subseteq \MarkovOmega$ be as in Lemma \ref{lem:constructing_phi^m}. Note that this construction relies on assumptions \ref{assumption:EFComputableInput} and \ref{assumption:EFDel1Info}.

By contradiction, let $\GammaEO: \domainOracleXi \to \{0,1\}$ be an algorithm that always halts, and assume that $\GammaEO(\stringinputm,y) \in \MarkovExitXiOracle(\stringinputm,y)$ for every $m \in \N$ and $y \in \ballQ{\MarkovXi(\MarkovOmega)}$.
	
We now construct a recursive algorithm $\Gamma^E:\{\stringinputm \, \vert \, m \in \N\}\subseteq \MarkovOmega \to \{0,1\}$ (without oracle) that always halts and that can solve the exit-flag problem given by $\{\MarkovExitXi, \{\stringinputm\}_{m \in \N}, \{0,1\},\MarkovLambda\}$. This, however, will contradict Proposition \ref{prop:EF}.
	Define the following algorithm $\Gamma^E:\{\stringinputm\, \vert \, m \in \N \} \subseteq \MarkovOmega \to \{0,1\}$:
	{\small
		\begin{equation}\label{eq:Gamma_E}
			\begin{split}
				\Gamma^E(\stringinputm) \coloneq
				\begin{cases}
					\GammaEO(\stringinputm,y^1) & \text{ if } \GammaEO(\stringinputm,y^1) = \GammaEO(\stringinputm,y^2);\\
					\GammaEO(\stringinputm,y^1) & \text{ if } \GammaEO(\stringinputm,y^1) \neq \GammaEO(\stringinputm,y^2) \land \varphi_m(m) = 1; \\
					\GammaEO(\stringinputm,y^2) & \text{ if } \GammaEO(\stringinputm,y^1) \neq \GammaEO(\stringinputm,y^2) \land \varphi_m(m) = 2.
				\end{cases}
			\end{split}	
		\end{equation}
	}	
	Note that the algorithm is well defined. In particular, note that $\varphi_m(m)$ is guaranteed to halt when 
	\[
	\GammaEO(\stringinputm,y^1) \neq \GammaEO(\stringinputm,y^2),
	\]
	and that the value of $\varphi_m(m)$ is necessarily either $1$ or $2$. To see this, fix $m \in \N$ such that $\GammaEO(\stringinputm,y^1) \neq \GammaEO(\stringinputm,y^2)$.
	If $\stringinputm$ corresponds to $\iota^0$, then by assumption \ref{assumption:EFexists_xj_in_intersection} we have $y^j \in \ball{\omega}{\Xi(\iota^0)}$ for $j \in \{1,2\}$. Thus by the definition in \eqref{eq:ExitXiOracleDef} we see that $\MarkovExitXiOracle(\stringinputm,y^1) =\MarkovExitXiOracle(\stringinputm,y^2)
	\in \{0,1\}
	$ (and in particular, both $\MarkovExitXiOracle(\stringinputm,y^1)$ and $\MarkovExitXiOracle(\stringinputm,y^2)$ are single-valued). Recalling the assumption that $\GammaEO(\stringinputm,y) \in \MarkovExitXiOracle(\stringinputm,y)$ for every $(\stringinputm, y)$ gives a contradiction as follows:
	\begin{align*}
		\GammaEO(\stringinputm,y^1) = \MarkovExitXiOracle(\stringinputm,y^1) = \MarkovExitXiOracle(\stringinputm,y^2) = \GammaEO(\stringinputm,y^2),
	\end{align*}
	contradicting the assumption that $\GammaEO(\stringinputm,y^1) \neq \GammaEO(\stringinputm,y^2)$.
	Therefore, $\stringinputm$ does not correspond to $\iota^0$. Consequently, from \eqref{eq:EFInitialIotaCorrespond} we see that $\stringinputm$ corresponds either to $\iota^1_n$ or to $\iota^2_n$ for some $n$. In either case, it must hold that $\varphi_m(m)$ halts with either $\varphi_m(m)=1$ or $\varphi_m(m) =2$. Thus $\G^E(\stringinputm)$ is well defined for every $m \in \N$.
	
	We now claim that the recursive algorithm $\Gamma^E: \{\stringinputm\}_{m \in \N} \subseteq \MarkovOmega \to \{0,1\}$ solves the exit-flag problem $\{\MarkovExitXi, \{\stringinputm\}_{m \in \N}, \{0,1\},\MarkovLambda\}$.
	Indeed, for every $m \in \N$, we have the following possibilities: 
	
	 \emph{Case (1)}: If $\GammaEO(\stringinputm, y^1) \neq \GammaEO(\stringinputm, y^2)$, then, by the reasoning above, either $\varphi_m(m) =1$ or $\varphi_m(m)=2$. Suppose $\varphi_m(m) = j$ with $j \in \{1,2\}$. Then $\stringinputm$ corresponds to $\iota^j_n$ for some $n$ by \eqref{eq:EFInitialIotaCorrespond}. Thus, by assumption \ref{assumption:EFexists_xj_in_intersection}, we have  $y^j \in \ball{\omega}{\Xi(\iota^j_n)}=\ball{\omega}{\MarkovXi(\stringinputm)}$. In particular, by definitions \eqref{eq:ExitXiOracleDef} and \eqref{eq:Gamma_E} we see that 
	$
	\MarkovExitXi(\stringinputm) = \MarkovExitXiOracle(\stringinputm,y^j) = \GammaEO(\stringinputm,y^j)=\Gamma^E(\stringinputm)
	$
	so that $\Gamma^E(\stringinputm)=\MarkovExitXi(\stringinputm)$ as desired.
	
	 \emph{Case (2)}: If instead $\GammaEO(\stringinputm, y^1) = \GammaEO(\stringinputm, y^2)$, then there exists $j \in \{1,2\}$ such that $y^j \in \ball{\omega}{\MarkovXi(\stringinputm)}$. In fact, by \eqref{eq:EFInitialIotaCorrespond} either $\stringinputm$ corresponds to $\iota^j_n$ for some $n\in \N$ and $j \in \{1,2\}$, in which case $y^j \in  \ball{\omega}{\Xi(\iota^j_n)} = \ball{\omega}{\MarkovXi(\stringinputm})$ by assumption \ref{assumption:EFexists_xj_in_intersection}; or $\stringinputm$ corresponds to $\iota^0$, in which case $y^j \in  \ball{\omega}{\Xi(\iota^0)} = \ball{\omega}{\MarkovXi(\stringinputm)}$ for every $j \in \{1,2\}$, again by assumption \ref{assumption:EFexists_xj_in_intersection}. Thus, from definition \eqref{eq:ExitXiOracleDef} we see that for the $j$ above
	\[
	\MarkovExitXi(\stringinputm) = \MarkovExitXiOracle(\stringinputm,y^j) = \GammaEO(\stringinputm, y^1) = \Gamma^E(\stringinputm).
	\]
	Thus $\Gamma^E(\stringinputm)=\MarkovExitXi(\stringinputm)$ as desired.
	
	Therefore, the recursive algorithm $\Gamma^E: \{\stringinputm\}_{m \in \N} \subseteq \MarkovOmega \to \{0,1\}$ solves the exit-flag problem $\{\MarkovExitXi, \{\stringinputm\}_{m \in \N}, \{0,1\},\MarkovLambda\}$. As noted earlier, this contradicts
Proposition \ref{prop:EF}, whose assumptions \ref{assumption:EFComputableInput}-\ref{assumption:EFConstantIotaj} are all assumed to hold. This completes the proof of Proposition \ref{prop:EF_oracle}.
\end{proof}

\subsection{Minimisers of Convex Optimisation Problems} \label{sec:geometry}

In this section we explicitly construct the input set $\Omega_{\no,\nt}$ discussed in \S \ref{section:setup_CRP} used in Theorems \ref{thm:crp_1_2}, \ref{thm:crp_3_4} and \ref{thm:crp_5}. This set will depend on which computational problem we are considering. For this reason, we will also discuss various results related to the solutions of such optimisation problems. These results are taken from \cite{comp} - we repeat these simple proofs here for the sake of completeness.

\subsubsection{The sets $\Omega_{\no,\nt}$ and their related computational problems}\label{sec:geometry_of_inputs}

Let $\no \geq 2$, $\nt \geq 1$. Fix $\kappa =10^{-1}$ and define for $u_1,u_2 \in \Q_{>0}$:
\begin{equation}\label{eq:y_and_A_LP}
\begin{split}
		U(u_1,u_2):=U(u_1,u_2,\no,\nt) &= \begin{pmatrix} u_1 & u_2 \end{pmatrix} \oplus \begin{pmatrix} I_{\nt-1} & 0_{\nt-1 \times \no-\nt-2} \end{pmatrix},\\
		b \coloneq b(\nt) &= 2 \cdot \kappa \cdot e_1,
\end{split}
\end{equation}
where $\{e_1,\dotsc, e_\nt\}$ denotes the canonical basis of $\R^\nt$.
For $\theta \in [1/8,1/4] \cap \Q$ let
\begin{equation*}\label{eq:L}
	\mathcal{L}_\theta:=\left\{ (u_1,u_2) \in [\theta,1/2]^2 \, \vert \, \exists \text{ at most one } i \text{ with } u_i \neq 1/2\right\} \cap \Q^2.
\end{equation*}	
Define the sets of inputs
\begin{align}\label{eq:set_inputs}
	\Omega_{\no,\nt} = \Omega_{\no,\nt}(\theta)& \coloneq \{(b(\nt), U(u_1,u_2,\no,\nt) \ | \ (u_1,u_2) \in \mathcal{L}_\theta)\}.
\end{align}

\begin{remark}
	By varying $\theta \in [1/8,1/4] \cap \Q$ in \eqref{eq:set_inputs}, we obtain infinitely many collections of inputs $\Omega_{\no,\nt}(\theta)$ for which the results of the CRP Theorem hold.
\end{remark}

Let $\Lambda = \{g_i\}_{i = 1}^\nt \cup \{h_{i,j}\}_{i = 1, j = 1}^{i = \nt, j= \no}$ be given by the entry-wise component functions $g_{i}(y,A) \coloneq y_i$ and $h_{i,j}(y,A)= A_{i,j}$ for every $i,j$ and $(y, A) \in \Q^\nt \times \Q^{\nt \times \no}$. We denote $k \coloneq |\Lambda| = \nt +\nt\no$ and rename and re-enumerate the functions so that $\Lambda = \{f_i\}_{i = 1,\dotsc, k}$ with $f_1 \coloneq h_{1,1}$ and $f_2 \coloneq h_{1,2}$.

\subsubsection{Linear Programming}

 Let $c=\ones_{\no}$ be the $\no$-dimensional vector of ones. Given $A \in \mathbb{R}^{\nt \times \no}$, $y \in \mathbb{R}^\nt$, consider the Linear Programming (LP) mapping $\xilp \colon \R^{\nt} \times \R^{\nt \times \no} \rightrightarrows \R^\no$ given by
\begin{align*}\label{eq:xilp}
	\xilp(y,A) \coloneq \argmin \{\langle x , c \rangle \ | \ x \in \R^\no, \ Ax = y, \ x \geq 0\}
\end{align*}
We now state a simple lemma that relates these inputs to the corresponding solutions of the LP problem. Its proof is taken from \cite{comp}. 
\begin{lemma}[Linear Programming] \label{lemma:ProblemBasicExampleLP}
	Let $c=\ones_{\no}$ be the $\no$-dimensional vector of ones. Then the solution  $\xilp$ to the linear programming problem satisfies
	\begin{equation*}\label{eq:lpsolnsbasic}
		\xilp(b,U(u_1,u_2)) =\begin{cases}
			\left\{\frac{2\kappa}{u_1}e_1\right\} & \text{ if } u_1 > u_2 \\
			\left\{\frac{2\kappa}{u_2}e_2\right\}  &\text{ if } u_2 > u_1\\
			\left\{\frac{2\kappa}{u_1}(te_1 + (1-t)e_2) \, \vert \, t \in [0,1]\right\} & \text{ if } u_1 = u_2
		\end{cases}.
	\end{equation*}
\end{lemma}

\begin{proof}[Proof of Lemma \ref{lemma:ProblemBasicExampleLP}]
	Denote $U \coloneq U(u_1,u_2)$. For any feasible $x$ (that is, $x$ with $x \geq 0$ and $Ux=b$), and recalling that $u_1$ and $u_2$ are assumed to be positive, we have 
	\begin{equation}\label{eq:InequalitiesForLPBasic1}
		\langle c,x \rangle \geq x_1 + x_2  \geq  \frac{u_1 x_1 + u_2 x_2}{u_1 \vee u_2} = \frac{2\kappa}{u_1 \vee u_2}
	\end{equation}
	which implies that $\min \{\langle c,x \rangle  \, \vert \, x\geq 0, Ux = b \} \geq 2\kappa/(u_1 \vee u_2)\,$.
	Furthermore, all claimed minimisers $x$ in the statement of the lemma are feasible for the LP problem and satisfy \eqref{eq:InequalitiesForLPBasic1} as an equality. We can thus deduce that they are indeed minimisers and that
	$\min \{\langle c,x \rangle  \, \vert \, x\geq 0, U x = b \} = 2\kappa/(u_1 \vee u_2),$ and that the solutions to $\Xi_{\LP}(b,U)$ are precisely those vectors $x$ for which every inequality in \eqref{eq:InequalitiesForLPBasic1} is obeyed as an equality. More precisely, the following conditions must hold:
		\begin{enumerate}[leftmargin = 6mm]
		\item If the first inequality is an equality then $x_3 = x_4 = \dotsb = x_\no = 0$.
		\item If the second inequality is an equality then $x_2 = 0$ in the case $u_1 > u_2$ and $x_1 = 0$ in the case $u_1 < u_2$. In the case $u_1 = u_2$ this is always an equality.
	\end{enumerate}
	It is straightforward to check that the $x$ which satisfy all these conditions as well as the feasibility conditions $x \geq 0$ and $U x = b$ are precisely the claimed minimisers in the statement of the lemma.
\end{proof}

\subsubsection{Basis Pursuit}

Let $\kappa \in \Q_{\geq 0}$ and $\eta \in \Q$ be such that $0 \leq \eta \leq 2\kappa$. The Basis Pursuit solution mapping is
\begin{align*}
	\Xi_{\BP}(y,A) & \coloneq \argmin_{x \in \mathbb{R}^{\no}}  \|x\|_1, \ \text{ such that }\|Ax-y\|_{2}\leq \eta.
\end{align*}
As with Lemma \ref{lemma:ProblemBasicExampleLP}, this easy lemma about BP solutions and its proof are taken from \cite{comp}.
\begin{lemma}[Basis Pursuit]\label{lemma:ProblemBasicExampleBPDNL1}
	Assuming that $2\kappa \geq \eta$, we have
	\begin{equation*}\label{eq:bpdnsolnsbasic}
		\Xi_{\BP}(b,U(u_1,u_2)) = 
		\begin{cases} 
		\left\{\frac{2\kappa - \eta}{u_1} e_1\right\} & \text{if } u_1 > u_2\\
			\left\{\frac{2\kappa - \eta}{u_2} e_2\right\} & \text{if } u_1 < u_2 \\
			\left\{\frac{2\kappa - \eta}{u_1} \left( t  e_1 + (1-t) e_2 \right)\, \vert \, t \in [0,1]\right\} & \text{if } u_1 = u_2
		\end{cases}.
	\end{equation*}
\end{lemma}

\begin{proof}[Proof of Lemma \ref{lemma:ProblemBasicExampleBPDNL1}] 

		Denote $U \coloneq U(u_1,u_2)$. From the definition of $U$ and $b$, condition $\|U x-b \|_2 \leq \eta$ becomes $|u_1 x_1 + u_2 x_2 - 2\kappa| \leq \eta$. Thus we have the following chain of inequalities:
		\begin{equation}\label{eq:InequalitiesForBPDNBasic}
			(u_1 \vee u_2)\|x\|_1 \geq (u_1 \vee u_2) (|x_1| + |x_2|) \geq u_1|x_1| + u_2 |x_2| \geq u_1 x_1 + u_2 x_2 \geq 2\kappa - \eta,
		\end{equation}
		which implies that $\min \{\|x\|_1 \, \vert \, \|U x-b\|_2 \leq \eta \} \geq ( 2\kappa - \eta)/(u_1 \vee u_2)$. 
		On the other hand, all claimed minimisers $x$ defined in the statement of the lemma are feasible for the BP problem and satisfy \eqref{eq:InequalitiesForBPDNBasic} as an equality. Therefore we can conclude that they are indeed minimisers and that
		$\min \{\|x\|_1 \, \vert \, \|U x- b\|_2 \leq \eta \} = ( 2\kappa - \eta)/(u_1 \vee u_2)$.
		From this, we deduce that the solutions to $\xibp(b,U)$ are exactly the vectors $x$ for which every inequality in \eqref{eq:InequalitiesForBPDNBasic} is obeyed as an equality. More specifically, all the following conditions must hold:
		\begin{enumerate}[leftmargin = 6mm]
			\item If the first inequality is an equality then $x_3 = x_4 = \dotsb = x_\no = 0$.
			\item If the second inequality is an equality then $x_2 = 0$ in the case $u_1 > u_2$ and $x_1 = 0$ in the case $u_1 < u_2$. In the case $u_1 = u_2$, this is always an equality.
			\item If the third inequality is an equality then  $x_1$ and $x_2$ are non-negative.
			\item If the final inequality is an equality then $u_1 x_1 + u_2 x_2 = 2\kappa - \eta$.
		\end{enumerate}
		It is a straightforward task to check that the $x$ which fulfil these conditions are precisely the claimed minimisers in the statement of the lemma.
\end{proof}

\subsubsection{LASSO}

Let $\kappa \in \Q_{\geq 0}$ and $\lambda \in \Q$ be such that $0 < \lambda \leq 2\kappa$. The LASSO solution mapping is
\begin{align*}
	\Xi_{\LASSO}(y,A) & \coloneqq \argmin_{x \in \mathbb{R}^{\no}}  \lambda\|x\|_1+\|Ax-y\|_2^2.
\end{align*}
As with Lemmas \ref{lemma:ProblemBasicExampleLP} and \ref{lemma:ProblemBasicExampleBPDNL1}, the following lemma regarding the solutions to the LASSO problem, as well as its proof, are taken from \cite{comp}. 
\begin{lemma}[LASSO]\label{lemma:ProblemBasicExampleULASSO}
	Assuming that $\max\{u_1,u_2\} \geq \lambda/(4\kappa)$, the solution $\Xi_{\LASSO}$ to the LASSO problem  satisfies
	\begin{equation*}\label{eq:unequalAlphaBetaLASSO}
		\Xi_{\LASSO}(b,U(u_1,u_2))  =\begin{cases}
			\left\{\frac{4u_1 \kappa-\lambda}{2u_1^2}e_1\right\} & \text{ if } u_1 > u_2 \\
			\left\{\frac{4a_2 \kappa-\lambda}{2u_2^2}e_2\right\}  &\text{ if } u_1 < u_2\\
			\left\{\frac{4u_1 \kappa-\lambda}{2u_1^2}(te_1 + (1-t)e_2) \, \vert \, t \in [0,1]\right\} & \text{ if } u_1 =u_2
		\end{cases}.
	\end{equation*}
\end{lemma}

\begin{proof}[Proof of Lemma \ref{lemma:ProblemBasicExampleULASSO}]
		Denote $U \coloneq U(u_1,u_2)$. Consider the following candidate solutions, as in the statement of the lemma, given by $x^{opt}=\frac{4u_1 \kappa -\lambda}{2u_1^2} e_1$ if $u_1\geq u_2$, and $x^{opt}=\frac{4u_2 \kappa-\lambda}{2u_2^2} e_2$ if $u_2>u_1$. We define the dual vector $p\coloneq U x^{opt}-b=-\frac{\lambda}{2(u_1\vee u_2)}e_1\in \R^m$, which satisfies
		\begin{equation*}
			-\frac{2}{\lambda}U^* p= \left(1\wedge \frac{u_1}{u_2}\right)e_1 + \left(1\wedge \frac{u_2}{u_1}\right)e_2\in \partial \|\cdot\|_{1}(x^{opt})
		\end{equation*}
		where $\partial$ denotes the subdifferential of a function. Therefore, for every $x\in\R^N=\no$ we have
		\begin{align}
			\frac{1}{2}\|U x -b\|_{2}^2 + \frac{\lambda}{2}\|x\|_{1}&\geq \langle U x -b,p\rangle  - \frac{1}{2}\|p\|_{2}^2 + \frac{\lambda}{2} \|x\|_{1}\notag\\
			&=   \langle U x^{opt} - b,p \rangle  - \frac{1}{2}\|p\|_{2}^2+   \frac{\lambda}{2}\left(\|x\|_{1} - \langle  x -x^{opt} , -\frac{2}{\lambda}U^*p\rangle   \right)\notag\\
			&\geq \frac{1}{2}\|U x^{opt} -b\|_{2}^2+  \frac{\lambda}{2}\|x^{opt}\|_1\notag.
		\end{align} 
		where the first inequality follows by expanding $\|(Ux-b)-p\|^2_2$, and the last inequality follows by expanding the square $\|(Ux^{opt}-b)-p\|^2$ and the fact that $-\frac{2}{\lambda}U^* p \in \partial \|\cdot\|_{1}(x^{opt})$. 
		We deduce that $x^{opt}$ is a minimiser, and that any vector $x$ is a minimiser if and only if this chain of inequalities holds with equality. This is the case if and only if $U x - b = p =U x^{opt}-b$ and $\|x\|_{1} - \langle  x -x^{opt},-\frac{2}{\lambda}U^*p\rangle = \|x^{opt}\|_{1}  $. From the definition of $U$ and $b$, this is equivalent to the conditions $u_1 x_1 + u_2 x_2 = u_1 x^{opt}_1 + u_2 x^{opt}_2 = 2\kappa-\frac{\lambda}{2(u_1\vee u_2 )}$,  $x_j = x_j^{opt}$ for $j =3,\dotsc,N$, and
		\begin{equation*}
			|x_1| - (x_1-x_1^{opt})\left(1\wedge \frac{u_1}{u_2}\right)=|x^{opt}_1|,\quad \text{and}\quad |x_2| - (x_2-x_2^{opt})\left(1\wedge \frac{u_2}{u_1}\right)=|x^{opt}_2|.
		\end{equation*}
		It is a straightforward to task to ensure that the vectors $x$ that obey all these conditions are precisely those as in the statement of the lemma.
\end{proof}

\subsubsection{The collection of problems}\label{sec:collection_problems} The proof of the CRP results will rely on \S \ref{section:proof_4} (namely, Lemma \ref{lem:constructing_phi^m}, Propositions \ref{prop:de-oracolisation}, \ref{prop:DrivingNegativeProposition}, \ref{prop:EF}, \ref{prop:EF_oracle}, and assumptions \ref{assumption:EFComputableInput}-\ref{assumption:EFexists_xj_in_intersection} in \S  \ref{sec:assumptions_for_exit_flag}). As a consequence, we will now construct some of the objects mentioned in the assumptions of these propositions.  In particular, we proceed to construct the sequences of inputs $\{\iota^1_n\}_{n \in \N}\subseteq \Omega_{\no,\nt}$ and $\{\iota^2_n\}_{n \in \N} \subseteq \Omega_{\no,\nt}$, we construct the input $\iota^0 \in \Omega_{\no, \nt}$, we explicitly describe the sets $S^0$, $S^1$ and $S^2$ and the vectors $y^1$ and $y^2$ that appear in the aforementioned assumptions, and we construct the algorithms $\hat\Gamma^j_i$ (for $j \in \{0,1,2\}$ and $i = 1,\dotsc, k$) that return approximations to the coordinates of $\{\iota^1_n\}_{n \in \N}$, $\{\iota^2_n\}_{n \in \N}$ and $\iota^0$.

Recall the discussion and setup in \S \ref{section:setup_CRP}. For $\no \geq 2$, $\nt \geq 1$, let $k = \nt+\nt\no$. For any $\theta \in [1/8,1/4] \cap \Q$, consider the collection of inputs $\Omega_{\no,\nt}=\Omega_{\no,\nt}(\theta)$ as in \eqref{eq:set_inputs}. 
For $n \geq 1$, we define
\begin{align}\label{eq:iota}
	\iota^1_n & \coloneq \left(b, U\left(\frac{1}{2}, \frac{1}{2} - \frac{1}{4^n}\right)\right); \,
	\iota^2_n \coloneq \left(b, U\left(\frac{1}{2}- \frac{1}{4^n}, \frac{1}{2} \right)\right); \, 
	\iota^0  \coloneq \left(b, U\left(\frac{1}{2}, \frac{1}{2} \right)\right).
\end{align}

Fix $\kappa = 10^{-1}$ and let $\alpha, \omega \in \Q$ be such that $0 \leq \alpha \leq \omega < \kappa$.
Then, depending on the computational problem under consideration, we make the following definitions.

 \emph{Case (LP):} For the computational problem $\{\Xi_{\LP}, \Omega_{\no,\nt}(\theta), \M, \Lambda\}$ of Linear Programming as defined in \eqref{eq:LP},  define:
\begin{equation}
\begin{split}
	S^1 & \coloneq \{ 4\kappa e_1 \}; \quad
	S^2  \coloneq \{ 4\kappa e_2 \}; \quad
	S^0 \coloneq \{4\kappa(te_1 +(1-t)e_2) \ | \ t \in [0,1]\} \\
	y^1 & = 4\kappa e_1; \quad 
	y^2 = 4\kappa e_2. \quad \label{eq:y^j_lp}
\end{split}
\end{equation}

 \emph{Case (BP):}  For the computational problem $\{\Xi_{\BP}, \Omega_{\no,\nt}(\theta), \M, \Lambda\}$ of Basis Pursuit as defined in  \eqref{eq:BP}, and recalling from \S \ref{section:setup_CRP} that $\eta \in \Q$ satisfies$\kappa < \eta \leq 2\kappa$, define: 
\begin{equation}
\begin{split}
	S^1 & \coloneq \left\{ 2(2\kappa-\eta) e_1 \right\}; \quad
	S^2  \coloneq \left\{ 2(2\kappa-\eta) e_2 \right\};\quad
	S^0 \coloneq \left\{ 2(2\kappa-\eta)(te_1 +(1-t)e_2) \ | \ t \in [0,1] \right\} \\
	y^1 & = 2(2\kappa-\eta) e_1; \quad 
	y^2 = 2(2\kappa-\eta)  e_2. \quad \label{eq:y^j_bp}
\end{split}
\end{equation}

\emph{Case (LASSO):} For the computational problem $\{\Xi_{\LASSO}, \Omega_{\no,\nt}(\theta), \M, \Lambda\}$ of LASSO as defined in \eqref{eq:LASSO}, and recalling from \S \ref{section:setup_CRP} that $\lambda \in \Q$ satisfies $\kappa < \lambda \leq 2\kappa$, define:
\begin{equation}
\begin{split}
	S^1 & \coloneq \left\{ 2(2\kappa-\lambda) e_1 \right\}; \quad
	S^2  \coloneq \left\{ 2(2\kappa-\lambda) e_2 \right\};\quad
	S^0 \coloneq \left\{ 2(2\kappa-\lambda)(te_1 +(1-t)e_2) \ | \ t \in [0,1] \right\} \\
	y^1 & = 2(2\kappa-\lambda) e_1; \quad 
	y^2 = 2(2\kappa-\lambda)  e_2. \quad \label{eq:y^j_lasso}
\end{split}
\end{equation}

In the following Lemma, we explicitly verify the validity of various assumptions from previous Lemmas and Propositions involving the mathematical objects defined above. In particular, it will follow that Lemma \ref{lem:constructing_phi^m} and Propositions \ref{prop:DrivingNegativeProposition}, \ref{prop:EF} and \ref{prop:EF_oracle} can be applied.

\begin{lemma}\label{lemma:assumptions_verification}
With the definitions and assumptions provided in \S \ref{sec:collection_problems}, consider the inputs $\{\iota^1_n\}_{n \in \N}$,$\{\iota^2_n\}_{n \in \N}$, $\iota^0$, the sets $S^0$, $S^1$, $S^2$, and the vectors $y^1$, $y^2$. Then assumptions 
\ref{assumption:ComputableInput_lemma} and \ref{assumption:Del1withIota0_lemma}  of Lemma \ref{lem:constructing_phi^m},
 \ref{assumption:exists_xj_in_intersection_correspondence} of Proposition \ref{prop:de-oracolisation}, 
 \ref{assumption:ComputableInput}, \ref{assumption:Del1withIota0}, \ref{assumption:S1S2Distant}, \ref{assumption:MembershipComputableS2} of Proposition \ref{prop:DrivingNegativeProposition}, 
 and all assumptions \ref{assumption:EFComputableInput}-\ref{assumption:EFexists_xj_in_intersection} of \S \ref{sec:assumptions_for_exit_flag} hold.

\end{lemma}
\begin{proof}

Assumption \ref{assumption:ComputableInput_lemma} of Lemma \ref{lem:constructing_phi^m}, assumption  \ref{assumption:ComputableInput} of Proposition \ref{prop:DrivingNegativeProposition}, and assumption  \ref{assumption:EFComputableInput} from \S \ref{sec:assumptions_for_exit_flag} hold because we can define algorithms as follows:
$\hat \Gamma^{j}_i: \mathbb{N} \times \mathbb{N} \to \Q$ with $\hat \Gamma^{j}_i(n,r) = f_i(\iota^{j}_r)$, and $\hat \Gamma^{0}_i: \mathbb{N} \to \Q$ with $\hat \Gamma^{0}_i(n) = f_i(\iota^0)$, which satisfy $|\hat \Gamma^{j}_i(n,r) - f_i(\iota^{j}_r)| \leq 2^{-n-1}$ and $|\hat \Gamma^{0}_i(n) - f_i(\iota^{0})| \leq 2^{-n-1}$ for every $n \in \N$. These algorithms are well defined since $f_i(\iota^0), f_i(\iota^{j}_r)$ are each rational numbers for $i \in \{1,2,\dotsc,k\}$, $j \in \{1,2\}$, and $r \in \N$.

Assumption \ref{assumption:Del1withIota0_lemma} of Lemma \ref{lem:constructing_phi^m}, assumption  \ref{assumption:Del1withIota0} of Proposition \ref{prop:DrivingNegativeProposition} and assumption \ref{assumption:EFDel1Info} from \S \ref{sec:assumptions_for_exit_flag} hold because, by direct computation from \eqref{eq:iota}, $|f_i(\iota^j_n)-f_i(\iota^0)|\leq 2^{-n}$ for every $i \in \{1,\dotsc, k\}$, $j \in \{1,2\}$ and $n \in \N$.  

Assumption \ref{assumption:S1S2Distant} of Proposition \ref{prop:DrivingNegativeProposition} and \ref{assumption:EFS1S2Distant} from \S \ref{sec:assumptions_for_exit_flag} follow from Lemma \ref{lemma:ProblemBasicExampleLP}, \ref{lemma:ProblemBasicExampleBPDNL1} or \ref{lemma:ProblemBasicExampleULASSO}, depending on the specific problem under consideration. Indeed, for Linear Programming,
	$\inf_{\xi_1 \in S^1, \xi_2 \in S^2}d(\xi_1,\xi_2) = d(4\kappa e_1,4 \kappa e_2) = 4\kappa \|e_1-e_2\|_p \geq 4\kappa > 2\kappa$;
	 for Basis Pursuit, since $\kappa < \eta \leq 2\kappa$, we have that $\inf_{\xi_1 \in S^1, \xi_2 \in S^2}d(\xi_1,\xi_2) = d(2(2\kappa-\eta) e_1,2(2\kappa-\eta)  e_2) = (4\kappa-2\eta)  \|e_1-e_2\|_p  \geq 4\kappa-2\eta > 2\kappa$;
	and finally for Lasso, using $\kappa < \lambda \leq 2\kappa$, we have	
\begin{equation*}
\begin{split}
\inf_{\xi_1 \in S^1, \xi_2 \in S^2}d(\xi_1,\xi_2) &= d(2(2\kappa-\lambda) e_1,2(2\kappa-\lambda)  e_2) \\
&= (4\kappa-2\lambda)  \|e_1-e_2\|_p = (4\kappa-2\lambda)  \geq 4\kappa-2\lambda > 2\kappa.
\end{split}
\end{equation*}

To verify assumption \ref{assumption:MembershipComputableS2} of Proposition \ref{prop:DrivingNegativeProposition}, we note that membership in $\ball{\kappa}{S^2}$ is computable because for any $x \in \Q^d$, $x \in \ball{\kappa}{S^2}$ if and only if $\|x - y^2 \|_p \leq \kappa$, which is equivalent to either $\max_{i =1 ,\cdots, k}\{|x_i-y^2_i\} \leq \kappa$ in the case $p = \infty$ (which can clearly be verified computationally), or it is equivalent to to $|x_1|^p + |x_2-y^2_2|^p + |x_3|^p + \dots + |x_d|^p \leq \kappa^p$ if $p \in \N$. This inequality only involves positive integer exponents of rational numbers (since $p \in \N$ and $\kappa \in \Q$ by assumption) and therefore both the left hand side and right hand sides of this inequality can be computed exactly and thus compared computationally.  Moreover, after noting that $S^1\setminus S^0 = S^2 \setminus S^0 = \emptyset$, the same proof shows the validity of assumption \ref{assumption:EFMembershipComputableS1S2} of \S \ref{sec:assumptions_for_exit_flag}.

	Assumption \ref{assumption:EFXiIotajInSj} of \S \ref{sec:assumptions_for_exit_flag} is easily verified, since by Lemmas  \ref{lemma:ProblemBasicExampleLP}, \ref{lemma:ProblemBasicExampleBPDNL1} and \ref{lemma:ProblemBasicExampleULASSO} we have $\Xi(\iota^0) =  S^0$, $\Xi(\iota^1_n)  = S^1, \Xi(\iota^2_n) = S^2$ for every $n \in \N$. 
	
Assumption \ref{assumption:EFXiOmegaInBallAroundXiIota0S1AndS2} holds since by Lemma \ref{lemma:ProblemBasicExampleLP}, \ref{lemma:ProblemBasicExampleBPDNL1} or \ref{lemma:ProblemBasicExampleULASSO} we have $\Xi(\Omega) = \Xi(\iota^0)=S^0$ and thus 
\[
\Xi(\Omega) \subseteq \ball{\kappa-\alpha}{S^0} \cup \ball{\kappa - \alpha}{S^1} \cup \ball{\kappa - \alpha}{S^2}.
\]
 
 Assumption \ref{assumption:EFConstantIotaj} of \S \ref{sec:assumptions_for_exit_flag} is immediate from the fact that $\Xi(\iota^j_n)$ is independent of $n$ and satisfies $S^j = \Xi(\iota^j_n)$ for $j \in \{1,2\}$ and $n \in \N$, as is guaranteed by Lemma \ref{lemma:ProblemBasicExampleLP}, Lemma \ref{lemma:ProblemBasicExampleBPDNL1} and Lemma \ref{lemma:ProblemBasicExampleULASSO}. 
 
Assumption \ref{assumption:EFexists_xj_in_intersection} of \S \ref{sec:assumptions_for_exit_flag} and assumption \ref{assumption:exists_xj_in_intersection_correspondence} of Proposition \ref{prop:de-oracolisation} follow by the the rationality of $\kappa,\eta $ and $\lambda$ as well as the definition of $y^j$ in \eqref{eq:y^j_lp}, \eqref{eq:y^j_bp} and \eqref{eq:y^j_lasso} for each computational problem ($\Xi_{\LP}$, $\Xi_{\BP}$ and $\Xi_{\LASSO}$). 

This concludes the proof that all assumptions involving the mathematical objects constructed in the current section \S \ref{sec:collection_problems} are satisfied.
\end{proof}

\subsection{The proof of the CRP  I-V} \label{sec:proof_crp_parameters}

Armed with the setup and discussions in \S \ref{section:setup_CRP} and \S \ref{sec:geometry}, we proceed to prove the theorems that collectively constitute the CRP, namely Theorem \ref{thm:crp_1_2}, Theorem \ref{thm:crp_3_4}, and Theorem \ref{thm:crp_5}. Recall that, according to the setup of \S \ref{section:setup_CRP}, we take $\no \geq 2$, $\nt \geq 1$, $\kappa = 10^{-1}$, $\lambda\in\Q$ such that $\kappa < \lambda \leq 2\kappa$, and $\eta \in \Q$ such that $\kappa < \eta \leq 2\kappa$, and assume that $\R^d$ is equipped with the distance induced by the $\|\cdot\|_p$ norm for $p \in \mathbb{N} \cup \{\infty\}$. Moreover, recall from the statement of the aforementioned theorems that, whenever parameters $\alpha$ and $\omega$ appear, they satisfy $\alpha, \omega\in \Q$ and $0 \leq \alpha \leq \omega  < \kappa$.
In the proofs of Theorem \ref{thm:crp_1_2}, Theorem \ref{thm:crp_3_4}, and Theorem \ref{thm:crp_5} we will often make use of a particular algorithm $\Gamma^{\always}$ which we define and analyse in the next Lemma.

\begin{lemma}\label{lemma:routine}

	Consider the setup of \S \ref{section:setup_CRP}, where $\Omega = \Omega_{\no,\nt}$ is defined as in \eqref{eq:set_inputs}. Then for every $\stringinput =(\stringinput_1, \dotsc, \stringinput_k)\in \MarkovOmega$ the following holds
	\begin{enumerate}
	\item  if $\stringinput$ corresponds to $\iota^0$, then $|\stringinput_1(n) - \stringinput_2(n)| \leq 2 \cdot 2^{-n}$ for every $n \in \N$; \label{conclusion:2n_iota0}
	\item if $\stringinput$ does not correspond to $\iota^0$, then there exists $n_\stringinput \in \N$ such that $|\stringinput_1(n) - \stringinput_2(n) | > 2 \cdot 2^{-n}$ for every $n \geq n_\stringinput$ \label{conclusion:2n_not_iota0}.
	\end{enumerate}
Moreover, there exists an algorithm $\Gamma^{\always} \colon \MarkovOmega \to \M$ such that for every $\stringinput \in \MarkovOmega$, the following holds.
	\begin{enumerate}[label = (\roman*)]
	\item if $\stringinput$ corresponds to $\iota^0$, then $\Gamma^{\always}(\stringinput) \uparrow$; \label{conclusion:nh}
	\item  if $\stringinput$ does not correspond to $\iota^0$, then $\Gamma^{\always}(\stringinput) \in \ball{\kappa}{\MarkovXi(\stringinput)}$. \label{conclusion:correct}
	\end{enumerate}

\end{lemma}

\begin{proof}

Consider the following algorithm $\Gamma^{\always} \colon \MarkovOmega \to \M$. \emph{Construction of $\Gamma^{\always}$:} on input $\stringinput = (\stringinput_1,\stringinput_2,\dots,\stringinput_k) \in \MarkovOmega$, $\Gamma^{\always}$ acts as follows: initialise $n =1$, then:
\begin{algosteps}[leftmargin = 6mm]
\itemsep0em
	\item Compute $\delta \coloneq \stringinput_1(n) - \stringinput_2(n)$, and branch depending on the value of $\delta$:
	\begin{algosteps}
		\item if $\delta >  2 \cdot 2^{-n}$, then return $\Gamma^{\always}(\stringinput) \coloneq y^1$; \label{step:d+}
		\item if $\delta < -2 \cdot 2^{-n}$, then return $\Gamma^{\always}(\stringinput) \coloneq y^2$; \label{step:d-}
	\end{algosteps}  \label{step:d_lemma}
	\item If neither of these conditions are met then the loop continues by incrementing $n$ and executing the next iteration from \ref{step:d_lemma}.
\end{algosteps}
where $y^1$ and $y^2$ are defined as in \S \ref{sec:collection_problems}. We now prove all the conclusions of the lemma. Let $\stringinput \in \MarkovOmega$. There are exactly two possibilities:

\emph{Case (*)}: $\stringinput$ corresponds to $\iota^0$. Then, since $\iota^0 = (b,U(1/2,1/2))$, we have that at each stage $n$ of the loop the value $\delta$ satisfies $|\delta| = |\stringinput_1(n) - \stringinput_2(n)| \leq |\stringinput_1(n) - 1/2 | + |1/2 - \stringinput_2(n)| \leq 2 \cdot 2^{-n}$ using Definition \ref{def:Markov_Delta_1_Information}. This proves \eqref{conclusion:2n_iota0}. Furthermore, $\Gamma$ does not halt at \ref{step:d+} nor \ref{step:d-} for any $n$. This shows that, in this case, $\Gamma(\stringinput) \uparrow$, which proves \ref{conclusion:nh}.

\emph{Case (**)}: $\stringinput$ does not correspond to $\iota^0$. Then by definition of $\Omega_{\no,\nt}$ as in \eqref{eq:set_inputs}, the input $\iota =\CorrXi(\stringinput)$ must be of the form $\iota = (b(m), U(u_1,u_2,m,N))$ with either $u_1 = \frac{1}{2} > u_2$ or $u_1 < \frac{1}{2} = u_2$ (since the only input $\iota$ for which $u_1 = u_2$ is $\iota^0$, in which case $u_1=u_2 = \frac{1}{2}$). Assume first that $\iota = \CorrXi(\stringinput)$ is such that $u_1=\frac{1}{2} > u_2$. In this case, by the setup of \S \ref{sec:collection_problems} and Lemmas \ref{lemma:ProblemBasicExampleLP}, \ref{lemma:ProblemBasicExampleBPDNL1} and \ref{lemma:ProblemBasicExampleULASSO}, we have $\Xi(\iota) = \MarkovXi(\stringinput)= y^1$. Note that, because $\stringinput_1$ and $\stringinput_2$ satisfy the requirements of Definition \ref{def:Markov_Delta_1_Information}, the value of $\delta$ after $n$ iterations satisfies $\delta = \stringinput_1(n)-\stringinput_2(n) \geq (u_1 - 2^{-n}) - (u_2 + 2^{-n}) \geq (u_1-u_2) -2 \cdot 2^{-n}  > -2 \cdot 2^{-n}$, and hence \ref{step:d-} never results in the termination of the algorithm. By contrast, since $u_1 > u_2$, there exists $\epsilon > 0$ such that $u_1 > u_2 + \epsilon$. 
		Thus (once again using that $\stringinput_1$ and $\stringinput_2$ satisfy the requirements of Definition \ref{def:Markov_Delta_1_Information}), we see that \[\delta = \stringinput_1(n)-\stringinput_2(n) \geq (u_1 - 2^{-n}) - (u_2 + 2^{-n}) \geq (u_1-u_2) -2 \cdot 2^{-n} > \epsilon - 2\cdot 2^{-n} > 2 \cdot 2^{-n}\] for sufficiently large $n$, so that there exists $n_\stringinput \in \N$  such that $|\stringinput_1(n)-\stringinput_2(n)| = \stringinput_1(n)-\stringinput_2(n) > 2 \cdot 2^{-n}$ for $n \geq n_{\stringinput}$. Therefore there is an iteration $n$ such that the algorithm $\Gamma$ halts at \ref{step:d+} and returns $\Gamma(\stringinput)$ with $\Gamma(\stringinput) = y^1 = \Xi(\iota) = \MarkovXi(\stringinput)$, as desired. The case $u_2 = \frac{1}{2} > u_1$ is analogous. This proves \eqref{conclusion:2n_not_iota0} and \ref{conclusion:correct}, which concludes the proof of the lemma.
\end{proof}

We now prove  in succession the three theorems that collectively establish the CRP.

\begin{proof}[Proof of Theorem \ref{thm:crp_1_2}, CRP \eqref{crp:1} and \eqref{crp:2}]
	
	Consider the setup of \S \ref{section:setup_CRP} and the solution mappings $\Xi_{\LP}$, $\Xi_{\BP}$ and $\Xi_{\LASSO}$, corresponding to the \eqref{eq:LP}, \eqref{eq:BP} and \eqref{eq:LASSO} problems respectively. Let $\theta \in [1/8,1/4] \cap \Q$ be arbitrary, and consider the input set $\O= \Omega_{\no,\nt}(\theta)$ defined in \eqref{eq:set_inputs}. The proof is almost identical for the mappings $\Xi_{\LP}$, $\Xi_{\BP}$ and $\Xi_{\LASSO}$. Therefore we will use $\Xi$ to denote any of them as appropriate and only mention where this particular choice leads to small differences in the argument. Indeed, let $S^1$, $S^2$, $S^0$, $\{\iota^1_n\}_{n \in \N}$, $\{\iota^2_n\}_{n \in \N}$, $\iota^0$, $y^1$ and $y^2$ be defined as in \S \ref{sec:collection_problems}, depending on the choice of either $\Xi_{\LP}$, $\Xi_{\BP}$ or $\Xi_{\LASSO}$, and recall the definition of $\CorrXi$ from Definition \ref{def:correspondence_problem}.

	\textit{Proof of \eqref{crp:1}}: Let $\hat \Omega \subseteq \MarkovOmega$ such that $|(\CorrXi)^{-1}(\iota) \cap \hat \Omega| = 1$ for every $\iota \in \Omega$. In particular, there exists one string $\stringinput^0 = (\stringinput_1^0,\stringinput_2^0,\dots,\stringinput_k^0) \in \MarkovOmega$ that corresponds to $\iota^0 \in \Omega$. 
	We construct the following algorithm $\Gamma$ which itself utilises the algorithm $\Gamma^{\always}$ from Lemma \ref{lemma:routine}.
	
	\textit{Construction of $\Gamma$:} on input $\stringinput = (\stringinput_1,\stringinput_2,\dots,\stringinput_k) \in \MarkovOmega$, $\Gamma$ acts as follows:
	\begin{algosteps}[leftmargin = 6mm]
	\itemsep0em
		\item if $\stringinput = \stringinput^0$, then return $\Gamma(\stringinput) \coloneq y^1$ \label{step:hardcoded};
		\item otherwise, if $\stringinput \neq \stringinput^0$, then run the algorithm $\Gamma^{\always}$ on $\stringinput$ and return $\Gamma(\stringinput) \coloneq \Gamma^{\always}(\stringinput)$. 
	\end{algosteps}
We now prove that, for every $\stringinput \in \MarkovOmega$, either $\Gamma(\stringinput) \in \ball{\kappa}{\MarkovXi(\stringinput)}$ or $\Gamma(\stringinput) \nh$,
	and that $\Gamma(\stringinput) \in \ball{\kappa}{\MarkovXi(\stringinput)}$ for every $\stringinput \in \hat \Omega$. We will consider three cases.
	
 \emph{Case (*):}	If $\stringinput = \stringinput^0$, then $\Gamma$ halts at \ref{step:hardcoded} and returns $\Gamma(\stringinput) = y^1$. Note that by either Lemma \ref{lemma:ProblemBasicExampleLP}, Lemma \ref{lemma:ProblemBasicExampleBPDNL1} or Lemma \ref{lemma:ProblemBasicExampleULASSO} depending on the problem under consideration, it follows that $y^1\in \Xi(\iota^0) = \MarkovXi(\stringinput)$ as desired.
	
 \emph{Case (**):}	 If $\stringinput \neq \stringinput^0$ and $\CorrXi(\stringinput) = \iota^0$, then $\Gamma^{\always}(\stringinput) \uparrow$ by conclusion \ref{conclusion:nh} of Lemma \ref{lemma:routine} and thus $\Gamma(\stringinput) \uparrow$. 
		
 \emph{Case (***):}	If $\stringinput \neq \stringinput^0$ and $\CorrXi(\stringinput) \neq \iota^0$, then $\Gamma(\stringinput) \in \ball{\kappa}{\MarkovXi(\stringinput)}$ by conclusion \ref{conclusion:correct} of Lemma \ref{lemma:routine}. 
 Thus, all the desired properties for $\Gamma$ from the statement of Theorem \ref{thm:crp_1_2}  have been verified.

Finally, to conclude the proof of \eqref{crp:1}, it remains to show that the correspondence problem with oracle $\{\CorrXi,\MarkovOmega,\Omega,\MarkovLambda\}^{\mathcal{O}}$ (as per Definitions \ref{def:correspondence_problem} and \ref{def:problem_with_oracle}) is such that any algorithm that attempts to solve it will fail. By contradiction, assume that 
$
\CorrGO \colon \domainOracleXi \to \Omega
$
 is an algorithm that always halts (since the result trivially holds if $\CorrGO$ does not halt on some input) such that $\CorrGO(\stringinput,y) \in \CorrXiO(\stringinput,y)$ for every $(\stringinput, y) \in \domainOracleXi$. Our strategy is to apply Proposition \ref{prop:de-oracolisation}. More specifically, in the notation of Proposition \ref{prop:de-oracolisation}, we set $\Omega' \coloneq \{\iota^1_n\}_{n=0}^\infty \cup \{\iota^2_n\}_{n =0}^\infty \cup \{\iota^0\}$,  and 
 \[
 \subsetMarkov \coloneq \{\stringinputm\}_{m \in \N} \subseteq (\Omega')^M \subseteq  \MarkovOmega.
 \]
  We also define $\Xi_1\colon \Omega \to \Omega$ given by $\Xi_1(\iota) \coloneq \iota$ so that $\MarkovXi_1 \colon \subsetMarkov \to \O$ satisfies $\MarkovXi_1 = \CorrXi|_{\subsetMarkov}$, and $\Xi_2 \colon \Omega \to \M$ given by $\Xi_2 \coloneq \Xi$ so that $\MarkovXi_2 = \MarkovXi$. 

We now verify that all the assumptions of Proposition \ref{prop:de-oracolisation} hold. Assumption \ref{assumption:Xi1_single_valued} clearly holds because $\Xi_1\colon \Omega \to \Omega, \iota \mapsto \iota$ is single-valued. Assumption \ref{assumption:exists_xj_in_intersection_correspondence} follows from the definitions of $y^1$ and $y^2$ and of the sets $S^0$, $S^1$ and $S^2$ defined in \S\ref{sec:collection_problems}, as already noted in Lemma \ref{lemma:assumptions_verification}. Assumption \ref{assumption:exists_correct_oracle_algorithm} follows from the current hypothesis on the existence of the algorithm $\CorrGO$, by setting $\GO \coloneq \CorrGO|_{\subsetMarkov}$. Finally, it remains to prove assumption \ref{assumption:exists_algorithm_check_j}, which we proceed to show by explicitly constructing an algorithm 
$
\largerGamma \colon \subsetMarkov \to \{1,2\}
$
such that $\largerGamma(\stringinput) = j$ if $\CorrXi(\stringinput) = \iota^j_n$ for some $n \in \N$ and $\largerGamma(\stringinput) \uparrow$ if $\CorrXi(\stringinput) = \iota^0$. 
	The algorithm $\largerGamma$ works as follows: $\largerGamma$ on input $\stringinputm$ runs $\varphi_m(m)$ until it terminates (if this does not occur then clearly $\largerGamma$ does not terminate either). If $\varphi_m(m) \in \{1,2\}$ then $\largerGamma(\stringinputm)$ is set to $\varphi_m(m)$. If $\varphi_m(m)$ terminates with a value other than $1$ or $2$, $\largerGamma$ enters an infinite loop. More concisely, we have
	\begin{align*}
		\largerGamma(\stringinputm) \coloneq 
		\begin{cases}
			\varphi_m(m) & \text{ if } \varphi_m(m) \in \{1,2\} \\
			\uparrow & \text{otherwise}
		\end{cases}
	\end{align*} 
	for every $m \in \N$. It follows immediately from Lemma \ref{lem:constructing_phi^m}, equation \eqref{eq:EFInitialIotaCorrespond} that $\largerGamma$ is such that $\largerGamma(\stringinputm) = j$ if $\stringinputm$ corresponds to $\iota^j_n$ for some $n \in \N$, while $\largerGamma(\stringinputm) \uparrow$ if $\stringinputm$ corresponds to $\iota^0$.
	
Thus, by Proposition \ref{prop:de-oracolisation}, there exists an algorithm $\CorrG\colon \{\stringinputm\}_{m \in \N} \to \Omega$ that always halts, such that $\CorrG(\stringinputm) = \CorrXi(\stringinputm)$ for every $m \in \N$. Now, consider the algorithm $\Gamma' \colon \Omega \to \Q^m$,
	\begin{align*}
		\Gamma'(b,U(u_1,u_2)) \coloneq
		\begin{cases}
			y^1 & \text{ if } u_1 \geq u_2; \\
			y^2 & \text{ if } u_1 < u_2.
		\end{cases}
	\end{align*}
	(where $y^1$ and $y^2$ are chosen according to \eqref{eq:y^j_lp}, \eqref{eq:y^j_bp} or \eqref{eq:y^j_lasso} as appropriate).
	Then, the composition $\Gamma\coloneq \Gamma' \circ \CorrG \colon \{\stringinputm\}_{m \in \N} \to \M$ is an algorithm that always halts and such that $\Gamma(\stringinputm) \in \MarkovXi(\stringinputm)$ for every $m \in \N$ by Lemmas \ref{lemma:ProblemBasicExampleLP}, \ref{lemma:ProblemBasicExampleBPDNL1} and \ref{lemma:ProblemBasicExampleULASSO}.
	On the other hand, Lemma \ref{lemma:assumptions_verification} ensures that we can apply Proposition \ref{prop:DrivingNegativeProposition}, which states there exists $q \in \N$ such that $d_\mathcal{M}(\Gamma(\stringinputq), \MarkovXi(\stringinputq))> \kappa$. This is a contradiction. Therefore, we have shown that for any algorithm with oracle 
$
\CorrGO \colon \domainOracleXi \to \Omega
$
 there exists 
 \[
 (\stringinput,y) \in \domainOracleXi \text{ such that } \CorrGO(\stringinput,y) \notin \CorrXiO(\stringinput,y).
 \]

	\textit{Proof of \eqref{crp:2}}: Let $\Gamma \colon \MarkovOmega \to \M$ be an algorithm. If there exists $\stringinput \in \MarkovOmega$ such that $\G(\stringinput) \uparrow$, then by Definition \ref{def:failure}, $\Gamma$ $\kappa$-fails on $\Phi$. Alternatively, if the algorithm $\Gamma$ always halts, then Lemma \ref{lemma:assumptions_verification} guarantees that all assumptions of Proposition \ref{prop:DrivingNegativeProposition} are verified in the setup of  \S\ref{section:setup_CRP}.
Therefore, we can apply Proposition \ref{prop:DrivingNegativeProposition} and conclude that there exists $m \in \N$ such that	
\[
d_\M(\Gamma(\stringinputm), \MarkovXi(\stringinputm))> \kappa.
\]
	 Thus for such an $m \in \N$, 
	$\Gamma(\stringinputm) \notin \ball{\kappa}{\MarkovXi(\stringinputm)}$ so that $\Gamma$ $\kappa$-fails on $\stringinputm \in \MarkovOmega$ as per Definition \ref{def:failure}.
	
	To show that there exist infinitely many other $\stringinput' \in \MarkovOmega$ on which $\Gamma$ $\kappa$-fails, assume by contradiction that there are only finitely many $\stringinput^1,\dots,\stringinput^n$ such that $\Gamma$ $\kappa$-fails on $\stringinput^i$ for $i = 1,\dotsc,n$. For every $\iota^i = \CorrXi(\stringinput^i)$, let $x^i \in \ballQ[\kappa]{\Xi(\iota^i)}$. Then the following algorithm:
	\begin{equation*}
		\tilde\Gamma\colon \MarkovOmega \to \M, \quad 
		\stringinput  \mapsto 
		\begin{cases}
			x^i & \text{ if } \stringinput = \stringinput^i \text{ for some $i = 1,\dots,n$} \\
			\Gamma(\stringinput)  & \text{ otherwise},
		\end{cases}
	\end{equation*}
	always halts and is such that $\tilde\Gamma(\stringinput) \in \ball{\kappa}{\MarkovXi(\stringinput)}$ for every $\stringinput \in \MarkovOmega$. This contradicts the above argument (applied to the algorithm $\tilde\Gamma$). Therefore, there must be infinitely many inputs in $\MarkovOmega$ for which $\tilde\Gamma$ $\kappa$-fails. The second part of the statement follows directly from Remark \ref{rmk:equivalence_failure_hallucination}. We have thus completed the proof of \eqref{crp:2}.
\end{proof}

\begin{remark}
	The final paragraph of the proof of \eqref{crp:1}, after the `oracle has been removed', ends with simply the correspondence problem. The fact that the the correspondence problem is non-computable can be shown in many ways, for example using Rice's theorem \cite{Soare}. However, the point of our result is to show that computing $\Xi^M$ (being consistently reasoning) is strictly easier than the correspondence problem. 
\end{remark}

\begin{proof}[Proof of Theorem \ref{thm:crp_3_4}, CRP \eqref{crp:3} and \eqref{crp:4}:]
	As in the proof of \ref{thm:crp_1_2}, consider the setup of \S \ref{section:setup_CRP} and the solution mappings $\Xi_{\LP}$, $\Xi_{\BP}$ and $\Xi_{\LASSO}$, corresponding to the \eqref{eq:LP}, \eqref{eq:BP} and \eqref{eq:LASSO} problems respectively. Let $\theta \in [1/8,1/4] \cap \Q$ be arbitrary, and consider the input set $\O= \Omega_{\no,\nt}(\theta)$ defined in \eqref{eq:set_inputs}. The proof is almost identical for the mappings $\Xi_{\LP}$, $\Xi_{\BP}$ and $\Xi_{\LASSO}$. Therefore we will use $\Xi$ to denote any of them as appropriate and only mention where this particular choice leads to small differences in the argument. Indeed, let $S^1$, $S^2$, $S^0$, $\{\iota^1_n\}_{n \in \N}$, $\{\iota^2_n\}_{n \in \N}$, $\iota^0$, $y^1$ and $y^2$ be defined as in \S \ref{sec:collection_problems}, depending on the choice of either $\Xi_{\LP}$, $\Xi_{\BP}$ or $\Xi_{\LASSO}$, and recall the definition of $\CorrXi$ from Definition \ref{def:correspondence_problem} and that the parameters $\alpha, \omega\in \Q$ satisfy $0 \leq \alpha \leq \omega  < \kappa$.

	\textit{Proof of \eqref{crp:3}}: 
	Let $\Gamma \colon \MarkovOmega \to \ball{\kappa}{\MarkovXi(\MarkovOmega)}$ be an algorithm that is within the $\alpha$-range of $\MarkovXi$ as per Definition \ref{def:within_the_range}. 
	
	First, we note that by Lemma \ref{lemma:assumptions_verification} all assumptions \ref{assumption:EFComputableInput}-\ref{assumption:EFexists_xj_in_intersection} of \S \ref{sec:assumptions_for_exit_flag} hold.  In particular, both Proposition \ref{prop:EF} and  \ref{prop:EF_oracle} can be applied.
		
	\textit{Proof of \eqref{crp:3a}}  Let $\Gamma^E\colon \MarkovOmega \to \{0,1\}$ be an algorithm that always halts. We proceed to prove that there exist infinitely many $\stringinput \in \MarkovOmega$ such that $\Gamma^E(\stringinput)  \neq  \MarkovExitXi(\stringinput)$. Since assumptions \ref{assumption:EFComputableInput}-\ref{assumption:EFConstantIotaj} from \S \ref{sec:assumptions_for_exit_flag} hold, we can use Proposition \ref{prop:EF} to see that there exists $m \in \N$ such that 
\[
\Gamma^E(\stringinputm) \neq \MarkovExitXi(\stringinputm), \quad \{\stringinputm\}_{m \in \N}\subseteq \MarkovOmega,
\]
 as defined as in Lemma \ref{lem:constructing_phi^m}. This shows that there exists input on which $\Gamma^E$ hallucinates. The argument that there are infinitely many such inputs is identical to the one in part \eqref{crp:2} of Theorem \ref{thm:crp_1_2}: if there were only finitely many, they could be hardcoded as exceptional cases to form a new algorithm that would never hallucinate, yielding a contradiction with 
	Proposition \ref{prop:EF}. This concludes the first part of \ref{crp:3a}.
	
	Now let $\GammaEO:\domainOracleXi \to \{0,1\}$  be an algorithm that always halts for the exit-flag problem associated to $\Gamma$ with an oracle for $\MarkovXi$. We proceed to prove that there exist infinitely many $(\stringinput,y) \ \in \domainOracleXi$ such that $\GammaEO(\stringinput,y) \notin \MarkovExitXiOracle(\stringinput,y)$. Since assumptions \ref{assumption:EFComputableInput}-\ref{assumption:EFexists_xj_in_intersection} from \S \ref{sec:assumptions_for_exit_flag} hold, by Proposition \ref{prop:EF_oracle}
	there exist 
\[
m \in \N \text{ and } y \in \ballQ{\MarkovXi(\MarkovOmega)} \text{ such that } \GammaEO(\stringinputm,y) \notin \MarkovExitXiOracle(\stringinputm,y),
\]
 and thus $\GammaEO$ hallucinates on input $\stringinput$ as per Definition \ref{def:hallucination}. The argument that there are infinitely many such inputs on which $\GammaEO$ hallucinates is identical to the one above: if there were only finitely many, they could be hardcoded as exceptional cases in a new algorithm that would never hallucinate, yielding a contradiction. This concludes the proof of \eqref{crp:3a}.	
	
	\textit{Proof of \eqref{crp:3b}}: 
	Let $\subsetMarkov \subseteq \MarkovOmega$ be any subset, and assume $\Gamma^{E, \ran}:\subsetMarkov \to \{0,1\}$ is a probabilistic Turing machine such that 
\begin{align}\label{eq:prob_greater_1/2_2}
					\mathbb{P}\left(\Gamma^{E,\ran}(\stringinput) = \MarkovExitXi(\stringinput)\right) > \frac{1}{2}
\end{align}
holds for every $\stringinput \in \subsetMarkov$. Applying Proposition \ref{prop:de-randomisation}, point \eqref{prop:de-randomisation_single_valued}, to the single-valued computational problem $\{\MarkovExitXi,\subsetMarkov,\{0,1\},\MarkovLambda\}$ and the Probabilistic Turing machine $\Gamma^{E,\ran}$, there must exist a deterministic Turing machine $\Gamma^E$  that always halts such that $\Gamma^E(\stringinput) = \MarkovExitXi(\stringinput)$ for every $\stringinput \in \subsetMarkov$. This concludes the proof of the first part of \eqref{crp:3b}.
		Now, by contradiction, assume that there exists a probabilistic Turing machine $\Gamma^{E,\ran}: \MarkovOmega \to \{0,1\}$ such that \eqref{eq:prob_greater_1/2_2} holds for every $\stringinput \in \MarkovOmega$. Then, taking $\subsetMarkov = \MarkovOmega$ in 
		the above part, there exists a deterministic Turing machine $\Gamma^E\colon \MarkovOmega \to \{0,1\}$ that always halts such that $\Gamma^E(\stringinput) = \MarkovExitXi(\stringinput)$ for every $\stringinput \in \MarkovOmega$.
		However, this contradicts part \ref{crp:3a} from this Theorem. 
		Thus, no such probabilistic Turing machine $\Gamma^{E,\ran}$ can exist and the proof of \eqref{crp:3b} is completed.

	\textit{Proof of \eqref{crp:3c}}: 
	By contradiction, assume that there exists $p > 1/2$ and a probabilistic Turing machine $\GammaEO:\domainOracleXi \to \{0,1\}$ that always halts such that 
\[
\mathbb{P}\left(\Gamma^{E,\mathcal{O},\ran}(\stringinput,y) \in \MarkovExitXiOracle(\stringinput,y)\right) \geq p, \quad \forall \, (\stringinput,y) \in \domainOracleXi.
\]

Finally, assumption \ref{assumption:EFexists_xj_in_intersection} of \S \ref{sec:assumptions_for_exit_flag} holds since  $\Xi$ is multi-valued only at $\iota^0$ .

We now verify that the assumptions of Proposition \ref{prop:de-randomisation}, point \eqref{prop:de-randomisation_multi_valued}, hold, with respect to the multi-valued problem  $\{\MarkovExitXiOracle,\domainOracleXi ,\{0,1\},\MarkovLambdaOracle\}$ and to the Probabilistic Turing machine $\Gamma^{E,\mathcal{O},\ran}$ that always halts: in fact, we observe that $\Xi$ is multivalued only at $\iota^0$ by Lemma \ref{lemma:ProblemBasicExampleLP}, Lemma \ref{lemma:ProblemBasicExampleBPDNL1} and Lemma \ref{lemma:ProblemBasicExampleULASSO} and by the definition of $\Omega = \Omega_{\no,\nt}$ as in \eqref{eq:set_inputs}, and that in Proposition \ref{prop:de-randomisation}, point \eqref{prop:de-randomisation_multi_valued} we can take $y_0 \coloneq y^1$ where $y^1$ is given in assumption \ref{assumption:EFexists_xj_in_intersection}. Therefore, we can apply Proposition \ref{prop:de-randomisation}, point \eqref{prop:de-randomisation_multi_valued}, which guarantees that there exists a deterministic Turing machine $\GammaEO: \domainOracleXi \to \{0,1\}$ such that 
\[
\forall (\stringinput,y) \in\domainOracleXi \text{ we have } \GammaEO(\stringinput,y) \in \MarkovExitXiOracle(\stringinput,y).
\] This, however, contradicts the result from part \eqref{crp:3a}, which completes the proof of \eqref{crp:3c}.

	\textit{Proof of \eqref{crp:4}:}
	The following proof utilises a technique originated in \cite{post1944recursively} and revised throughout the literature, such as in \cite{franzen2005godel} where G\"odel-like theorems are proven via Hilbert's 10th Problem.
Assume that ZFC is $\Sigma_1$-sound, and consider the standard model of arithmetic. Recall that the hypothesis that ZFC is $\Sigma_1$-sound implies that ZFC is consistent (since an inconsistent theory can prove everything, even false $\Sigma^0_1$ statements, and thus cannot be $\Sigma_1$-sound). We now proceed to construct the class of inputs $\hat\Omega \subseteq \MarkovOmega$ and the algorithm $\Gamma \colon \hat\Omega \to \ball{\kappa}{\MarkovXi(\MarkovOmega)}$ mentioned in the statement of \eqref{crp:4}. Before doing so, we will prove some auxiliary claims. 
By Lemma \ref{lemma:assumptions_verification}, the assumptions of Lemma \ref{lem:constructing_phi^m} are met and thus we can consider the set 
	 \begin{align*}
	 G & \coloneq \{m \in \N \ | \ \stringinputm \text{ does not correspond to } \iota^0\}
	 \end{align*}
	 where $\{\stringinputm\}_{m \in \N} \subseteq \MarkovOmega$ is defined as in Lemma \ref{lem:constructing_phi^m}.

\textbf{Claim (A):} $G$ is recursively enumerable but not decidable.

	\emph{Proof of Claim (A)}: $G$ is recursively enumerable for the following reason: by Lemma \ref{lemma:routine}, $m \in G$ if and only if there is an $n_m \in \N$ such that $|\phi^m_1(n_m) - \phi^m_2(n_m)| > 2 \cdot 2^{-n_m}$. Therefore, $G$ is recursively enumerable because, for every $m \in \N$, one can repeatedly compute the difference $\delta\coloneq|\phi^m_1(n)-\phi^m_2(n)|$ over $n \in \N$, checking whether or not $\delta$ exceeds $2 \cdot 2^{-n}$ for sufficiently large $n$ -- thus certifying that $m \in G$ -- and otherwise not halting. Explicitly, we have that 
	\begin{align*}
	G = \{m \in \N\ | \ \exists n \in \N \text{ such that } |\phi^m_1(n)-\phi^m_2(n)| > 2\cdot 2^{-n} \},
	\end{align*}
	which is clearly a recursively enumerable set. However, $G$, it is not decidable. To see this, suppose (by contradiction) that there is an algorithm $\Gamma_G \colon \N \to \{0,1\}$ such that $\Gamma_G(m) = 1$ if $m \in G$ and $\Gamma_G(m) = 0$ if $m \notin G$. Then the computational problem $\{\MarkovXi, \{\stringinputm\}_{m \in \N}, \M, \MarkovLambda\}$ can be solved by the following algorithm $\Gamma^{\text{s}} \colon \{\stringinputm\}_{m \in \N} \to \M$:

	\textit{Construction of $\Gamma^{\text{s}}$:} on input $\stringinputm$, $\Gamma^{\text{s}}$ acts as follows: 
	\begin{algosteps}
	\itemsep0em
		\item if $\G_G(m)=0$, then return $\Gamma^{\text{s}}(\stringinputm) \coloneq y^1$ \label{step:hardcoded*};
		\item otherwise if $\G_G(m)=1$, then run the routine $\Gamma^{\always}$ from Lemma \ref{lemma:routine} on $\stringinputm$, and return $\Gamma^{\text{s}}(\stringinputm) \coloneq \Gamma^{\always}(\stringinputm)$.
	\end{algosteps}
We claim that the algorithm $\Gamma^{\text{s}}$ is such that $\Gamma^{\text{s}}(\stringinputm) \in \ball{\kappa}{\MarkovXi(\stringinputm)}$ for every $m \in \N$. Fix $m \in \N$ so that one of the following two cases occurs:

\emph{Case ($i$)}: if $\CorrXi(\stringinputm)=\iota^0$ (i.e., $\stringinputm$ corresponds to $\iota^0$) then $m \notin G$ and thus $\G_G(m) = 0$. In this case, $\Gamma^{\text{s}}$ outputs $y^1$ and, by the construction as in \S \ref{sec:collection_problems} (and in particular Lemma \ref{lemma:assumptions_verification}) we know that $y^1 \in \ball{\kappa}{\Xi(\iota^0)} = \ball{\kappa}{\MarkovXi(\stringinputm)}$. Thus $\Gamma^{\text{s}}(\stringinputm) = y^1 \in \ball{\kappa}{\MarkovXi(\stringinputm)}$ as desired.

\emph{Case ($ii$)}: if $\CorrXi(\stringinputm)\neq\iota^0$ (i.e., $\stringinputm$ does not correspond to $\iota^0$) then $m \in G$ and thus $\G_G(m) = 1$. In this case, $\Gamma^{\text{s}}$ behaves in the same way as the algorithm $\Gamma^{\always}$ from Lemma \ref{lemma:routine} - in particular, $\Gamma^{\text{s}}$ satisfies the result of Lemma \ref{lemma:routine}, \ref{conclusion:correct}. This leads to the desired conclusion that $\Gamma^{\text{s}}(\stringinputm) \in \ball{\kappa}{\MarkovXi(\stringinputm)}$.

Therefore, under the assumption that $G$ is decidable, we have constructed the algorithm $\Gamma^{\text{s}}$, such that $\Gamma^{\text{s}}(\stringinputm) \in \ball{\kappa}{\MarkovXi(\stringinputm)}$ for every $m \in \N$, thus solving the computational problem 
\[
\{\MarkovXi, \{\stringinputm\}_{m \in \N}, \M, \MarkovLambda\}.
\]
 On the other hand, by Lemma \ref{lemma:assumptions_verification}, we can apply Proposition \ref{prop:DrivingNegativeProposition} to get the existence of $m \in \N$ such that $d_\mathcal{M}(\Gamma^{\text{s}}(\stringinputm), \MarkovXi(\stringinputm))> \kappa$. This is a contradiction so that $G$ is not decidable. This concludes the proof of Claim (A).

	 \textbf{Claim (B):} There exist $m_0 \in \N$ such that it is \textit{true} that $m_0 \notin G$ in the standard model of arithmetic, but it is \textit{not provable} that `$m_0 \notin G$' within ZFC (assuming ZFC is consistent).

	 	\emph{Proof of Claim (B):}
	 We will argue by contradiction, after making two preliminary observations.  First, note that for every $m \in \N$, if it is true that $m \in G$ in the standard model, then there is a proof in ZFC of the statement `$m \in G$'.  In fact, the statement `$m \in G$' is a $\Sigma^0_1$-sentence (since $G$ is recursively enumerable by Claim (A)) and ZFC can prove any true $\Sigma_1^0$-sentence.
	Consequently, if $m \in \N$ is such that $m \in G$ in the standard model, then there is no proof in ZFC of the statement `$m \notin G$' by the assumption that ZFC is consistent (otherwise, there would both be a proof of `$m \in G$' and of its negation`$m \notin G$', contradicting consistency).
	Secondly, note that by one of the standard equivalent definitions of recursively enumerable set, the set $G$ (which is recursively enumerable by Claim (A)) can be enumerated, meaning that there is a procedure $\Gamma^{\en} \colon \N \to G$ that is surjective. On the other hand, since ZFC is an effectively axiomatised theory, there is an algorithm $\Gamma^{\proofsearch} \colon \N \to \{\text{theorems provable in } ZFC\}$ that searches through proofs to enumerate all theorems provable in $ZFC$. We can thus consider the algorithm $\Gamma' \colon \N \to \{1,0\}$ given by the following construction.
	 
	 	\textit{Construction of $\Gamma'$:} on input $m \in \N$, $\Gamma'$ acts as follows: initialise $t =1$, then:
	 	\begin{algosteps}
	 		\item if $\Gamma^{\en}(t) = m$, $\Gamma'$ halts and returns 1; \label{step:enumerate}
	 		\item if $\Gamma^{\proofsearch}(t) =$ \text{`$m \notin G$'}, $\Gamma'$ halts and returns 0; \label{step:search_proof}
	 		\item if neither of these conditions are met then the loop continues by incrementing $t$ and executing the next iteration from \ref{step:enumerate}.
	 	\end{algosteps}
	 	Armed with the algorithm $\Gamma'$, we can prove Claim (B). Assume by contradiction that, for every $m \in \N$ such that $m \notin G$ in the standard model, the sentence `$m \notin G$' is provable in ZFC. Then the algorithm $\Gamma'$ above would decide the set $G$: in fact, on input $m' \in \N$, there are only two possible cases.

	 	\emph{Case ($i$)}: $m' \in G$ in the standard model. Then as observed before there exists no proof within ZFC of `$m' \notin G$' (due to consistency and provability of the true $\Sigma_1^0$-statement `$m' \in G$'). Therefore, $\G'$ cannot halt at \ref{step:search_proof}. On the other hand, since $G$ is enumerable there exists $t \in \N$ such that $\Gamma^{\en}(t)=m'$. Thus $\Gamma'$ halts at \ref{step:enumerate} and returns 1, correctly deciding that $m' \in G$.

	 	\emph{Case ($ii$)}: $m' \notin G$ in the standard model. Then $m'$ will not appear in the enumeration given by $\Gamma^{\en}$, and thus $\Gamma'$ cannot halt at \ref{step:enumerate}. On the other hand, we are assuming for the sake of contradiction that there exists a proof in ZFC of `$m' \notin G$', so that there exists $t \in \N$ such that $\Gamma^{\proofsearch}(t) = `m' \notin G$'. Therefore, $\Gamma'$ halts at \ref{step:search_proof} and returns `no', correctly deciding that $m' \notin G$.

	 	Therefore, the algorithm $\G'$ decides the set $G$ in the standard model. We have thus shown, under the assumption that for every $m \in \N$ such that $m \notin G$ in the standard model the sentence `$m \notin G$'  is provable in ZFC, that there is an algorithm that decides the set $G$ in the standard model. But if such an algorithm really existed, then ZFC would prove the existence of that algorithm because ZFC is $\Sigma_1$-complete. We have already shown via Claim (A) that ZFC proves that no such algorithm exists, and thus we have reached a contradiction by the assumption on the consistency of ZFC. Therefore, there must be $m_0 \in \N$ such that it is \textit{true} that $m_0 \notin G$ in the standard model, but the sentence `$m_0 \notin G$' is \textit{not provable} in ZFC. This concludes the proof of Claim (B).

	 \textbf{Claim (C):} It is not provable that `$m_0 \in G$' within ZFC (assuming that ZFC is $\Sigma_1$-sound).

	 	\emph{Proof of Claim (C):} Consider the sentence `$m_0 \in G$': this is a $\Sigma^0_1$ sentence ($G$ being recursively enumerable by Claim (A)). Since we are assuming that ZFC is $\Sigma_1$-sound, if `$m_0 \in G$' could be proven, then it would be true in the standard model. However, we know from the definition of $m_0$ that `$m_0 \in G$' is false in the standard model. Therefore, the sentence `$m_0 \in G$' cannot be proven. This concludes the proof of Claim (C).

In conclusion, the sentence $\psi \coloneq `m_0 \notin G$' is such that $ZFC$ cannot prove either $\psi$ nor $\neg \psi$, while the interpretation of $\psi$ is true in the standard model of arithmetic.
Let $\stringinput^0\coloneq \{\phi^{m_0}_i\}_{i=1}^k \in \MarkovOmega$ be the code indexed by $m_0$. By Claim (B), we know that $m_0 \notin G$, and from the definition of $G$ it follows that $\stringinput^0$ corresponds to $\iota^0$, i.e. $\CorrXi(\stringinput^0) = \iota^0$. Therefore, from Claims (B) and (C), we have shown that it is \textit{true} that $\stringinput^0$ corresponds to $\iota^0$, but it is \textit{not provable} that $\stringinput^0$ corresponds to $\iota^0$, nor that $\stringinput^0$ does \emph{not} correspond to $\iota^0$.

We now proceed to construct the set $\hat \Omega \subseteq \MarkovOmega$ mentioned in statement \eqref{crp:4} of Theorem \ref{thm:crp_3_4}. For every $\iota \in \Omega \setminus \{\iota^0\} \subseteq \Q^{\nt} \times \Q^{\nt \times \no}$ consider the code $\stringinput^\iota\coloneq (\stringinput^\iota_1,\stringinput^\iota_2, \dotsc, \stringinput^\iota_k) \in \MarkovOmega$ such that its $i$-th component is the constant Turing machine $\stringinput^\iota_i(n) \coloneq f_i(\iota) \in \Q$ for every $n \in \N$ and $i \in \{1,\dotsc, k\}$. Finally, define $\hat \Omega \coloneq \{\stringinput^\iota \ | \ \iota \in \Omega \setminus \{\iota^0\}\} \cup \{\stringinput^0\}$. From Lemma \ref{lem:constructing_phi^m} and Definition \ref{def:Markov_Delta_1_Information} it follows that $\hat \Omega$ is a subset of $\MarkovOmega$.

Now, we proceed to construct the algorithm $\Gamma \colon \hat \Omega \to \ball{\kappa}{\MarkovXi(\MarkovOmega)}$ mentioned in statement \eqref{crp:4} of Theorem \ref{thm:crp_3_4}.

	\textit{Construction of $\Gamma$:} on input $\stringinput = (\stringinput_1,\stringinput_2,\dots,\stringinput_k) \in \hat\Omega$, $\Gamma$ acts as follows:
	\begin{algosteps}
	\itemsep0em
		\item if $\stringinput = \stringinput^0$, then return $\Gamma(\stringinput) \coloneq \frac{1}{2}(y^1+y^2)$ \label{step:hardcoded_4};
		\item if $\stringinput_1(1) > \stringinput_2(1)$, then return $\Gamma(\stringinput) \coloneq y^1$;
		\item if $\stringinput_1(1) < \stringinput_2(1)$ then return $\Gamma(\stringinput) \coloneq y^2$.
	\end{algosteps}
		 Recall that the precise definition of $y^1$ and $y^2$ depends on whether the problem is $\Xi_\LP$, $\Xi_{\BP}$ or $\Xi_{\LASSO}$ as in \S \ref{sec:collection_problems}.
Now that $\hat \Omega$ and $\Gamma$ are defined and using Claims (A) through (C), we begin proving \eqref{crp:4a}, \eqref{crp:4b} and \eqref{crp:4c} from the statement of Theorem \ref{thm:crp_3_4}.

\emph{Proof of \eqref{crp:4a}}: We now prove that, in the standard model, for every $\iota \in \Omega$ there exists exactly one $\stringinput \in \hat \Omega$ corresponding to $\iota$. This follows directly from the definition of $\hat \Omega$ given by 
$
\hat \Omega = \{\stringinput^\iota \ | \ \iota \in \Omega \setminus \{\iota^0\}\} \cup \{\stringinput^0\}.
$
 More explicitly, note that for every $\iota \neq \iota^0$, $\stringinput^\iota$ corresponds to $\iota$ by Definition \ref{def:Markov_Delta_1_Information} (and again, by that same definition, the correspondence is unique) whilst we have already shown that $\stringinput^0$ corresponds to $\iota^0$ in the standard model. This concludes the proof of \eqref{crp:4a}.

\emph{Proof of \eqref{crp:4b}}: We now prove that in the standard model, for all $\stringinput \in \hat \Omega$, the statement $\Gamma(\stringinput) \in \MarkovXi(\stringinput)$ holds. For every $\stringinput \in \hat \Omega$, there are three possible cases.

\emph{Case ($i$):} $\stringinput = \stringinput^0$. In this case, by construction of $\Gamma$, we have $\Gamma(\stringinput) = \frac{1}{2}(y^1 + y^2)$. Since $\stringinput^0$ corresponds to $\iota^0$ in the standard model, we know that $\MarkovXi(\stringinput^0) = \Xi(\iota^0) = S^0$ where $S^0$ is the line segment connecting $y^1$ and $y^2$ as per the setup of \S \ref{sec:collection_problems} and Lemmas \ref{lemma:ProblemBasicExampleLP}, \ref{lemma:ProblemBasicExampleBPDNL1} and \ref{lemma:ProblemBasicExampleULASSO}. Thus $\Gamma(\stringinput) = \frac{1}{2}(y^1 + y^2) \in \Xi(\iota^0)=\MarkovXi(\stringinput)$ as desired.

\emph{Case ($ii$)}: $\stringinput_1(1) > \stringinput_2(1)$. In this case, by construction of $\Gamma$, we have $\Gamma(\stringinput) = y^1$. Moreover, from the definition of $\hat \Omega$, we know that $\stringinput_i(1) = f_i(\iota)$ for every $i \in \{1,\dotsc, k\}$, where $\iota \in \Omega$ is the unique input such that $\stringinput$ corresponds to $\iota$. In particular, from the order of $\Lambda = \{f_i\}_{i =1}^k$ as in \S \ref{sec:collection_problems}, we know that $\iota = (b, U(u_1, u_2))$ with $u_1 > u_2$. Then, from Lemmas \ref{lemma:ProblemBasicExampleLP}, \ref{lemma:ProblemBasicExampleBPDNL1} and \ref{lemma:ProblemBasicExampleULASSO} and the definition of $y^1$ in \S \ref{sec:collection_problems}, we see that $y^1 \in \Xi(\iota)$. Therefore, $\Gamma(\stringinput) = y^1 \in \Xi(\iota)=\MarkovXi(\stringinput)$ as desired.

\emph{Case ($iii$)} $\stringinput_1(1) > \stringinput_2(1)$. In this case, by construction of $\Gamma$, we have $\Gamma(\stringinput) = y^2$. Moreover, from the definition of $\hat \Omega$, we know that $\stringinput_i(1) = f_i(\iota)$ for every $i \in \{1,\dotsc, k\}$, where $\iota \in \Omega$ is the unique input such that $\stringinput$ corresponds to $\iota$. In particular, from the order of $\Lambda = \{f_i\}_{i =1}^k$ as in \S \ref{sec:collection_problems}, we know that $\iota = (b, U(u_1, u_2))$ with $u_1 < u_2$. Then, from Lemmas \ref{lemma:ProblemBasicExampleLP}, \ref{lemma:ProblemBasicExampleBPDNL1} and \ref{lemma:ProblemBasicExampleULASSO} and the definition of $y^2$ in \S \ref{sec:collection_problems}, we see that $y^2 \in \Xi(\iota)$. Therefore, $\Gamma(\stringinput) = y^2 \in \Xi(\iota)=\MarkovXi(\stringinput)$ as desired.

To conclude that ($i$)-($iii$) are all the possible cases, note that if $\stringinput \in \hat \Omega$ is such that $\stringinput \neq \stringinput^0$, then it is not possible that $\stringinput_1(1) = \stringinput_2(1)$. In fact, by the definition of $\hat \Omega$, we have $\stringinput_1(1) = f_1(\iota)$ and $\stringinput_2(1) = f_2(\iota)$, where $\iota$ is the input corresponding to $\stringinput$; and by the definition of $\Omega$ as in \S \ref{sec:collection_problems} we know that $f_1(\iota) = f_2(\iota)$ if and only if $\iota = \iota^0$, and the unique code in $\hat \Omega$ that corresponds to $\iota^0$ is $\stringinput^0$. Thus $\stringinput_1(1) \neq \stringinput_2(1)$ for every $\stringinput \in \hat \Omega\setminus\{\stringinput^0\}$.
Therefore, we have proven that in the standard model, for all $\stringinput \in \hat \Omega$, the statement $\Gamma(\stringinput) \in \MarkovXi(\stringinput)$ holds. This concludes the proof of \eqref{crp:4b}.

\emph{Proof of \eqref{crp:4c}}:We now proceed to show that it is impossible to prove that $\Gamma(\stringinput^0) \in \MarkovXi(\stringinput^0)$ and that $\Gamma(\stringinput^0 \notin \MarkovXi(\stringinput^0)$ within ZFC. We will do so by showing that a proof of either of these statements would also prove that $\stringinput^0$ corresponds to $\iota^0$ or that $\stringinput^0$ does not correspond to $\iota^0$, both of which were shown to be unprovable in Claims (B) and (C). Recall that $m_0 \in \N$ is such that $\stringinput^0 = \{\phi^{m_0}_i\}_{i = 1}^k$.

\textbf{Claim (D)}: ZFC proves that `$\Gamma(\stringinput^0) \in \MarkovXi(\stringinput^0)$' is equivalent to `$m_0 \notin G$'.

\emph{Proof of Claim (D):} From the construction of $\Gamma$, we have that $\G(\stringinput^0) = \frac{1}{2}(y^1 + y^2)$. Thus, `$\Gamma(\stringinput^0) \in \MarkovXi(\stringinput^0)$' is equivalent to `$\frac{1}{2}(y^1 + y^2) \in \MarkovXi(\stringinput^0)$'. Recall from Definition \ref{def:Markov_Delta_1_Information} that the Markov mapping $\MarkovXi$ is defined for every $\stringinput \in \MarkovOmega$ as $\MarkovXi(\stringinput) = \Xi(\iota_\stringinput)$, where $\iota_\stringinput\in \Omega$ is the unique input such that $\stringinput$ corresponds to $\iota_\stringinput$. Therefore, `$\frac{1}{2}(y^1 + y^2) \in \MarkovXi(\stringinput^0)$' is equivalent to `$\stringinput^0$ corresponds to $\iota_{\stringinput^0}$ and $\frac{1}{2}(y^1 + y^2) \in \Xi(\iota_{\stringinput^0})$'. From the definition of $\Omega = \Omega_{\no,\nt}$ as in equation \eqref{eq:set_inputs} and from Lemmas \ref{lemma:ProblemBasicExampleLP}, \ref{lemma:ProblemBasicExampleBPDNL1} and \ref{lemma:ProblemBasicExampleULASSO}, we see that $\frac{1}{2}(y^1 + y^2) \in \Xi(\iota)$ if and only if $\iota=\iota^0$. Therefore, `$\stringinput^0$ corresponds to $\iota_{\stringinput^0}$ and $\frac{1}{2}(y^1 + y^2) \in \Xi(\iota_{\stringinput^0})$' is equivalent to `$\stringinput^0$ corresponds to $\iota^0$'. From the definition of $m_0$ and $G$, we have that `$\stringinput^0$ corresponds to $\iota^0$' is equivalent to `$m_0 \notin G$'. By the transitive property of equivalence, this concludes the proof of Claim (D).

As a consequence of Claim (D), the statement `$\Gamma(\stringinput^0) \in \MarkovXi(\stringinput^0)$' is provable in ZFC if and only if `$m_0 \notin G$' is provable in ZFC; and by considering their negations,`$\Gamma(\stringinput^0) \notin \MarkovXi(\stringinput^0)$' is provable in ZFC if and only if `$m_0 \in G$' is provable in ZFC.

The conclusion of \eqref{crp:4c} now follows easily. In fact, Claim (B) guarantees that $`m_0 \notin G$' is not provable within ZFC, and Claim (C) guarantees that $`m_0 \in G$' is not provable within ZFC. Thus the statements and `$\Gamma(\stringinput^0) \in \MarkovXi(\stringinput^0)$'  and `$\Gamma(\stringinput^0) \notin \MarkovXi(\stringinput^0)$' are not provable in ZFC either. This concludes the proof of \eqref{crp:4c}, thus the proof of \eqref{crp:4}, and therefore of Theorem \ref{thm:crp_3_4}.
\end{proof}

\begin{proof}[Proof of Theorem \ref{thm:crp_5}, CRP \eqref{crp:5}] 
As in the proof of \ref{thm:crp_1_2}, consider the setup of \S \ref{section:setup_CRP} and the solution mappings $\Xi_{\LP}$, $\Xi_{\BP}$ and $\Xi_{\LASSO}$, corresponding to the \eqref{eq:LP}, \eqref{eq:BP} and \eqref{eq:LASSO} problems respectively. Let $\theta \in [1/8,1/4] \cap \Q$ be arbitrary, and consider the input set $\O= \Omega_{\no,\nt}(\theta)$ defined in \eqref{eq:set_inputs}. The proof is almost identical for the mappings $\Xi_{\LP}$, $\Xi_{\BP}$ and $\Xi_{\LASSO}$. Therefore we will use $\Xi$ to denote any of them as appropriate and only mention where this particular choice leads to small differences in the argument. Indeed, let $S^1$, $S^2$, $S^0$, $\{\iota^1_n\}_{n \in \N}$, $\{\iota^2_n\}_{n \in \N}$, $\iota^0$, $y^1$ and $y^2$ be defined as in \S \ref{sec:collection_problems}, depending on the choice of either $\Xi_{\LP}$, $\Xi_{\BP}$ or $\Xi_{\LASSO}$, and recall the definition of $\CorrXi$ from Definition \ref{def:correspondence_problem}.

\emph{Proof of \eqref{crp:5a}}: By Definition \ref{def:Markov_Delta_1_Information}, $\MarkovXi$ is multivalued precisely at those codes $\stringinput$ that correspond to inputs $\iota_\stringinput$ at which $\Xi$ itself is multivalued. From Lemmas \ref{lemma:ProblemBasicExampleLP}, \ref{lemma:ProblemBasicExampleBPDNL1} and \ref{lemma:ProblemBasicExampleULASSO} about the solutions of $\Xi$, and from the construction of the input set $\Omega = \Omega_{\no,\nt}$ as in \eqref{eq:set_inputs}, we see that $\Xi$ is multivalued only at the input $\iota^0$ defined in \eqref{eq:iota}. Since $\Xi^{\mv}$ is the function that takes the value $1$ precisely on those inputs on which $\MarkovXi$ is single-valued and $0$ otherwise, we conclude that $\Xi^{\mv}(\stringinput) = 0$ if and only if $\stringinput$ corresponds to $\iota^0$. This concludes the proof of \eqref{crp:5a}.

In order to prove \eqref{crp:5b} and \eqref{crp:5c}, we will make use of conclusions \eqref{conclusion:iff_1} and \eqref{conclusion:iff_2} respectively of Theorem \ref{thm:iff}. Note that, by \eqref{crp:5a}, we have 
\begin{equation*}
\begin{split}
\Omega^*_1 \coloneq (\Xi^{\mv})^{-1}(1) &= \{\stringinput \in \MarkovOmega \ | \ \MarkovXi \text{ is single-valued at } \stringinput\}\\
& = \{\stringinput \in \MarkovOmega \ | \ \stringinput \text{ does not corresponds to } \iota^0\} = \MarkovOmega \setminus (\CorrXi)^{-1}(\iota^0).
\end{split}
\end{equation*}
To prove \eqref{crp:5b} we will show that both \eqref{condition:sigma1} and \eqref{condition:b2} from Theorem \ref{thm:iff} hold; whereas to prove \eqref{crp:5c}, we will show that \eqref{condition:sigma1} does not hold. From this, we will deduce the desired conclusions.

\emph{Proof of \eqref{crp:5b}}: First, we show that \eqref{condition:sigma1} holds, namely that 
\[
\{\Xi^*, \MarkovOmega, \{0,1\},\MarkovLambda\} \in \Sigma^A_1.
\]
 Explicitly, we proceed to construct a sequence $\{\tilde\G_n\}_{n \in \N}$ of algorithms $\tilde\G_n \colon \MarkovOmega \to \{0,1\}$ such that $\tilde\G_n(\stringinput) \nearrow \Xi^*(\stringinput)$ for every $\stringinput \in \MarkovOmega$ and such that the map $(n,\stringinput) \mapsto \tilde\G_n(\stringinput)$ is recursive. We recall that, by \eqref{crp:5a}, $\Xi^*(\stringinput) = 0$ if and only if $\stringinput$ corresponds $\iota^0$, otherwise $\Xi^*(\stringinput) = 1$.

For every $n \in \N$, consider the algorithm $\tilde\G_n \colon \MarkovOmega \to \{0,1\}$ given by
\begin{align*}
\tilde\G_n(\stringinput) \coloneq 
\begin{cases}
1 & \text{ if there exists $n' \leq n$ such that }  |\stringinput_1(n') - \stringinput_2(n') | > 2 \cdot 2^{-n'}  \\
0 & \text{ otherwise }
\end{cases}
.
\end{align*}
We now verify that $\tilde\G_n \nearrow \Xi^*$. For any arbitrary $\stringinput \in \MarkovOmega$, we distinguish between two cases:

 \emph{Case ($i$)}: If $\stringinput$ corresponds to $\iota^0$, then as noted in conclusion \eqref{conclusion:2n_iota0} of Lemma \ref{lemma:routine} there is no $n \in \N$ such that $|\stringinput_1(n) - \stringinput_2(n) | > 2 \cdot 2^{-n}$; thus $\tilde\G_n(\stringinput) = 0$ for every $n \in \N$. Moreover, $\Xi^*(\stringinput) = 0$ as $\Xi$ is multi-valued on $\iota^0$. Therefore $\lim_{n \to \infty} \tilde\G_n(\stringinput) =0 = \Xi^*(\stringinput)$ and $\tilde\G_n(\stringinput) = 0 = \tilde\G_{n + 1}(\stringinput)$ for every $n \in \N$, as desired.

 \emph{Case ($ii$)}: If $\stringinput$ corresponds to $\iota \neq \iota^0$, set 
 \[
 n_{\min} \coloneq \min\{n \in \N \ | \  |\stringinput_1(n) - \stringinput_2(n) | > 2 \cdot 2^{-n}\} \in \N;
 \] by conclusion \eqref{conclusion:2n_not_iota0} of Lemma \ref{lemma:routine}, such a minimum exists. Then by construction,  $\tilde\G_n(\stringinput) = 0$ for every $n < n_{\min}$ and $\tilde\G_n(\stringinput)=1$ for every $n \geq n_{\min}$. Moreover, $\Xi^{\mv}(\stringinput) = 1$ as $\Xi$ is single-valued on $\iota$. Therefore $\lim_{n \to \infty} \tilde\G_n(\stringinput) = 1 = \Xi^*(\stringinput)$ and $\tilde\G_n(\stringinput) \leq \tilde\G_{n+1}(\stringinput) \leq 1 = \Xi^*(\stringinput)$ for every $n \in \N$, as desired.
This concludes the verification that $\tilde\G_n \nearrow \Xi^*$. Furthermore, it is easy to see that the map $(n,\stringinput) \mapsto \tilde\G_n(\stringinput)$ is recursive. So $\{\Xi^*, \MarkovOmega, \{0,1\},\MarkovLambda\} \in \Sigma^A_1$ and thus we have verified that condition \eqref{condition:sigma1} of Theorem \ref{thm:iff} holds.
Secondly, we prove that condition \eqref{condition:b2} of Theorem \ref{thm:iff} also holds. Consider the algorithm $\Gamma^{\always} \colon \Omega^*_1 \to \M$, where $\Gamma^{\always}$ is the routine of Lemma \ref{lemma:routine}. Since $\Omega^*_1 \coloneq (\Xi^{*})^{-1}(1)$ we see that $\Omega^*_1$ is the set of those $\stringinput$ that do not correspond to $\iota^0$. By conclusion \ref{conclusion:correct} of Lemma \ref{lemma:routine} we deduce that $\Gamma^{\always}(\stringinput) \in \ball{\kappa}{\MarkovXi(\stringinput)}$ for every $\stringinput \in \Omega^*_1$. Thus we have verified condition \eqref{condition:b2}.

Since both \eqref{condition:sigma1} and \eqref{condition:b2} of Theorem \ref{thm:iff} hold, by statement \eqref{conclusion:iff_2} of Theorem \ref{thm:iff}, there exists a $\kappa$-trustworthy AI with `giving up' parameter $\{\G_n\}_{n \in \N}$ (where $\G_n \colon \Omega \to \M \cup \{\idk\}$ for every $n \in \N$) such that $\Xi^* = \Xi^{\Idk}_{\{\G_n\}}$. This gives the first part of the desired statement \eqref{crp:5b}.

\emph{Proof of \eqref{crp:5c}}: We use the equivalence guaranteed by \eqref{conclusion:iff_1} of Theorem \ref{thm:iff}. We now show that \eqref{condition:delta0} does not hold, namely that 
$
\{\Xi^*, \MarkovOmega, \{0,1\},\MarkovLambda\} \notin \Delta^A_0.
$
Suppose, by contradiction, that $\{\Xi^*, \MarkovOmega, \{0,1\},\MarkovLambda\} \in \Delta^A_0$. Then there exists an algorithm $\Gamma^{0} \colon \MarkovOmega \to \{0,1\}$ such that $\Gamma^{0}(\stringinput) = \Xi^*(\stringinput)$ for every $\stringinput \in \MarkovOmega$. In particular, by \eqref{crp:5a} and the fact that $\Xi^{\mv}$ is the indicator function of those codes on which $\MarkovXi$ is single-valued, we deduce that $\Gamma^{0}(\stringinput) = 0$ if and only if $\stringinput$ corresponds to $\iota^0$, and $\Gamma^{0}(\stringinput) = 1$ otherwise.

Recall the routine $\Gamma^{\always}$ from Lemma \ref{lemma:routine}. Construct the following algorithm $\Gamma^{\text{s}} \colon \MarkovOmega \to \M$:
\begin{align*}
\Gamma^{\text{s}}(\stringinput) \coloneq 
\begin{cases}
\Gamma^{\always}(\stringinput) & \text{ if } \Gamma^{0}(\stringinput) = 1\\
y^1 & \text{ if } \Gamma^{0}(\stringinput) = 0,
\end{cases}
\end{align*}
where $y^1$ is given as in either \eqref{eq:y^j_lp}, \eqref{eq:y^j_bp} or \eqref{eq:y^j_lasso} depending on the solution map $\Xi_{\LP}$, $\Xi_{\BP}$ or $\Xi_{\LASSO}$ respectively. Recall that by \S \ref{sec:collection_problems} (and in particular, Lemma \ref{lemma:ProblemBasicExampleLP}, Lemma \ref{lemma:ProblemBasicExampleBPDNL1} or Lemma \ref{lemma:ProblemBasicExampleULASSO} depending on the solution map) that $y^1 \in \Xi(\iota^0)$.

We claim that the algorithm $\Gamma^{\text{s}}$ solves the computational problem $\{\MarkovXi,\MarkovOmega,\M,\MarkovLambda\}$. In fact, if $\stringinput$ corresponds to $\iota^0$, then $\Gamma^{0}(\stringinput) = \Xi^*(\stringinput) = 0$ and thus $\Gamma^{\text{s}}(\stringinput) = y^1 \in \Xi(\iota^0) = \MarkovXi(\stringinput)$. On the other hand, if $\stringinput$ does not correspond to $\iota^0$, then $\Gamma^{0}(\stringinput) = \Xi^*(\stringinput) = 1$ and thus $\Gamma^{\text{s}}(\stringinput) = \Gamma^{\always}(\stringinput) \in \ball{\kappa}{\MarkovXi(\stringinput)}$ by conclusion \ref{conclusion:correct} of Lemma \ref{lemma:routine}. 
Thus, $\Gamma^{\text{s}}$ solves the computational problem $\{\MarkovXi,\MarkovOmega,\M,\MarkovLambda\}$. However, this contradicts conclusion \eqref{crp:2} of Theorem \ref{thm:crp_1_2}. Therefore, there does not exist any such algorithm $\Gamma^{0}$, and 
\[
\{\Xi^*, \MarkovOmega, \{0,1\},\MarkovLambda\} \notin \Delta^A_0.
\] This proves that condition \eqref{condition:delta0} of Theorem \ref{thm:iff} does not hold.
From the equivalence guaranteed by statement \eqref{conclusion:iff_1} of Theorem \ref{thm:iff}, it follows that there does not exists any $\kappa$-trustworthy AI of the form $\G\colon \Omega \to \M \cup \{\idk\}$ such that $\Xi^* = \Xi^{\Idk}_{\G}$. We have thus proven the statement of \eqref{crp:5c}. Therefore, the proof of CRP \eqref{crp:5}, and thus of Theorem \ref{thm:crp_5}, is complete.
\end{proof}

\section{Quantifying the CRP -- Constructing specific failure sentences for AI}\label{sec:example}

In this section, we provide concrete examples for which the CRP occurs. Indeed, as mentioned in Remark \ref{rmk:quantification}, as well as in \S \emph{``The Consistent Reasoning Paradox (CRP) - A stronger CRP II: Failure sentences and equivalence''} on page 4 of the main part of the paper and in \S \emph{``Stronger statements – Quantifying the CRP''} on page 7 of the Methods section, our proof techniques allow us to provide explicit examples of inputs on which any AGI will fail. We now provide a collection of such examples - out of the infinitely many - that induce the failure mentioned in CRP II. Moreover, we quantify the lengths of such inputs (in terms of the number of characters) to show that the questions for which the CRP applies are not exotic, or abstract examples of potentially unbounded length. In particular, these failure sentences for the AGI occur by creating a sentence that is only a small number of additional characters plus the length of the code of the AGI itself. The provided codes are not intended to be the optimal minimum length possible, but rather they are designed to showcase the techniques required to induce the failure mentioned in the CRP in a clear way.

\begin{remark}[Language-specific examples]
	In this section, we provide explicit \emph{language-specific} examples of inputs that will make any AGI fail, meaning that they depend on an arbitrary choice of a programming language in which the codes are written. For reasons that will be further elaborated later, we provide examples of codes written in MATLAB, but they could also be implemented in one's favourite programming language - such as C or Python.
\end{remark}

\textbf{Assumption on the AGI:} We now briefly discuss the type of AGI we will be considering. Similarly to modern chatbots, such AGI takes as input a string describing a problem, and returns a candidate solution to the problem under consideration. Explicitly, denoting by $\mathcal{A}$ the alphabet of Unicode characters and by $\mathcal{A}^*$ the set of strings of such alphabet, the AGI will be defined on a collection $\mathcal{C} \subseteq \mathcal{A}^*$  of strings that describe `well formed' questions; we will not specify the boundaries of such a class $\mathcal{C}$ -- as such a task could be open to debates --, but simply make the minimal assumption that it contains at least the family of strings 
\begin{equation}\label{eq:C_LP}
	\mathcal{C}_{\LP} \coloneq \{\stringinput(d,m) \ | \ d \in \N, \ m \text{ MATLAB code})\} \subseteq \mathcal{C}\subseteq \mathcal{A}^*
\end{equation}
where for every $d \in \N$ and $m$ code for a MATLAB function, $\stringinput(d,m)$ is defined as follows. The string $\stringinput(d,m)$ is specified by inserting a specific numerical value for $d$ in place of `{[insert numerical value of $d$]}', by providing a MATLAB code $m$ in place of `{[insert MATLAB  code $m$]}', and by including all the dependencies of the code $m$ in place of `[insert codes of the dependencies of the code $m$]' in the following string :

\begin{displayquote}\label{questions}
	$\stringinput(d,m) \coloneq $ \enquote{\it 
		Consider the integer dimension $d=$ \emph{[insert numerical value of $d$]}. Provide an approximation within $\kappa = 10^{-1}$ in the $\| \cdot \|_\infty$ norm of a minimiser of the following Linear Programming problem
		\[
		\min_{x \in \R^d} \ \langle x , c \rangle,  \ \text{ such that } \ Ax = y, \ x \geq 0 	\tag{LP}
		\]
		where $c = \ones_{d}\in \Q^{d}$ is the $d$-dimensional vector with $1$ in each entry, $y = 2\kappa = 2 \cdot 10^{-1}$, and the input $A \in \Q^{1 \times d}$ is given as follows. 
	
			The input $A \in \Q^{1 \times d}$ is given via the following computer code: calling the MATLAB code \emph{[insert MATLAB code $m$]} with parameter $(n,i,d)$ will give the $i$th coordinate of $A$ to accuracy $2^{-n}$ , provided $i \leq d$. 
			You are also given access to all the dependencies for the previous MATLAB code: [insert dependencies for the code $m$].
			}.
\end{displayquote}

In a similar way to the notion of correspondence as in Definition \ref{def:Markov_Delta_1_Information}, we say that the question $\Phi(d,m)$ \emph{corresponds to} $\iota = (2 \cdot 10^{-1},A) \in \Q \times \Q^{1 \times d}$ if calling the MATLAB code $m$ on input $(n,i,d)$ returns a $2^{-n}$ approximation to the $i$th coordinate of $A \in \Q^{1 \times d}$.

The AGI under consideration is thus a function $\Gamma \colon \mathcal{C} \to \bigcup_{d \in \N} \Q^d \cup \M'$, where $\mathcal{C} \supseteq \mathcal{C}_{\LP}$ and $\M'$ is the set of candidate solutions to the problems in $\mathcal{C}\setminus \mathcal{C}_{\LP}$. We will mostly focus on the action of the AGI simply on the class $\mathcal{C}_{\LP}$, thus we shall only consider $\Gamma|_{\mathcal{C}_{\LP}} \colon \mathcal{C}_{\LP} \to \bigcup_{d \in \N} \Q^d$.

\begin{remark}[Dependencies]
There are two types of functions that a MATLAB code can call during its execution: either MATLAB's native functions, or alternatively non-built-in-functionality. A \textit{dependency} is a non-built-in function, script, or file that a MATLAB routine requires for its execution and which is not part of the MATLAB core library. Dependencies may include user-defined functions and custom scripts, and they typically must be accessible in the MATLAB path for the routine to run successfully.
\end{remark}

\begin{remark}[The choice of language]
	Our code operates in MATLAB. This language was chosen due to its simple parallel functionality (including the ability to recursively start new parallel operations and the ability to easily terminate running parallel processes) through the Parallel Computing Toolbox and its rich integration with other languages such as Python, C and C++. Of course, this means our approach applies to AGIs written in any of these languages; the particular language used for the AGI is unimportant, so long as the code describing the AGI is given. Moreover, it is easy to see that our approach would generalise to any modern language that implements multiprocessing, timing and thread termination.
\end{remark}

\begin{remark}[Defensive techniques and input validation]
	We work under the assumption that the AGI accepts and returns inputs belonging to the class $\mathcal{C}$, that are assumed to be valid (of the correct type and syntax). Therefore, there will be no need to employ input validation and defensive programming techniques, such as fail-safe defaults and exception handling. A secondary motivation for this approach is a desire for concision and clarity - input validation is not an essential part of the argument and our goal is to showcase the type of questions relevant to CRP II. It should be noted however that implementing such fail-safe mechanism could be quickly accomplished with few extra lines of code. 
\end{remark}

The following theorem illustrates the concrete examples of failure sentences for an AGI as mentioned in CRP II. Explicitly, for any integer $K \geq 2$, we create $K$ questions. The $d$th such question tasks the AGI with solving the linear program with input $A$ so that the $i$th coordinate of $A \in \Q^{1 \times d}$ is generated by a computer code and so that $y \in \Q$ is given by $y= 2\kappa = 2 \times 10^{-1}$.  We state our result only for the computational problem of Linear Programming as in \eqref{eq:LP} with one equality constraint, and for the choice of MATLAB as programming language, but this is simply done for ease of presentation: the same argument can easily be applied to other computational problems (such as Basis Pursuit and Lasso as in \eqref{eq:BP} and \eqref{eq:LASSO}) and other choices of programming language.

As a significant feature, our theorem \emph{quantifies} the Consistent Reasoning Paradox in the following sense: it provides an upper bound on the length of failure sentences for an AGI, as mentioned in the statement of CRP II. We adapt the following notation: we denote by $\operatorname{len}(\Phi)$ the length of a string $\stringinput \in \mathcal{A}^*$, defined as the total number of Unicode characters that compose it. Furthermore, given an AGI $\Gamma$ whose code is expressed as a string in the MATLAB programming language and contained in a file `\emph{AGI.m}' (which can call any auxiliary files that may be written in other languages, such as Python or C), we denote by $\operatorname{len}(\Gamma)$ the sum of the lengths of the code \emph{AGI.m} and of its external dependencies.

\begin{theorem}[Quantifying CRP II for Linear Programming with MATLAB codes]\label{thm:quantifyinf_crp}
	Let $\mathcal{A}$ be the Unicode alphabet and $\mathcal{C}_{\LP} \subseteq \mathcal{A}^*$ be the collection of questions as in \eqref{eq:C_LP}. Consider any AGI, defined on a set of strings $\mathcal{C} \subseteq \mathcal{A}^*$, that takes as input a string describing a problem and returns a candidate solution to such problem. Assume that $\mathcal{C} \supseteq \mathcal{C}_{\LP}$ and denote the restriction of the AGI to the collection $\mathcal{C}_{\LP}$ by $\Gamma \colon \mathcal{C}_{\LP}\to \bigcup_{d \in \N} \Q^d$. Moreover, assume that the code for the AGI is expressed as a string in the MATLAB programming language -- equipped with Parallel Computing Toolbox -- and is contained in a file `\emph{AGI.m}' which can call any auxiliary files that may be written in other languages, such as Python or C.
	
	Then for every integer $K \in \N$ there exist $K$ questions $\stringinput^1,\dotsc, \stringinput^K \in \mathcal{C}_{\LP}$, each describing a Linear Programming problem such that 
	\begin{align}\label{eq:agi_wrong}
		\Gamma(\stringinput^l) \text{ is not a correct solution to the problem described by } \stringinput^l,
	\end{align}
	and
$
		\operatorname{len}(\stringinput^l) \leq \operatorname{len}(\Gamma) + \log_{10}(K) + \epsilon
$
	for every $l= 1 \dotsc, K$, where $\epsilon = \epsilon_{\text{MATLAB}} \leq 3300$ characters.
\end{theorem}

\begin{remark}
	[Language-specific proof]\label{rmk:language_specific}
	The proof of Theorem \ref{thm:quantifyinf_crp} is an adaptation of the language-free proof of CRP II (more precisely, the proof of Proposition \ref{prop:DrivingNegativeProposition}) to the specific case of the MATLAB language. The original proof was language-free, and thus Turing machines were accessed via their G\"odel numbers and procedures were considered to be defined on such numbers. On the other hand, the following proof will be specific to the (arbitrarily chosen) programming language MATLAB, and will consider routines that access the MATLAB codes directly, rather then their G\"odel numbers. It is clear that, by choosing another language instead of MATLAB, the proof of Theorem \ref{thm:quantifyinf_crp} would remain mostly unchanged and the only difference would be the precise value of the constant $\epsilon$.
\end{remark}

\begin{proof}

	The proof will be structured in the following way: first, we present the question on which the AGI is guaranteed to fail, followed by the MATLAB codes that are supplied with the question. Interlaced with the codes will be explanations describing the algorithms and functions that these MATLAB codes are computing. Finally, we prove why such questions and codes are guaranteed to make the AGI fail, which is very closely related to the proof of CRP II and Theorem \ref{thm:crp_1_2}.
	
	Fix an integer $K \in \N$ and $l \in \{1,\dotsc, K\}$. Let $d \coloneq l+1$, and  denote by ${d_k d_{k-1} \dotsc d_2 d_1}$ the expansion of $d$ in base $10$. Note that $\operatorname{len}({d_k d_{k-1} \dotsc d_2 d_1}) \leq \log_{10}(d) +1 \leq \log_{10}(K+1) + 1 \leq \log_{10}(K) + 2$. 
	
	The failure sentence $\stringinput^l$ for the AGI is the following:
	\begin{displayquote}
		\textbf{Question $\Phi^l$}: \enquote{\it 
			Consider the integer dimension $d={d_k d_{k-1} \dotsc d_2 d_1}$. Provide an approximation within $\kappa = 10^{-1}$ in the $\| \cdot \|_\infty$ norm of a minimiser of the following Linear Programming problem
			\[
			\min_{x \in \R^d} \ \langle x , c \rangle,  \ \text{ such that } \ Ax = y, \ x \geq 0
			\]
			where $c = \ones_{d}\in \Q^{d}$ is the $d$-dimensional vector with $1$ in each entry, $y = 2\kappa = 2 \cdot 10^{-1}$, and the input $A \in \Q^{1 \times d}$ is given as follows.  
		
			\noindent	The input $A \in \Q^{1 \times d}$ is given via the following computer code: calling the MATLAB code
					
					@(maxTime, coordinateIndex, dimension) phi(`gammaFunc', maxTime, coordinateIndex, dimension)
					
				\noindent with parameter $(n,i,d)$ will give the $i$th coordinate of $A$ to accuracy $2^{-n}$ , provided $i \leq d$. 
				You are also given access to all the dependencies for the previous MATLAB code: \emph{gammaApprox.m}, \emph{phi.m}, \emph{gammaFunc.m}, and \emph{AGI.m}.}.
	\end{displayquote}
		
	Note that the failure sentence $\Phi^l$ is induced by the MATLAB code $\code$ given by:
	
\begin{lstlisting}[style = StyleA, language=Matlab, caption={Code $\code$, inducing the failure sentence $\Phi^l = \Phi(d,\code)$.}]
@(maxTime, coordinateIndex, dimension) phi(`gammaFunc', maxTime, coordinateIndex, dimension)
\end{lstlisting}
	
	and thus $\Phi^l = \Phi(d, \code)\in \mathcal{C}_{\LP}$ where $d = l+1$. This choice of $d$ guarantees that the number of columns of the matrix $A$ is at least $2$, so that Lemma \ref{lemma:ProblemBasicExampleLP} applies.
	
	We will soon proceed to list the codes for the dependencies of $\code$, namely \emph{gammaApprox.m}, \emph{phi.m}, \emph{gammaFunc.m} and \emph{AGI.m} mentioned in the question above. Such codes are presented in an order so that a function is called only if it has been previously defined. The reason why the codes are split into different \textit{.m} files is the way MATLAB handles functions, requiring each function to be passed as a separate file, with the exception of gammaApprox.m which we define separately to aid clarity.

	Before proceeding to the codes, we illustrate the strategy leading to their design. Recall that the AGI is defined on (at least) the collection $\mathcal{C}_{\LP}$ of strings that contains the questions $\stringinput(d,m)$ requesting to solve a Linear Programming problem \eqref{eq:LP} with fixed codomain dimension $\nt=1$ and arbitrary domain dimension $\no=d \in \N$, for a specific input. The input to such problem is of the form $\iota = (2\cdot 10^{-1}, A) \in \Q \times \Q^{1\times d}$, and the coordinates of $A \in \Q^{1 \times d}$ are determined via the MATLAB code $m$ parametrising the question $\stringinput(d,m)$. In particular, provided that $d \geq 2$, it is possible to design MATLAB codes $m$ (together with their dependencies) that correspond to the inputs $\iota^0, \iota^j_n \in \Omega_{1,d}$ (for $j \in \{1,2\}$ and $n \in \N$) as defined in \S \ref{sec:collection_problems} for the Linear Programming case with the intention of employing an argument akin to that used in Proposition \ref{prop:DrivingNegativeProposition}. This is precisely how the following MATLAB codes are designed: they implement certain auxiliary functions defined in \S \ref{section:proof_4}, namely \emph{gammaApprox.m} (which implements $\hat \Gamma$ as constructed in the proof of Lemma \ref{lemma:assumptions_verification}), \emph{phi.m} (which implements $\{\phi^m_i\}_{i=1}^d$ for a given string $m$, as constructed in Lemma \ref{lem:constructing_phi^m}) and \emph{gammaFunc.m} (which implements the function $\gamma$ as in equation \eqref{eq:gamma}).
		
		We now proceed to introduce and explain each of the above codes.
	
	\textbf{Step (I): The code for AGI.m.} Firstly, the code \emph{AGI.m} is the source code for the AGI itself, and thus depends on the given AGI. Recall that we assume that the code for the AGI is contained in a file `AGI.m' which can call any auxiliary files that may be written in e.g. Python or C (these files are also supplied to the AGI as part of the question above).
\begin{lstlisting}[style = StyleA, language=Matlab, caption=Code for the AGI $\Gamma$.]
function [output] = AGI(string)
		% Code for the AGI, defined on strings that describe Linear Programming problems, and returning a candidate solution
		...
end
\end{lstlisting}

	\textbf{Step (II): The code for gammaApprox.m}. We present the code \emph{gammaApprox.m}. This code computes the algorithms $\hat{\Gamma}^j_i$ and $\hat\Gamma^0_i$  mentioned in the proof of Lemma \ref{lemma:assumptions_verification}, which we recall here: these are $\hat \Gamma^{j}_i: \mathbb{N} \times \mathbb{N} \to \Q$ with $\hat \Gamma^{j}_i(n,r) = f_i(\iota^{j}_r)$, and $\hat \Gamma^{0}_i: \mathbb{N} \to \Q$ with $\hat \Gamma^{0}_i(n) = f_i(\iota^0)$, for $i \in \{1,2,\dotsc,d\}$, $j \in \{1,2\}$, and $r \in \N$.
		Recall that $f_{i} \colon \Omega_{1,d} \to \Q$ is the evaluation function $f_i \in \Lambda$ that on input $\iota \in \Omega_{1,d}$ return its $i$th coordinate $f_i(\iota) \in \Q$.
	
\begin{lstlisting}[style = styleA, language=Matlab, caption=Code for $\hat \Gamma$.]
function [oValue] = gammaApprox(j, n, t, coordinateIndex)
		% Returns the coordinateIndex-th entry of the input iota^j_t
		% (Note that t is irrelevant when j==0)
		a1 = 1/2;
		a2 = 1/2;
		if j == 1
				a2 = a2 - 4^(-t);
		elseif j == 2
				a1 = a1 - 4^(-t);
		end
		
		if coordinateIndex == 1
				oValue = a1;
		elseif coordinateIndex == 2
				oValue = a2;
		elseif coordinateIndex > 2
				oValue = 0;
		end
end
\end{lstlisting}
	
It is clear from its construction that the code for $\hat\Gamma$, on input $(j,n,t,i)$ (where $i = $  \verb*|coordinateIndex|) returns the $i$th coordinate of $\iota^j_t$ if $j \neq 0$, or of $\iota^0$ if $j = 0$, as desired.
These outputs are precisely $f_i(\iota^j_t) = \hat\Gamma^{j}_i(n,t)$ and $f_i(\iota^0) = \hat\Gamma^0_i(n)$ respectively, as desired.
	 
	 Note that, if $j = 0$, the parameter $t$ is irrelevant.
	
	\textbf{Step (III): The code for phi.m.} We now proceed to present the code \emph{phi.m}, which computes the function $\phi^m$ defined in equation \eqref{eq:def_phiDerived} of Lemma \ref{lem:constructing_phi^m}. We recall that, for every coordinate index $i \in \{1,\dots, d\}$, the function $\phi^m_i \colon \N \to \Q$ is defined by:
	\begin{equation*}
		\phi^m_i(n):= \begin{cases} \hat\Gamma^1_i(n,t) & \text{ if } [m \in \mathcal{W}(m,t)] \land [t \leq n] \land [\varphi_m(m) = 1];\\ 
			\hat\Gamma^2_i(n,t) & \text{ if } [m \in \mathcal{W}(m,t)] \land [t \leq n] \land [\varphi_m(m) = 2];\\
			\hat \Gamma^0_i(n) & \text{ otherwise.}
		\end{cases}
	\end{equation*}
	We recall again that the parameter $m$ used to represent an \textit{integer} in Lemma \ref{lem:constructing_phi^m}, namely the G\"odel number of a Turing machine; whereas in the current context, $m$ is the \textit{string} that encodes a MATLAB routine. We implement $\phi$ as follows, with an explanation of why this represents $\phi$ described below:

\begin{lstlisting}[style = styleA, language=Matlab, caption=Code for $\{\phi^{m}_i\}_{i = 1}^d$.]
function [approxCoordinateValue] = phi(fileName, maxTime, coordinateIndex, dimension)
		% Computes phi^m_i(n), which is a 2^(-n-1) approximation to the i-th coordinate of an input iota determined by the code m
		% In particular, m = fileName is the file name of a code, i = coordinateIndex is an integer between 1 and d = dimension, and n = MaxTime gives the accuracy error 2^(-n) on the input coordinate
		phiStatus = parfeval(str2func(fileName), 1, dimension, fileName);
		wait(phiStatus, "finished", maxTime);
		
		% If running m on itself hasn't halted within n steps, set j = 0
		if ~strcmp(phiStatus.State, "finished") || ~isempty(phiStatus.Error)
			cancel(phiStatus);
			pathway = 0;
		else	% Otherwise, m has halted within n steps, so assign j to either 0, 1, or 2
			outputValue = fetchOutputs(phiStatus);
			if outputValue == 1 || outputValue == 2
					pathway = outputValue;
			else
					pathway = 0;
			end
			% Find the time t it takes for code m to halt on itself (surely t is at most n)
			timeTaken = seconds(phiStatus.RunningDuration);	
		end
		if pathway == 0
			approxCoordinateValue = gammaApprox(0, maxTime, 0, coordinateIndex);	
			% Recall that if j == 0, then t is superfluous
		else
			approxCoordinateValue = gammaApprox(pathway, maxTime, timeTaken, coordinateIndex);
		end
end
\end{lstlisting}
	
	The code for $\phi$ implements the function $\phi^m_i(n)$ on every input $(m,n,i)$ where (for the sake of writing clear code) $m = $ \verb*|fileName| is a string, $n = $ \verb*|maxTime| is a natural number and $i = $ \verb*|coordinateIndex| is between $1$ and $d =$ \verb*|dimension|. 
	
	Recall from the definition of $\phi^m_i$ as in Lemma \ref{lem:constructing_phi^m} that $\phi^m_i(n)$ runs machine $m$ on input $m$ for a certain number of steps (or amount of time) indicated by $n$: if the procedure has halted before time $n$ (more precisely, at time $t = $  \verb*|timeTaken|  $\leq n$) and has returned a value $j = $ \verb*|outputValue| $ \in \{0,1\}$, then $\phi^m_i(n)$ correspondingly returns $\hat\Gamma(j,n,t,i)$; in all other cases -- explicitly, if either machine $m$ on input $m$ has not yet halted before time $n$, or if it has halted within $n$ seconds but returned an output not in $\{0,1\}$ -- then $\phi^m_i(n)$ returns $\hat\Gamma(0,n,0,i)$.

	An extra word needs to be spent on how $\phi$ implements `running machine $m$ on input $m$', and checking if this procedure `has halted by step $n$'. In the definition of $\phi^m_i(n)$ as in equation \eqref{eq:def_phiDerived}, these two properties are stated respectively by checking the value of $\varphi_m(m)$, and by assessing the condition $[m \in \mathcal{W}(m,t)] \land [t \leq n]$, where the set $\mathcal{W}$ is defined precisely by checking how many steps it takes for an algorithm to halt (see \S \ref{sec:constructing} for a precise definition). However, as noted in Remark \ref{rmk:properties_of_W}, any function $\mathcal{W}'$  satisfying properties \ref{item:W_property1} and \ref{item:W_property2} can be used instead of $\mathcal{W}$ to complete the argument. In particular, whilst it was useful to consider the \textit{number of steps} of a Turing machine for the theoretical results of \S \ref{section:proof_4}, here, we instead rely on the \textit{amount of time} taken, which is a concept much more relevant to practical implementations of the CRP. Primarily for that reason the execution $\varphi_m(m)$ for a code $m$ is implemented via parallel computing as follows:
	
\begin{lstlisting}[style = styleB, language=Matlab]
phiStatus = parfeval(str2func(fileName), 1, dimension, fileName);
\end{lstlisting}
	In this line, the function  \verb*|parfeval| runs the function \verb*|str2func(fileName)| (which returns one output: hence the value \verb*|1| in the call to \verb*|parfeval|) on input $m =$  \verb*|fileName| and dimension $d = $ \verb*|dimension|, running this computation in a parallel thread. The reason for the usage of parallel computation is principally driven by the need to avoid the main thread stalling; this is particularly important in the case where  \emph{phi.m} does not terminate while running $m$ on $m$ itself. We can use the command
\begin{lstlisting}[style = styleB, language=Matlab]
wait(phiStatus, "finished", maxTime);
\end{lstlisting}
	to ensure that the computation finishes within a maximum of \verb*|maxTime| $=n$ seconds. The code of \verb*|phi| continues, either because of some failure (e.g. \verb*|maxTime| seconds elapsed or an error was thrown) or because the execution of $\varphi_m(m)$ completed. We can check if the former occurred with the following segment of code
\begin{lstlisting}[style = styleB, language=Matlab]
if ~strcmp(phiStatus.State, "finished") || ~isempty(phiStatus.Error)
		cancel(phiStatus);
pathway = 0;
\end{lstlisting}
	in which the parallel thread is cancelled and we set $j = 0$.
	
	Ultimately, if the thread completes we can measure the time taken using the following command
\begin{lstlisting}[style = styleB, language=Matlab]
timeTaken = seconds(phiStatus.RunningDuration);	
\end{lstlisting}
	The lines above, therefore, implement the complement of the condition $[m \in \mathcal{W}(m,t)] \land [t \leq n]$.
	
	The rest of the code for \emph{phi.m} is a straightforward interpretation of the definition of the function $\phi$ as in Lemma \ref{lem:constructing_phi^m}.
	
	\textbf{Step (IV): The code for gammaFunc.m.} We now present the code \emph{gammaFunc.m}, which implements the function
	\begin{align}\label{eq:gamma_with_codes}
		\gamma_d(m)\coloneq 
		\begin{cases}
			1 & \text{ if } \Gamma(\Phi(d,\phi^{m})) \in \ball{\kappa}{S^2}; \\
			2 & \text{ otherwise}.
		\end{cases}
	\end{align}
	for any $d \in \N$ and $m$ name of a MATLAB function.

		The function $\gamma_d$ is a straightforward adaptation of the language-free function $\gamma$ defined in equation \eqref{eq:gamma} to the current language-specific case of MATLAB and for the LP problem in dimension $d$. 
		
	We recall again that $m$, in the current context, does not represent an integer but rather the name of a MATLAB function. 
	
\begin{lstlisting}[style = styleA, language=Matlab, caption=Code for $\gamma_d$.]
function [gammaAnswer, LPAnswer] = gammaFunc(dimension, fileName)
		codeForAI = cat(2, 'Consider the integer dimension $d= ', num2str(dimension), '$. Provide an approximation within $\kappa = 10^{-1}$ in the $\| \cdot \|_\infty$ norm of a minimiser of the following Linear Programming problem
		\[
		\min_{x \in \R^d} \ \langle x , c \rangle,  \ \text{ such that } \ Ax = y, \ x \geq 0
		\]
		where $c = \ones_{d}\in \Q^{d}$ is the $d$-dimensional vector with $1$ in each entry, $y = 2\kappa = 2 \cdot 10^{-1}$, and the input $A \in \Q^{1 \times d}$ is given as follows.  
		The input $A \in \Q^{1 \times d}$ is given via the following computer code: calling the MATLAB code @(maxTime,coordinateIndex,dimension) phi(''', fileName,''',maxTime,coordinateIndex,dimension) with parameters $(n,i,d)$ will give the $i$th coordinate of $A$ to accuracy $2^{-n}$ , provided $i \leq d$. 
		You are also given access to all the dipendencies for the previous MATLAB code: gammaApprox.m, phi.m,  ', fileName,'.m, and AGI.m.');
		LPAnswer = AGI(codeForAI);
		insideBoolean = inBallS2(LPAnswer, 0.1);
		if insideBoolean
				gammaAnswer = 1;
		else
				gammaAnswer = 2;
		end
end

function insideBoolean = inBallS2(x, kappa)
		compVec = zeros(size(x));
		compVec(2) = 4 * kappa;
		insideBoolean = max(abs(x - compVec)) < kappa;
end
\end{lstlisting}

	The code for $\gamma_d$ is a straightforward implementation of the function $\gamma_d$ defined in \eqref{eq:gamma_with_codes} with the ball $\ball{\kappa}{S^2}$ assumed to be in the $\| \cdot \|_\infty$ norm. Recall from \S \ref{sec:collection_problems} that in the case of Linear Programming, we have $S^2 = \{4\kappa e_2\} \subseteq \Q^d$ and that $\kappa = 10^{-1}$. On input $m$, in order to run the equivalent of $ \Gamma(\Phi(d,\phi^{m}))$, the code \emph{gammaFunc} contains the lines (shortened here for sake of brevity):
\begin{lstlisting}[style = styleB, language=Matlab]
codeForAI = cat(2, 'Consider the integer dimension' ... ', and AGI.m.');
LPAnswer = AGI(codeForAI);
\end{lstlisting}
	which passes the string that describes the Linear Programming question $\stringinput(d,m)$ as an input to the AGI.

	\textbf{Step (V): The code $\code$.} After having introduced all of its dependencies, we finally analyse the MATLAB code $\code$. Recall that it is given by the single line:

\begin{lstlisting}[style = StyleA, language=Matlab, caption={Code $\code$, inducing the failure sentence $\Phi^l = \Phi(d,\code)$ where $d = l+1$.}]
@(maxTime, coordinateIndex, dimension) phi(`gammaFunc', maxTime, coordinateIndex, dimension)
\end{lstlisting}

	The code for $\code$ is an anonymous function - as introduced by the symbol `@' - that on input (\text{maxTime}, \text{coordinateIndex}, \text{dimension}) calls the function \text{phi} on such inputs and with string parameter fixed to $m =$ \emph{`gammaFunc'}. Therefore, the MATLAB code $\code$ simply computes the function $\phi^{\text{`gammaFunc'}}$. This is straightforward code and so we do not discuss its implementation further. This code provides the entry point for the type of question used in the proof of CRP II (more precisely, it implements $\Phi^l = \Phi(d,\code)$ with $d = l+1$).

	\textbf{Step VI: Verification of \eqref{eq:agi_wrong}.} The verification that the code above will make the AGI fail follows the same lines of the proof of \eqref{crp:2} of Theorem \ref{thm:crp_1_2}, and more precisely that of Proposition \ref{prop:DrivingNegativeProposition}. The main adaptation that needs to be made is that in the current context, Turing machines are not accessed via their G\"odel numbers, but directly by the MATLAB strings that encode them. As has been previously mentioned, the main difference is therefore that $m$ no longer represents an integer, but rather a code.
	
	The verification that $\Gamma$ will fail on $\Phi^l$ is a straightforward rewording of the proof of Proposition \ref{prop:DrivingNegativeProposition} in the current context, making the syntactical adaptations listed above. For completeness, we provide the full verification here.
	
	We now proceed to show that $\Gamma(\stringinput^l)$ is not a correct solution to the problem described by $\Phi^l$. Recall that the question $\Phi^l$ describes the problem of providing an approximation within $\kappa = 10^{-1}$ to a solution $\Xi_{\LP}(\iota)$ of the Linear Problem as in \eqref{eq:LP} with dimensions $\no = d =l+1$ and $\no = 1$, where $\iota = (y,A)$ is such that $\Phi^l$ corresponds to $\iota$. Thus, we proceed to show that $\Gamma(\Phi^l) \notin \ball{\kappa}{\Xi_{\LP}(\iota)}$. Recall from \S \ref{sec:collection_problems} the sets $S^1,S^2 \subseteq \Q^d$ and the inputs $\iota^1_n, \iota^2_n \in \Omega_{1, d}$, which satisfy $\Xi_{\LP}(\iota^j_n) \in S^j$ for every $n \in \N$ and $j \in \{1,2\}$, and are such that $\ball{\kappa}{S^1} \cap \ball{\kappa}{S^2} = \emptyset$ with $\kappa = 10^{-1}$.
	
	To prove the desired conclusion, consider the value of $\gamma_{d}(\text{`gammaFunc'})$, where $\gamma_{d}$ is given in \eqref{eq:gamma_with_codes}.
	From the definition of $\gamma_d$ and from the assumption that the AGI $\Gamma$ must halt when given question $\Phi^l=\Phi(d,\code)$, returning the output $\Gamma(\Phi^l) = $ \text{LPAnswer}, the value of $\gamma_d(\text{`gammaFunc'})$ is necessarily either $1$ or $2$. We consider these cases separately.
	
	\emph{Case (I)}: Suppose that $\gamma_d(\text{`gammaFunc'}) = 1$. Consider the execution $\gamma_d(\text{`gammaFunc'})$. We see that from Lines 10-14 from \emph{gammaFunc.m} the vector \text{LPAnswer} must be such that the call to \text{inBallS2(LPAnswer,0.1)} evaluates to `true'. Examining the code of \text{inBallS2} from Line 18 through to Line 22 of \emph{gammaFunc.m} shows us that the vector \text{LPAnswer} must be contained in $\ball{\kappa}{S^2}$. Note from Lines 2-9 of \emph{gammaFunc.m} that the vector \text{LPAnswer} $= \Gamma(\Phi^l)$ is precisely the output of the AGI to the question $\Phi^l = \Phi(d,\code)$. Therefore, we have observed that $\Gamma(\Phi^l) \in \ball{\kappa}{S^2}$. On the other hand, let us compute the true solution $\Xi_{\LP}(\iota)$ where $\iota$ is such that $\Phi^l=\Phi(d,\code)$ corresponds to $\iota$. As analysed previously, $\code$ computes the function $\phi^{\text{`gammaFunc'}}$, which provides approximations to an input $\iota= (2\cdot 10^{-1}, A) \in \Q \times \Q^{1 \times d}$ in the following sense: when executed on input $(n,i,d)$, the function $\phi^{\text{`gammaFunc'}}$ returns a $2^{-n}$ approximation to the $i$th coordinate of $A$. We now proceed to determine $\iota$.
	
	For every $(n,i,d) = (\text{maxTime}, \text{coordinateIndex}, \text{dimension}) \in \N^3$, Lines 4-5 of \emph{phi.m} execute the function $\gamma_d(\text{`gammaFunc'})$ until time $n$. Since we are assuming that $\gamma_d(\text{`gammaFunc'})=1$, for large enough time $n$ this procedure will halt. Therefore, for such large values of $n$, the `if' condition in Lines 8-10 of \emph{phi.m} will not be met, and instead we execute the \text{else} branch in Line 11 of \emph{phi.m}. Furthermore, as the procedure $\gamma_d(\text{`gammaFunc'})$ returns output $1$, the variable \text{outputValue} will be set to $1$ in Line 12 of \emph{phi.m}. Line 14 of \emph{phi.m} then assigns the value $j = $ \text{pathway} $= 1$. Finally, Line 25 of \emph{phi.m} calls the function \emph{gammaApprox} with input $(1,n,t,i)$ (where the value of $t=$ \text{timeTaken} is the time taken to run $\gamma_{d}(\text{`gammaFunc'})$), thus returning exactly the $i$th coordinate $\hat\Gamma^1_i(n,t) = f_{i}(\iota^1_t)$ of the input $\iota^1_t$ as defined in \eqref{eq:iota} in \S \ref{sec:collection_problems}. 
	
	We conclude that $\phi^{\text{`gammaFunc'}}$ returns approximations to the coordinates of an input $\iota^1_t$ for some $t \in \N$. Therefore the question $\Phi^l$ corresponds to $\iota^1_t$. From the setup of \S \ref{sec:collection_problems} (specifically, Lemma \ref{lemma:assumptions_verification}), we know that $\Xi_{\LP}(\iota^1_t) \subseteq S^1$ and that $\ball{\kappa}{S^1} \cap \ball{\kappa}{S^2} = \emptyset$. Recalling that we observed earlier that $\Gamma(\Phi^l) \in \ball{\kappa}{S^2}$, we conclude that $\Gamma(\Phi^l) \notin \ball{\kappa}{\Xi_{\LP}(\iota)}$, as desired.
	
		\emph{Case (II)}: Suppose that $\gamma_d(\text{`gammaFunc'}) = 2$. The argument in this case is analogous to the previous case. We note the following differences: in the execution $\gamma_d(\text{`gammaFunc'})$, Lines 10-14 from \emph{gammaFunc.m} must be such that the call to \text{inBallS2(LPAnswer,0.1)} evaluates to `\textit{false}' rather than `true', and thus examining the code of \text{inBallS2} from Line 18 through to Line 22 of \emph{gammaFunc.m} shows us that the vector \text{LPAnswer} must be \textit{outside} $\ball{\kappa}{S^2}$. Thus $\Gamma(\Phi^l) \notin \ball{\kappa}{S^2}$.

		Execution of the MATLAB code $\code$ proceeds identically, with the only difference being that Line 14 of \emph{phi.m} assigns the value $j = $ \text{pathway} $= 2$, so that $\phi^{\text{`gammaFunc'}}$ returns approximations to the coordinates of an input $\iota^2_t$ for some $t \in \N$. Thus, the question $\Phi^l$ corresponds to $\iota^2_t$, From the setup of \S \ref{sec:collection_problems} it holds that $\Xi_{\LP}(\iota^2_t) \subseteq S^2$, and since we already observed that $\Gamma(\Phi^l) \notin \ball{\kappa}{S^2}$, we conclude that $\Gamma(\Phi^l) \notin \ball{\kappa}{\Xi_{\LP}(\iota)}$, as desired.

	Either way, we have proven that $\Gamma(\Phi^l) \notin \ball{\kappa}{\Xi_{\LP}(\iota)}$, so that $\Gamma(\Phi^l)$ is not a correct solution to the problem described by $\Phi^l$. Thus the AGI given by $\Gamma$ is guaranteed to fail on the question $\Phi^l = \Phi(d,\code)$, and the verification of \eqref{eq:agi_wrong} is complete.
	
	Since $d \in \{1,\dotsc, K\}$ was arbitrary, this proves that there are $K$ codes $\Phi^1, \dotsc, \Phi^K$ on which the AGI is guaranteed to fail, for every $K \geq 2$.
	
	\textbf{Step VII: Checking the length of the failure sentence for the AGI}: Finally, for the fixed values of fixed $K \in \N$ and $d \in \{1,\dotsc, K\}$, we compute the length of the question $\Phi^l=\Phi(d,\code)$ where we recall that $d = l+1$. The number of characters in this question is given by the sum of the lengths of the separate codes (namely \emph{AGI.m} and its dependencies, as well as the MATLAB code $\code$ and its dependencies \emph{gammaApprox.m}, \emph{phi.m}, and \emph{gammaFunc.m}) and the length of the English sentences appearing in $\mathcal{C}_{\LP}$. 
	
	Note that the decimal expansion of the dimension $d$ appears in the English sentence \emph{`Consider the integer dimension $d={d_k d_{k-1} \dotsc d_2 d_1}$'}. The contribution to the overall length from including this dimension is at most $\log_{10}(d) + 1$, which is bounded from above by $\log_{10}(K+1) + 1 \leq \log_{10}(K) +2$.
	
	The length $\operatorname{len}(\emph{`Consider the integer dimension... , and AGI.m.'})$  is therefore at most $879 + \log_{10}(K)$. The lengths of the MATLAB codes and dependencies (with comments removed) are as follows: $\operatorname{len}(\code) \leq 92$, $\operatorname{len}(gammaApprox.m) \leq 307 $, $\operatorname{len}(phi.m) \leq 708$, and $\operatorname{len}(gammaFunc.m) \leq 1314$. Moreover, recall that $\operatorname{len}(\Gamma)$ denotes by definition the sum of $\operatorname{len}(AGI.m)$ and the length of all its dependencies.
	
	By adding the previous quantities together, we conclude that the length of the code $\Phi^l$ is bounded above by
	\begin{align*}
		\operatorname{len}(\stringinput^d) \leq \operatorname{len}(\Gamma) +  \log_{10}(K) + \epsilon
	\end{align*}
	where $\epsilon = \epsilon_{\text{MATLAB}} = 3300$ characters. This concludes the proof of the Theorem.
\end{proof}

\bibliographystyle{abbrv}
\bibliography{References_CRP_SM}